%% file: DRO-LinearMDP.tex
\documentclass[10pt,english]{article}
\usepackage{natbib}

\usepackage[margin=1in]{geometry} 
\usepackage[T1]{fontenc}
\usepackage[latin9]{inputenc}
\usepackage{bm}
\usepackage{amsmath}
\usepackage{amssymb}
\usepackage{tabularx}
\usepackage{pifont}
\usepackage{tablefootnote}
\usepackage{longtable}
\usepackage{lipsum} 
\usepackage{caption}
\usepackage{bbding}
\usepackage{makecell}
\usepackage{threeparttable}
\usepackage{url}   
\usepackage{amsfonts}
\usepackage{microtype}
\usepackage{bbm}
\usepackage{xspace}

\usepackage[unicode=true,bookmarks=false, breaklinks=false,pdfborder={0 0 1},colorlinks=false]{hyperref}
\hypersetup{colorlinks,citecolor=blue,filecolor=blue,linkcolor=blue,urlcolor=blue}
\usepackage{xcolor,colortbl}
\definecolor{Gray}{gray}{0.85}
\usepackage{enumitem}
\makeatletter
\usepackage{comment}
\usepackage{amsthm} 
\usepackage{booktabs,mathtools}
\usepackage{graphicx}
\usepackage{algorithm}
\usepackage{algpseudocode}

\usepackage{multirow}
\usepackage{dsfont}
\usepackage{color}
\usepackage{float}

\input{macro}

\definecolor{hew}{RGB}{0,47,167}

\definecolor{lxs}{RGB}{138,43,226}

\title{Sample Complexity of Offline Distributionally Robust Linear Markov Decision Processes}

\author{
 	He Wang\thanks{Department of Electrical and Computer Engineering, Carnegie Mellon University,  PA 15213, USA; Email: \texttt{\{hew2,yuejiec\}@andrew.cmu.edu}.}\\
 	CMU
 	\and
     Laixi Shi\thanks{Department of Computing Mathematical Sciences, California Institute of Technology, CA 91125, USA; Email: \texttt{laixis@caltech.edu}.}\\
 	Caltech 
    \and
 	Yuejie Chi\footnotemark[1] \\ 	 
     CMU  
 	}

\date{ \today}

\begin{document}

\maketitle
\begin{abstract}
    In offline reinforcement learning (RL), the absence of active exploration calls for attention on the model robustness to tackle the sim-to-real gap, where the discrepancy between the simulated and deployed environments can significantly undermine the performance of the learned policy. To endow the learned policy with robustness in a sample-efficient manner in the presence of high-dimensional state-action space, this paper considers the sample complexity of  distributionally robust linear Markov decision processes (MDPs) with an uncertainty set characterized by the total variation distance using offline data. We develop a pessimistic model-based algorithm and establish its sample complexity bound under minimal data coverage assumptions,  which outperforms prior art by at least $\widetilde{O}(d)$, where $d$ is the feature dimension. We further improve the performance guarantee of the proposed algorithm by incorporating a carefully-designed variance estimator. 
 
\end{abstract}
\allowdisplaybreaks
\setcounter{tocdepth}{2}
\tableofcontents

\input{section/introduction}

\input{section/problem_setup}

\input{section/DRLSVI.tex}
\input{section/variance_est.tex}

\section{Conclusion}
In this paper, we investigate the sample complexity for distributionally robust offline RL when the model has linear representations and the uncertainty set can be characterized by TV distance. We develop a robust variant of pessimistic value iteration with linear function approximation, called \alg. We establish the sub-optimality guarantees for \alg under various offline data assumptions. Compared to the prior art, \alg notably improves the sample complexity by at least $\widetilde{O}(d)$, under the partial feature coverage assumption. We further incorporate \alg with variance estimation to develop an enhanced \alg, referred to as \algvar, which improves the sub-optimality bound of \alg. In the future, it will be of interest to consider different choices of the uncertainty sets and establish the lower bound for the entire range of the uncertainty level.

\section*{Acknowledgement}
 
This work is supported in part by the grants ONR N00014-19-1-2404, NSF CCF-2106778, DMS-2134080, and CNS-2148212.  The authors also acknowledge Chongyi Zheng for valuable discussions. 

\bibliographystyle{apalike}
\bibliography{ref,bibfileDRO,bibfileRL}
\appendix

\input{appendix/preliminaries.tex}

\input{appendix/proof_thm1.tex}

\input{appendix/proof_thm1_lemmas.tex}

\input{appendix/proof_corollary.tex}
\input{appendix/proof_thm2_varest.tex}

\end{document}

%% file: macro.tex
\newcommand{\ie}{\text{i.e.}}

\newcommand{\Var}{\operatorname{Var}}
\newcommand{\varest}{\widehat{\sigma}}

\renewcommand{\S}{\mathcal{S}}
\newcommand{\A}{\mathcal{A}}
\newcommand{\D}{\mathcal{D}}
\newcommand{\Dtilde}{\tilde{\D}}
\newcommand{\Dmain}{\mathcal{D}^{\operatorname{main}}}
\newcommand{\Daux}{\mathcal{D}^{\operatorname{aux}}}
\newcommand{\Dvar}{\mathcal{D}^{\operatorname{var}}}
\newcommand{\Dtrim}{\mathcal{D}^{\operatorname{trim}}}
\newcommand{\Dmainsub}{\mathcal{D}^{\operatorname{main,sub}}}
\newcommand{\Dvarsub}{\mathcal{D}^{\operatorname{var,sub}}}
\newcommand{\Diid}{\mathcal{D}^{\operatorname{i.i.d.}}}

\newcommand{\Crob}{C_{\operatorname{rob}}^{\star}}

\newcommand{\textb}{\operatorname{b}}

\newcommand{\intd}{{\mathrm{d}}}

\newcommand{\hatV}{\widehat{V}}
\newcommand{\N}{\mathcal{N}}
\newcommand{\M}{\mathcal{M}}
\newcommand{\F}{\mathcal{F}}
\newcommand{\E}{\mathbb{E}}
\newcommand{\B}{\mathbb{B}}

\newcommand{\Nmain}{N^{\operatorname{main}}}
\newcommand{\Naux}{N^{\operatorname{aux}}}
\newcommand{\Nvar}{N^{\operatorname{var}}}
\newcommand{\Ntrim}{N^{\operatorname{trim}}}
\newcommand{\subopt}{\operatorname{SubOpt}}
\renewcommand{\P}{\mathcal{P}}
\newcommand{\Prho}{\mathcal{P}^{\rho}}
\newcommand{\Urho}{\mathcal{U}^{\rho}}
\newcommand{\Pkernel}{\mathbb{P}}
\newcommand{\bellman}{\mathbb{B}}
\newcommand{\indicator}{\mathbbm{1}}
\newcommand{\1}{\mathbbm{1}}

\newcommand{\defeq}{\vcentcolon=}
\newcommand{\resp}{\text{resp.}}
\newcommand{\alg}{\texttt{DROP}\xspace}
\newcommand{\algvar}{\texttt{DROP-V}\xspace}
\newcommand{\twofold}{\texttt{Two-fold-subsampling}\xspace}
\newcommand{\threefold}{\texttt{Three-fold-subsampling}\xspace}
\newcommand{\Tr}{\operatorname{Tr}}

\algnewcommand\algorithmicinput{\textbf{Input:}}
\algnewcommand\Input{\item[\algorithmicinput]}
\algnewcommand\algorithmicoutput{\textbf{Ouput:}}
\algnewcommand\Output{\item[\algorithmicoutput]}
\algnewcommand\algorithmicinit{\textbf{Initialization:}}
\algnewcommand\Init{\item[\algorithmicinit]}
\algnewcommand\algorithmicvar{\textbf{Construct the variance estimator:}}
\algnewcommand\ConstructVar{\item[\algorithmicvar]}

\newcommand{\norm}[1]
{\left\lVert#1\right\rVert}
\DeclareMathOperator*{\argmax}{arg\,max}
\DeclareMathOperator*{\argmin}{arg\,min}

\theoremstyle{plain}
\newtheorem{lemma}{\textbf{Lemma}}

\newtheorem{corollary}{\textbf{Corollary}}

\newtheorem{theorem}{\textbf{Theorem}} 
\newtheorem{nono-theorem}{\textbf{Theorem}}[]

\newtheorem{assumption}{Assumption}
\theoremstyle{remark}

%% file: section/introduction.tex
\section{Introduction}\label{sec:introduction}

In reinforcement learning (RL), agents aim to learn an optimal policy that maximizes the expected total rewards, by actively interacting with an \textit{unknown} environment. However, online data collection may be prohibitively expensive or potentially risky in many real-world applications, e.g., autonomous driving \citep{gu2022constrained}, healthcare \citep{yu2021reinforcement}, and wireless security \citep{uprety2020reinforcement}. This motivates the study of {\em offline RL}, which leverages existing historical data (aka batch data) collected in the past to improve policy learning, and has attracted growing attention \citep{levine2020offline}. Nonetheless, the performance of the learned policy invoking standard offline RL techniques could drop dramatically when the deployed environment shifts from the one experienced by the historical data even slightly, necessitating the development of robust RL algorithms that are resilient against environmental uncertainty.

In response, recent years have witnessed a surge of interests in \textit{distributionally robust offline RL} \citep{zhou2021finite,yang2021towards,shi2022distributionally,blanchet2023double}.
In particular, given only historical data from a nominal environment, \textit{distributionally robust offline RL} aims to learn a policy that optimizes the \textit{worst-case performance} when the environment falls into some prescribed uncertainty set around the nominal one. Such a framework ensures that the performance of the learned policy does not fail drastically, provided that the \textit{distribution shift} between the nominal and deployment environments is not excessively large.

Nevertheless, most existing provable algorithms in distributionally robust offline RL only focus on the tabular setting with finite state and action spaces \citep{zhou2021finite,yang2021towards,shi2022distributionally}, where the sample complexity scales linearly with the size of the state-action space, which is prohibitive when the problem is high-dimensional. To expand the reach of distributionally robust offline RL, a few latest works \citep{ma2023distributionally,blanchet2023double} attempt to develop sample-efficient solutions by leveraging linear function approximation \citep{bertsekas2012dynamic}, which is widely used in both theoretic  \citep{jin2020provably,jin2021pessimism,xiong2023nearly} and practical \citep{prashanth2010reinforcement,bellemare2019distributional} developments of standard RL. However, the existing sample complexity is still far from satisfactory and notably larger than the counterpart of standard offline RL~\citep{jin2021pessimism,xiong2023nearly} with linear function approximation, especially in terms of the dependency on the dimension of the feature space $d$. Therefore, it is natural to ask:
\begin{quote}
\textit{Can we design a provably sample-efficient algorithm for distributionally robust offline RL with linear representations?}
\end{quote}

\subsection{Main contribution}\label{sec:contribution}

To answer this question, we focus on learning a robust variant of linear Markov decision processes in the offline setting.
Throughout this paper, we consider a class of finite-horizon distributionally robust linear MDPs (Lin-RMDPs), where the uncertainty set is characterized by the total variation (TV) distance between the feature representations in the latent space, motivated by its practical \citep{pan2024adjustable} and theoretical appeals \citep{shi2023curious}. 
The highlights of our contributions can be summarized as follows.
 
\begin{itemize}
    \item We propose a distributionally robust variant of pessimistic least-squares value iteration, referred to as \alg, which incorporates linear representations of the MDP model and devises a data-driven penalty function to account for data scarcity in the offline setting. We also establish the sub-optimality bound for \alg under the minimal offline data assumption (cf.~Theorem \ref{thm:main}).
    \item We introduce a clipped single-policy concentrability coefficient $\Crob \geq 1$ tailored to characterize the partial feature coverage of the offline data in Lin-RMDPs, and demonstrate that \alg attains $\epsilon$-accuracy (for learning the robust optimal policy) as soon as the sample complexity is above $\widetilde{O}(\Crob d^2H^4/\epsilon^2)$ (cf. Corollary \ref{corollary:sufficient_coverage}). Compared with the prior art \citep{blanchet2023double}, \alg~improves the sample complexity by at least $\widetilde{O}(d)$.

    \item  We further develop a variance-weighted variant of \alg by integrating a carefully designed variance estimator, dubbed by \algvar. Due to tighter control of the variance, \algvar exhibits an improved sub-optimality gap under the full feature coverage assumption (see Section~\ref{assump:full_coverage}).

\end{itemize}

See Table~\ref{tab:linearmdp} for a summary of our results in terms of the sub-optimality gap.

\begin{table*}[!h]
    \begin{center}
    \begin{tabular}{c|c|c}
    \toprule
    Algorithm &  Coverage  & Sub-optimality gap \tabularnewline
     \hline 
    {} \vphantom{$\frac{1^{7}}{1^{7^{7}}}$}    &  &     \tabularnewline
{ \multirow{-2}{*}{\cite{blanchet2023double}}} &     \multirow{-2}{*}{ partial} &\multirow{-2}{*}{$\sqrt{\frac{C^{\star}_{1} d^4H^4}{K}}$}    \tabularnewline  
    \hline
        {}  \vphantom{$\frac{1^{7}}{1^{7^{7}}}$}   & &    \tabularnewline
    {\multirow{-2}{*}{\alg \bf (this work)}}   &  \multirow{-2}{*}{partial}  &  \multirow{-2}{*}{$\sqrt{\frac{ \Crob d^2 H^4}{K}}$} \tabularnewline 
    \hline
    {}  \vphantom{$\frac{1^{7}}{1^{7^{7}}}$}   & &    \tabularnewline
    {\multirow{-2}{*}{\alg \bf (this work)}}   &  \multirow{-2}{*}{arbitrary}  &  \multirow{-2}{*}{$ \sqrt{d}H\sum_{i=1}^{d}\sum_{h=1}^{H} \underset{{P\in\Prho(P^0)}}{\max}\E_{\pi^{\star}, P}\|\phi_i(s,a)\1_i\|_{(\Lambda_h)^{-1}}$} \tabularnewline 
    \hline
    {}  \vphantom{$\frac{1^{7}}{1^{7^{7}}}$}   & &    \tabularnewline
    {\multirow{-2}{*}{\algvar \bf (this work)}}    & \multirow{-2}{*}{full}  &  \multirow{-2}{*}{$ \sqrt{d}\sum_{i=1}^{d}\sum_{h=1}^{H}  \underset{{P\in\Prho(P^0)}}{\max}\E_{\pi^{\star}, P}\|\phi_i(s,a)\1_i\|_{(\Sigma_h^{\star})^{-1}}$} \tabularnewline 
    \hline
    \toprule
    \end{tabular}
    \end{center}
    \caption{Our results and comparisons with prior art (up to log factors) in terms of sub-optimality gap, for learning robust linear MDPs with an uncertainty set measured by the TV distance. Here, $d$ is the feature dimension, $H$ is the horizon length, $K$ is the number of trajectories, $C_1^{\star},\Crob \in [1,\infty)$ are some concentrability coefficients (defined in Section~\ref{sec:alg_partial}) satisfying $\Crob \leq d C_1^{\star}$. In addition, $\Prho(P^0)$ denotes the uncertainty set around the nominal kernel $P^0$, $\phi_i(s,a)$ is the $i$-th coordinate of the feature vector given any state-action pair $(s,a)\in\S\times\A$, $\Lambda_h $ and $\Sigma_h^{\star}$ are some sort of cumulative sample covariance matrix and variance-weighted cumulative sample covariance matrix, satisfying $H \Lambda_h^{-1} \succeq (\Sigma_h^{\star})^{-1}$.}
    \label{tab:linearmdp}
    \end{table*}

\input{appendix/related_work.tex}

\paragraph{Notation.}
Throughout this paper, we define $\Delta(\S)$ as the probability simplex over a set $\S$, $[H]\defeq \{1,\ldots,H\}$ and $[d]\defeq \{1,\ldots,d\}$ for some positive integer $H,d>0$. We also denote $\1_i$ as the vector with appropriate dimension, where the $i$-th coordinate is $1$ and others are $0$. We use $I_d$ to represent the identity matrix of the size $d$. For any vector $x\in\mathbb{R}^d$, we use $\|x\|_{2}$ and $\|x\|_{1}$ to represent its $l_2$ and $l_1$ norm, respectively. In addition, we denote $\sqrt{x^{\top}Ax}$ as $\|x\|_A$ given any vector $x\in\mathbb{R}^d$ and any semi-definite matrix $A\in\mathbb{R}^{d\times d}$.  For any set $\D$, we use $|\D|$ to represent the cardinality (\ie, the size) of the set $\D$. Additionally, we use $\min\{a,b\}_{+}$ to denote the minimum of $a$ and $b$ when $a,b>0$, and $0$ otherwise. We also let $\lambda_{\min}(A)$ to denote the smallest eigenvalue of any matrix $A$.

%% file: appendix/related_work.tex
\subsection{Related works}\label{sec:prior_works}
In this section, we mainly discuss works that study sample complexity of linear MDPs and robust RL, which are closely related to this paper.

\paragraph{Finite-sample guarantees for linear MDPs.} Considering linear function approximation in RL, a significant body of works study linear MDPs with  linear transitions and rewards. Focusing on offline settings, \citet{jin2021pessimism} proposed Pessimistic Value Iteration (PEVI) for offline RL with finite-horizon linear MDPs, which incorporated linear function approximation together with the principle of pessimism \citep{rashidinejad2021bridging,shi2022pessimistic,yan2023efficacy,woo2024federated,li2022settling}. However, the temporally statistical dependency arising from backward updates leads to an $O(\sqrt{d})$ amplification in the sub-optimality gap with comparison to the lower bound \citep{zanette2021provable}.
Subsequently, \citet{min2021variance} and \citet{yin2022near} attempted to address this gap. 
The near-optimal sample complexity in \citet{xiong2023nearly} is achieved by applying variance estimation techniques, motivating us to explore variance information for robust offline RL. Beyond these works, there are many offline RL works that investigate model-free algorithms or the infinite-horizon setting that deviates from our focus \citep{zanette2021provable,uehara2021pessimistic,xie2021bellman}. In addition to the offline setting, linear MDPs are also widely explored in other settings such as generative model or online RL \citep{yang2020reinforcement,jin2020provably,agarwal2023vo,he2023nearly}.

\paragraph{Distributionally robust RL with tabular MDPs.}
To promote robustness in the face of environmental uncertainty, a line of research known as distributionally robust RL incorporates distributionally robust optimization tools with RL to ensure robustness in a worst-case manner \citep{iyengar2005robust,xu2012distributionally,wolff2012robust,kaufman2013robust,tamar2014scaling,ho2018fast,smirnova2019distributionally,derman2020distributional,ho2021partial,badrinath2021robust,goyal2022robust,ding2024seeing}. Recently, a surge of works focus on understanding/improving sample complexity and computation complexity of RL algorithms in the tabular setting \citep{wang2021online,zhou2021finite,dong2022online,yang2021towards,panaganti2021sample,liu2022distributionally,li2022first,xu2023improved,wang2023finite,wang2023sample,liang2023single,wang2023foundation,li2023first,yang2023avoiding,zhang2023regularized,kumar2023policy,blanchet2023double,shi2024sample}. Among them, \citet{zhou2021finite,yang2021towards,shi2022distributionally,blanchet2023double} consider distributionally robust RL in the offline setting, which are most related to us. Different from the sample complexity achieved in the tabular setting that largely depends on the size of state and action spaces, this work advances beyond the tabular setting and develops a sample complexity guarantee that only depends on the feature dimension based on the linear MDP model assumption.

\paragraph{Distributionally robust RL with linear MDPs.}
The prior art \citep{blanchet2023double} studies the same robust linear MDP (Lin-RMDP) problem as this work, while the sample complexity  dependency on the feature dimension $d$ is still worse than that of standard linear MDPs in offline setting \citep{jin2021pessimism,xiong2023nearly}, highlighting gap for potential improvements.
Besides TV distance considered herein, \citet{ma2023distributionally,blanchet2023double} consider the Kullback-Leibler (KL) divergence for the uncertainty set. Moreover, \citet{liu2024distributionally} explores Lin-RMDPs with an additional structure assumption in the online setting, which diverges from our focus on the offline context. 
We also note that a concurrent work \citep{liu2024minimax} studies offline Lin-RMDPs, which aligns with our interest. However, their focus is limited to the well-explored data coverage assumption.
In addition to linear representations for robust MDPs, \citet{badrinath2021robust,ramesh2023distributionally} consider other classes of function for model approximation.

%% file: section/problem_setup.tex
\section{Problem Setup}
In this section, we introduce the formulation of finite-horizon distributionally robust linear MDPs (Lin-RMDPs), together with the batch data assumptions and the learning objective. 

\paragraph{Standard linear MDPs.} Consider a finite-horizon standard linear MDP $\M = (\S, \A, H, P=\{P_h\}_{h=1}^H, r=\{r_h\}_{h=1}^H)$, where $\S$ and $\A$ denote the state space and action space respectively, and $H$ is the horizon length. At each step $h\in[H]$, we denote $P_h: \S\times\A\to \Delta (\S)$ as the transition kernel and $r_h:\S\times\A\to [0,1]$ as the deterministic reward function, which satisfy the following assumption used in \citet{yang2019sample,jin2020provably}. 

\begin{assumption}[Linear MDPs] \label{assump:linear-mdp}
    A finite-horizon MDP $\M = (\S, \A, H, P, r)$ is called a linear MDP if given a known feature map $\phi:\S\times \A\to\mathbb{R}^d$, there exist $d$ unknown measures $\mu_h^P = (\mu_{h, 1}^P, \cdots, \mu_{h, d}^P)$ over the state space $\S$ and an unknown vector $\theta_h \in \mathbb{R}^d$ at each step $h$ such that
\begin{align}
    r_h(s, a) &= \phi(s, a)^{\top} \theta_h,\quad P_h(s' \mid s , a) = \phi(s, a)^{\top} \mu_h^P(s'), \quad \quad \forall (h,s, a, s') \in [H]\times \S \times \A \times \S. \label{eq:linear-mdp-defn}
\end{align}
Without loss of generality, we assume $\|\phi(s,a)\|_2\le 1$ and $\phi_i(s,a)\ge 0$ for any $(s,a,i)\in\S\times\A\times[d]$, and  $\max\{\int_{\S}\|\mu_h^P(s)\|_2\intd s,\|\theta_h\|_2\} \leq \sqrt{d}$ for all $h \in [H]$.
\end{assumption}

In addition, we denote $\pi = \{\pi_h\}_{h = 1}^H$ as the policy of the agent, where $\pi_h:\S\to\Delta(\A)$ is the action selection probability over the action space $\A$ at time step $h$. Given the policy $\pi$ and the transition kernel $P$, the value function $V_h^{\pi,P}$ and the Q-function $Q_h^{\pi,P}$ at the $h$-th step are defined as: for any $(s,a)\in \mathcal{S}\times \mathcal{A}$,
\begin{align*}
  V^{\pi,P}_h (s) = \mathbb{E}_{\pi,P}\bigg[ \sum_{t=h}^H r_t(s_t,a_t)\mid s_h\!=\!s\bigg], \ \ 
    Q^{\pi,P}_h\!(s,a) =  r_h(s,a) \!+\! \mathbb{E}_{\pi,P}\bigg[\!\sum_{t=h+1}^H \!\!r_t(s_t,a_t)\mid s_h = s, a_h=a\bigg],
\end{align*}
where the expectation is taken over the randomness of the trajectory induced by the policy $\pi$ and the transition kernel $P$.
\paragraph{Lin-RMDPs: distributionally robust linear MDPs.}
In this work, we consider distributionally robust linear  MDPs (Lin-RMDPs), where the transition kernel can be an arbitrary one within an \textit{uncertainty set} around the nominal kernel --- an ensemble of probability transition kernels, rather than  a fixed transition in standard linear MDPs \citep{jin2021pessimism,yin2022near,xiong2023nearly}. 
Formally, we denote it by $\mathcal{M}_{\operatorname{rob}} = (\S,\A,H,\Prho(P^0),r)$, where $P^0$ represents a nominal transition kernel and then $\Prho(P^0)$ represents the \textit{uncertainty set} (a ball) around the nominal $P^0$ with some \textit{uncertainty level} $\rho\ge 0$. 
For notational simplicity, we denote $\mu_h^0 \defeq \mu_h^{P^0}$ and according to \eqref{eq:linear-mdp-defn}, we let
 \begin{align*}
  \forall (h,s,a,s')\in [H]\times \S\times \A\times\S: \quad  P^0_{h,s,a}(s')\defeq P^0_h(s'\mid s,a) &= \phi(s,a)^{\top}\mu_h^0(s').
 \end{align*}
In this work, we consider the total variation (TV) distance \citep{tsybakov08nonparametric} as the divergence metric for the uncertainty set $\Prho$ 
and we assume that Lin-RMDPs satisfy the following $d$-rectangular assumption \citep{ma2023distributionally,blanchet2023double}.
\begin{assumption}[Lin-RMDPs]\label{assump:d_rectangular}
    In Lin-RMDPs,  the uncertainty set $\Prho(P^0)$ is $d$-rectangular, i.e., $\mu^0_{h,i} \in \Delta(\S)$ for any $(h,i)\in[H]\times[d]$ and 
    
\begin{align*}
    \Prho(P^0)\defeq \otimes_{[H],\S, \A} \P^{\rho}(P_{h,s,a}^0), &\quad\text{with}~\P^{\rho}(P_{h,s,a}^0) \defeq \left\{ \phi(s, a)^{\top} \mu_h(\cdot): \mu_h \in \Urho(\mu_{h}^0) \right\},\\
  \text{where} \quad \Urho(\mu_{h}^0) \defeq \otimes_{[d]} \Urho(\mu_{h, i}^0), &\quad\text{with}~\Urho(\mu_{h, i}^0) \defeq \left\{\mu_{h, i}: D_{\operatorname{TV}}\left( \mu_{h, i} \parallel \mu^0_{h, i} \right) \leq \rho ~\text{and}~\mu_{h,i}\in\Delta(\S) \right\}.
\end{align*}
Here, $D_{\operatorname{TV}}( \mu_{h, i} \parallel \mu^0_{h, i} )$ represents the $\operatorname{TV}$-distance between two measures over the state space, i.e.
\begin{equation*}
    D_{\operatorname{TV}} (\mu_{h,i},\mu_{h,i}^0) = \frac{1}{2}\|\mu_{h,i}-\mu_{h,i}^0\|_1 ,
\end{equation*}
where $\otimes_{[d]}$ (resp.~$\otimes_{[H],\S, \A}$) denotes Cartesian products over $[d]$ (resp.~$[H]$, $\S$,  and $\A$). 
\end{assumption}
Assumption \ref{assump:d_rectangular} indicates that the uncertainty set can be decoupled into each feature dimension $i\in[d]$ independently, so called $d$-rectangularity. 
Note that by letting $d=SA$ and $\phi_i(s,a) = \1_i$ for any $(i,s,a)\in[d]\times\S\times\A$, the Lin-RMDP is instantiated to the tabular RMDP and $d$ rectangularity becomes the $(s,a)$-rectangularity commonly used in prior literatures \citep{yang2021towards,shi2023curious}. 
Moreover, when the uncertainty level $\rho=0$, Lin-RMDPs reduce to a subclass of standard linear MDPs satisfying Assumption~\ref{assump:linear-mdp}.

\paragraph{Robust value function and robust Bellman equations.}
Considering Lin-RMDPs with any given policy $\pi$, we define  \textit{robust value function} (\resp~\textit{robust Q-function}) to characterize the \textit{worst-case} performance induced by  all possible transition kernels over the uncertainty set, denoted as
\begin{equation*}
   \forall (h,s,a)\in [H] \times \S\times\A: \quad  V_h^{\pi,\rho}(s) \defeq \inf_{P \in \Prho(P^0)} V_h^{\pi,P}(s), \qquad
    Q_h^{\pi,\rho}(s,a) = \inf_{P \in \Prho(P^0)} Q_h^{\pi,P}(s,a).
\end{equation*}
They satisfy the following \textit{robust Bellman consistency equations}:
    \begin{align}
       \forall (h,s,a)\in [H] \times \S\times\A: ~~~    Q_h^{\pi,\rho}(s,a) = \bellman_h^{\rho} V_{h+1}^{\pi,\rho}(s,a), ~~ \text{where} \  V^{\pi,\rho}_h(s) = \E_{a\sim\pi_h(\cdot\mid s)} [Q^{\pi,\rho}_h(s, a)],\label{eq:robust_bellman_consistency}    
    \end{align}
and the robust linear Bellman operator and transition operator for any function $f:\S\to \mathbb{R}$ is defined by
    \begin{align}
        [\bellman^{\rho}_hf](s, a) &\defeq r_h(s, a) + [\mathbb{P}^{\rho}_hf](s, a),\label{eq:robust_bellman_operator}\\
        [\Pkernel^{\rho}_hf](s, a) &\defeq \inf_{\mu_{h} \in \Urho(\mu_h^0)} \int_{\S} \phi(s, a)^{\top} \mu_h(s') f(s')\intd s'. \label{eq:robust_transition_operator}
    \end{align}
 Note that under Assumption \ref{assump:d_rectangular}, the robust Bellman operator inherits the linearity of the Bellman operator in standard linear MDPs \citep[Proposition 2.3]{jin2020provably}, as shown in the following lemma. The proof is postponed to Appendix \ref{proof:lemma_linearity}.
\begin{lemma}[Linearity of robust Bellman operators]\label{lemma:linearity} Suppose that the finite-horizon Lin-RMDPs satisfies Assumption \ref{assump:linear-mdp} and \ref{assump:d_rectangular}. There exist weights $w^{\rho} = \{w^{\rho}_h\}_{h = 1}^H$, where $w^{\rho}_h \defeq \theta_h + \inf_{\mu_h\in \Urho(\mu_h^0)} \int_{ \S} \mu_h(s') f(s') \intd s'$ for any $h\in[H]$, such that $\mathbb{B}^\rho_h f(s, a)$ is linear with respect to the feature map $\phi$, i.e., $\mathbb{B}^\rho_h f(s, a) = \phi(s, a)^{\top} w^{\rho}_h$.
\end{lemma}

In addition, conditioned on some initial state distribution $\zeta$, we define the induced \textit{occupancy distribution} w.r.t. any policy $\pi$ and transition kernel $P$ by
\begin{equation}\label{eq:occupancy_distribution_any}
       \forall (h,s,a)\in [H] \times \S\times\A: ~~     d_h^{\pi,P}(s) \defeq d_h^{\pi,P}(s;\zeta) =  \mathbb{P}(s_h = s\mid \zeta ,\pi,P), ~~ d_h^{\pi,P}(s,a)  = d_h^{\pi,P}(s)\pi_h(a\mid s).
\end{equation}

To continue, we denote $\pi^{\star} = \{\pi^{\star}_h\}_{h = 1}^H$ as a deterministic \textit{optimal robust policy} \citep{iyengar2005robust}.  
The resulting  \textit{optimal robust value function} and \textit{optimal robust Q-function} are defined as:  
\begin{align*}
    V_h^{\star,\rho}(s) &\defeq V_h^{\pi^\star,\rho}(s)=\max_{\pi} V_h^{\pi,\rho}(s), & &\forall (h,s)\in [H]\times\S,\\
    Q_h^{\star,\rho}(s,a) &\defeq Q_h^{\pi^\star,\rho}(s,a)= \max_{\pi} Q_h^{\pi,\rho}(s,a), & &\forall (h,s,a)\in [H]\times\S\times\A.
\end{align*}
 Accordingly, we also have the following \textit{robust Bellman optimality equation}:
\begin{equation}\label{eq:robust_bellman_optimality_Q}
    Q_h^{\star,\rho}(s, a) = [\mathbb{B}_h^{\rho} V_{h + 1}^{\star,\rho}](s, a),  \quad \forall (h,s,a)\in[H]\times \S\times \A.
\end{equation}
Similar to \eqref{eq:occupancy_distribution_any}, we also denote the occupancy distribution associated with the optimal robust policy $\pi^{\star}$ and some transition kernel $P$ by
\begin{equation}\label{eq:occupancy_distribution_optimal}
  \forall (h,s,a)\in [H] \times \S\times\A: \quad  d_h^{\star,P}(s) \defeq d_h^{\pi^{\star},P}(s;\zeta), \qquad d_h^{\star,P}(s,a) \defeq d_h^{\pi^{\star},P}(s)\pi_h^{\star}(a\mid s).
\end{equation}

\paragraph{Offline data.}
Suppose that we can access a batch dataset $\D = \{(s_h^{\tau},a_h^{\tau},r_h^{\tau},s_{h+1}^{\tau})\}_{h\in [H],\tau\in [K]}$, which consists $K$ i.i.d. trajectories that are generated by executing some (mixed) behavior policy $\pi^{\textb} = \{\pi_h^{\textb}\}_{h=1}^H$ over some nominal linear MDP $\mathcal{M}^0 = (\S,\A, H,P^0,r)$. Note that $\D$ contains $KH$ transition-reward sample tuples in total, where 
each sample tuple $(s_h^{\tau},a_h^{\tau},r_h^{\tau},s_{h+1}^{\tau})$ represents that the agent took the action $a^\tau_h$ at the state $s_h^{\tau}$, received the reward $r^{\tau}_h = r_h(s_h^{\tau},a_h^{\tau})$, and then observed the next state $s_{h+1}^{\tau} \sim P_h^0(\cdot\mid s_h = s_h^{\tau}, a_h = a_h^{\tau})$. 
To proceed, we define 
\begin{equation*}
    \D_h^0 = \{(s_h^{\tau},a_h^{\tau},r_h^{\tau},s_{h+1}^{\tau})\}_{\tau\in[K]},
\end{equation*}
which contains all samples at the $h$-th step in $\D$. For simplicity, we abuse the notation $\tau\in \D_h^0$ to denote $(s_h^{\tau},a_h^{\tau},r_h^{\tau},s_{h+1}^{\tau})\in\D_h^0$.
In addition, we define the induced occupancy distribution w.r.t. the behavior policy $\pi^{\textb}$ and the nominal transition kernel $P^0$ at each step $h$ by 
\begin{equation}\label{eq:occupancy_distribution_behavior}
   \forall (h,s,a)\in [H] \times \S\times\A: \quad  d_h^{\textb}(s) \defeq d_h^{\pi^{\textb},P^0}(s;\zeta) \quad\text{and}\quad
    d_h^{\textb}(s,a) \defeq d_h^{\pi^{\textb},P^0}(s,a;\zeta),
\end{equation}
which is conditioned on the initial distribution $\zeta$. 
\paragraph{Learning goal.} Given the batch dataset $\D$, the goal of solving the Lin-RMDP $\M_{\operatorname{rob}}$ with a given initial state distribution $\zeta$ is to learn an $\epsilon$-optimal robust policy $\widehat{\pi}$ such that the sub-optimality gap satisfies
\begin{equation}\label{eq:suboptimality_gap}
    \subopt(\widehat{\pi};\zeta,\Prho)\defeq V_1^{{\star},\rho}(\zeta) - V_1^{\widehat{\pi},\rho}(\zeta) \le \epsilon,
\end{equation}
using as few samples as possible, where $\epsilon$ is the targeted accuracy level, 
\begin{equation*}
    V_1^{{\star},\rho}(\zeta) = \mathbb{E}_{s_1\sim\zeta}[V_1^{{\star},\rho}(s_1)] \quad\text{and}\quad V_1^{\widehat{\pi},\rho}(\zeta) = \mathbb{E}_{s_1\sim\zeta} [V_1^{\widehat{\pi},\rho}(s_1)].
\end{equation*}

%% file: section/DRLSVI.tex
\section{Algorithm and Performance Guarantees}
In this section, we propose a model-based approach referred to as \textit{\underline{D}istributionally \underline{Ro}bust \underline{P}essimistic Least-squares Value Iteration} (\alg), by constructing an empirical Bellman operator for Lin-RMDPs in an offline fashion. Then we analyze the sub-optimality bound of the robust policy learned from \alg  and discuss the sample complexity under different historical data quality scenarios.

\subsection{Empirical robust Bellman operator and strong duality}

Recalling the robust Bellman operator in \eqref{eq:robust_bellman_operator} gives that for any value function $V: \S \rightarrow [0,H]$,
\begin{equation*}
    (\bellman^{\rho}_hV)(s, a) = r_h(s, a) + \inf_{\mu_{h} \in \Urho(\mu_h^0)} \int_{\S} \phi(s, a)^{\top} \mu_h(s') V(s') \intd s',
\end{equation*}
which can be equivalently rewritten as its dual form:
\begin{equation}\label{eq:TV_dual_linear}
    (\bellman^{\rho}_hV)(s, a) = \phi(s, a)^{\top} \big(  \theta_h + \nu_h^{\rho,V} \big),
\end{equation}
from the views of strong duality \citep{iyengar2005robust,shi2023curious} (see the detailed proof in Appendix \ref{sec:TV_dual}). Here,  $\nu_{h}^{\rho,V} = [\nu_{h,1}^{\rho,V},\nu_{h,2}^{\rho,V}\ldots,\nu_{h,d}^{\rho,V}]^{\top}\in\mathbb{R}^d$ and its $i$-th coordinate is defined by
\begin{equation}\label{eq:nu}
    \nu_{h, i}^{\rho,V} \defeq \!\!\!\!\max_{\alpha\in [\min_s V(s),\max_s V(s)]}\!\! \big\{ \E_{s\sim\mu^0_{h,i}} [V]_{\alpha}(s) \!-\! \rho(\alpha\! - \min_{s'} [V]_{\alpha}(s'))\big\}, ~\text{with}~[V]_{\alpha}(s) = \begin{cases}
        \alpha, &\!\text{if }V(s) > \alpha,\\
        V(s), &\!\text{otherwise}.
    \end{cases}
\end{equation}
However, notice that we cannot directly apply the above robust Bellman operator and perform value iterations since we cannot have access to the ground-truth nominal linear MDP $\M^0$ (i.e., $\theta_h$ and $\mu_h^0$). Therefore, for each time step $h$, we incorporate ridge regression to obtain the estimator $\widehat{\theta}_h\in\mathbb{R}^d$ and $\widehat{\nu}_{h}^{\rho,V}\in\mathbb{R}^d$, based on the batch dataset $\mathcal{D}^0_h$ that contains all the samples at $h$-th step collected in $\D^0$. In particular, for any value function $V:\S\to[0,H]$ and any time step $h\in[H]$, the estimator $\widehat{\theta}_h$ and the $i$-th coordinate of $\widehat{\nu}_{h}^{\rho,V}$ are defined by
\begin{align}
     \widehat{\theta}_{h} &= \arg\min_{\theta\in\mathbb{R}^d} \sum_{\tau\in\D_h^0} \left(\phi(s_h^\tau,a_h^{\tau})^{\top}\theta - r_h^\tau \right)^2 + \lambda_0\|\theta\|_2^2= \Lambda_h^{-1} \bigg(\sum_{\tau\in\D_h^0}\phi(s_h^{\tau},a_h^{\tau})r_h^{\tau} \bigg),\label{eq:update_theta}\\
   \widehat{\nu}_{h,i}^{\rho,V} &= \max_{\alpha \in [\min_s {V}(s),\max_s {V}(s)]} \big\{ \bar{\nu}_{h,i}^{{V}}(\alpha) - \rho \big(\alpha - \min_{s'} [V]_{\alpha}(s')\big)\big\}, \qquad \forall i\in[d],\label{eq:update_nu}
\end{align}
where $\lambda_0>0$ is the regularization coefficient, $\bar{\nu}_{h,i}^{{V}}(\alpha)$ is the  $i$-th coordinate of $\bar{\nu}_h^{{V}}(\alpha)$ defined by
\begin{align}
    \bar{\nu}_h^{{V}}(\alpha) 
    &= \arg\min_{\nu\in\mathbb{R}^d}\!\sum_{\tau\in \D_h^0}\! \big(\phi(s_h^{\tau},a_h^{\tau})^{\top}\nu - [{V}]_\alpha(s_{h+1}^{\tau})\big)^2 + \lambda_0\|\nu\|_2^2 = \Lambda_h^{-1}\!\bigg( \sum_{\tau\in \D_h^0} \phi(s_h^{\tau},a_h^{\tau})[{V}]_{\alpha}(s_{h+1}^\tau)\bigg),\label{eq:update_bar_nu}
\end{align}
and the cumulative sample covariance matrix is denoted as
\begin{equation}
    \Lambda_h = \sum_{\tau\in \D_h^0} \phi(s_h^{\tau}, a_h^{\tau}) \phi(s_h^{\tau}, a_h^{\tau})^{\top} + \lambda_0 I_d. \label{eq:Lambda}
\end{equation}
Then following the linearity of the robust Bellman operator shown in Lemma \ref{lemma:linearity}, we construct the empirical robust Bellman operator $\widehat{\B}^{\rho}_h $ to approximate $\B^{\rho}_h $, using  the estimators obtained from \eqref{eq:update_theta} and \eqref{eq:update_nu}: for any value function $V:\S\to[0,H]$,
\begin{equation}\label{eq:empirical_bellman}
     (\widehat{\B}^{\rho}_h {V}) (s,a)= \phi(s,a)^{\top}(\widehat{\theta}_h+\widehat{\nu}_h^{\rho,{V}}),\qquad \forall(s,a,h)\in\S\times\A\times[H].
\end{equation}

\subsection{\alg: distributionally robust pessimistic least-squares value iteration}
To compute \eqref{eq:empirical_bellman} for all time steps $h\in[H]$ recursively, we propose a distributionally robust pessimistic least-squares value iteration algorithm, referred to as \alg and summarized as Algorithm~\ref{alg:DRLSVI}.
\begin{algorithm}[!t]
\caption{Distributionally Robust Pessimistic Least-Squares Value Iteration (\alg)}\label{alg:DRLSVI}
\begin{algorithmic}[1]
\Input Dataset $\mathcal{D}$; feature map $\phi(s,a)$ for $(s,a)\in\S\times\A$; parameters $\lambda_0,\gamma_0>0$.
\State Construct a temporally  independent dataset $\D_0 = \twofold(\D)$ (Algorithm~\ref{alg:two-sampling}).
\Init Set $\widehat{Q}_{H + 1}(\cdot,\cdot) = 0$ and $\widehat{V}_{H + 1}(\cdot) = 0$.
\For{step $h = H, H - 1, \cdots, 1$}
    \State $\Lambda_h =  \sum_{\tau\in \D_h^0}  \phi(s_h^{\tau}, a_h^{\tau}) \phi(s_h^{\tau}, a_h^{\tau})^{\top} + \lambda_0 I_d$.
    \State $\widehat{\theta}_h = \Lambda_h^{-1} \left(  \sum_{\tau\in \D_h^0}  \phi(s_h^{\tau}, a_h^{\tau}) r_h^{\tau} \right)$.
    \For{feature $i = 1, \cdots, d$}
    \State Update $\widehat{\nu}_{h,i}^{\rho,\widehat{V}}$ via \eqref{eq:update_nu}.
    \EndFor
    \State $\widehat{w}_h^{\rho,\widehat{V}} = \widehat{\theta}_h + \widehat{\nu}_h^{\rho,\widehat{V}}$.
    \State $\bar{Q}_h(\cdot, \cdot) = \phi(\cdot, \cdot)^{\top}\widehat{w}_h^{\rho,\widehat{V}} - \gamma_0\sum_{i=1}^d \|\phi_i(\cdot,\cdot)\1_i\|_{\Lambda_h^{-1}}$.
    \State $\widehat{Q}_h(\cdot, \cdot) = \min\left\{\bar{Q}_h(\cdot,\cdot), H - h + 1 \right\}_{+} $.
    \State $ \widehat{\pi}_h(\cdot) = \argmax_{a\in\A} \widehat{Q}_h(\cdot,a)$.
    \State $\widehat{V}_h(\cdot) =  \widehat{Q}_h(\cdot,\widehat{\pi}_h(\cdot))$.
\EndFor
\Output $\widehat{V}\defeq \{\hatV\}_{h=1}^{H+1}$, $\widehat{\pi} \defeq \{\widehat{\pi}_h\}_{h = 1}^H$.
\end{algorithmic}
\end{algorithm}

In Algorithm \ref{alg:DRLSVI}, we first construct a dataset $\D_0$ by subsampling from $\D$ by \twofold (cf.~Algorithm~\ref{alg:two-sampling}), inspired by \citet{li2022settling} to tackle the statistical dependency between different time steps $h\in[H]$ in the original batch dataset $\D$. As the space is limited, we defer the details of \twofold and the corresponding statistical independence property to Section \ref{sec:two-fold}. With $\D_0$ in hand and initializations $\widehat{Q}_{H + 1}(\cdot,\cdot) = 0$ and $\widehat{V}_{H + 1}(\cdot) = 0$, the updates at time step $h$ in \alg can be boiled down to the following two steps. The first one is to construct the empirical robust Bellman operator via \eqref{eq:update_theta}-\eqref{eq:empirical_bellman} conditioned on a fixed $\widehat{V}_{h+1}$ from the previous iteration (see the line 3-8 in Algorithm \ref{alg:DRLSVI}).  Then we can estimate the robust Q-function from the pessimistic value iteration as below:
\begin{equation*}
    \bar{Q}_h(s,a) = \widehat{\B}^{\rho}_h(\widehat{V}_{h+1}) (s,a)-\underset{\text{penalty function}~\Gamma_h:\S\times\A\to\mathbb{R}}{\underbrace{\gamma_0\sum_{i=1}^d \|\phi_i(s,a)\1_i\|_{\Lambda_h^{-1}}}},\quad\forall (s,a)\in \S\times \A,
\end{equation*}
where $\gamma_0>0$ is the coefficient of the penalty term.

The above pessimistic principle is widely advocated in both standard and robust offline RL \citep{jin2021pessimism,xiong2023nearly,shi2022distributionally}. When dealing with the uncertainty set characterized by TV distance, previous penalty designs tailored for standard linear MDPs \citep{jin2021pessimism,xiong2023nearly} and robust linear MDPs specifically addressing KL divergence \citep{ma2023distributionally}, are no longer applicable. To this end, we carefully devise the penalty function, denoted as $\Gamma_h$, for Lin-RMDPs with TV distance. Compared to the prior art \citep{blanchet2023double} which promotes pessimism by solving an inner constrained optimization problem, our proposed penalty function efficiently addresses the uncertainty in every feature dimension $i\in[d]$, which plays a crucial role in improving the sub-optimality gap.

\subsection{Performance guarantees of \alg}\label{sec:alg_results}
Next, we provide the theoretical guarantees for \alg, under different assumptions on the batch data quality. We start without any coverage assumption on the batch dataset $\D$, where the proof is postponed to Appendix \ref{proof:thm_main}.
\begin{theorem}\label{thm:main}
    Consider any $\delta\in(0,1)$. Suppose that Assumption \ref{assump:linear-mdp} and \ref{assump:d_rectangular} hold. By setting 
   \begin{equation*}
    \lambda_0 = 1, \quad \gamma_0 = 6\sqrt{d\xi_0}H,\quad \text{where }\xi_0 = \log(3H{K}/\delta),
      \end{equation*}
  one has with the probability at least $1-3\delta$, the policy $\widehat{\pi}$ generated by Algorithm \ref{alg:DRLSVI} satisfies
   \begin{equation}\label{eq:thm_main}
    \subopt(\widehat{\pi};\zeta,\Prho)\le \widetilde{O}(\sqrt{d}H)\sum_{h=1}^H\sum_{i=1}^d  \max_{P\in\Prho(P^0)}\E_{\pi^{\star}, P}\left[\|\phi_i(s_h,a_h)\1_i\|_{\Lambda_h^{-1}}\right].
   \end{equation}
\end{theorem}
Since we do not impose any coverage assumption on the batch data, Theorem \ref{thm:main} demonstrates an ``instance-dependent'' sub-optimality gap, which can be controlled by some general term (the right hand side of \eqref{eq:thm_main}) and the confidence level $\delta$. The sub-optimality gap largely depends on the quality of the batch data --- specifically, how well it explores the feature space within $\mathbb{R}^d$. Building upon the above foundational result, we proceed to examine the sample complexity required to achieve an $\epsilon$-optimal policy, considering different data qualities in the subsequent discussion.

\subsubsection{The case of partial feature coverage}\label{sec:alg_partial}
We first consider the partial feature coverage, which only compares the behavior policy with the optimal policy, in terms of the ability to explore each feature dimension. To accommodate with our Lin-RMDPs, we propose a tailored partial data coverage assumption, which depicts the worst-case dissimilarity between the optimal robust policy $\pi^{\star}$ in any transition kernel $P\in\Prho(P^0)$ and the behavior policy $\pi^{\textb}$ in the nominal kernel $P^0$ over every feature space dimension $i\in[d]$, as follows.
\begin{assumption}[Robust single-policy clipped concentrability]\label{assump:Crob}
    The behavior policy of the batch dataset $\D$ satisfies 
    \begin{equation}\label{eq:Crob}
        \max_{(u,h,i,P)\in\mathbb{R}^d\times[H]\times[d]\times\Prho(P^0)} \frac{u^{\top} \left(\min\{\E_{d_h^{\star,P}}\phi^2_i(s,a),1/d\}\cdot \1_{i,i}\right) u}{u^{\top}\left(\mathbb{E}_{d^{\textb}_h}[\phi(s,a)\phi(s,a)^{\top}]\right)u} \le \frac{\Crob}{d},
    \end{equation}
    for some finite quantity $\Crob \in [1,\infty)$. In addition, we follow the convention $0/0=0$.
\end{assumption}

The quantity $\Crob$ measures the expected quality of batch data, by comparing to the desired dataset associated with the optimal robust policy. Intuitively, $\Crob$ decreases as the batch dataset contains more expert data, and increases when the quality of the dataset is poorer --- generated from some policy far from the optimal policy.  Note that prior knowledge of $\Crob$ is not required when implementing \alg in practice. Here, we assume $\Crob < \infty$, which requires that the behavior policy is able to explore the same feature dimensions as the optimal robust policy. 
Under Assumption \ref{assump:Crob}, the following corollary provides the provable sample complexity of \alg to achieve an $\epsilon$-optimal robust policy. The proof is postponed to Appendix \ref{proof:corollary_sufficient_coverage}. 

\begin{corollary}[Partial feature coverage]\label{corollary:sufficient_coverage}
    With Theorem \ref{thm:main} and Assumption \ref{assump:Crob} hold and consider any $\delta\in (0,1)$. Let $d_{\min}^{\textb} = \min_{h,s,a}\{d_h^{\textb}(s,a):d_h^{\textb}(s,a)>0\}$. With probability exceeding $1-4\delta$,  the policy $\widehat{\pi}$ returned by Algorithm \ref{alg:DRLSVI} achieves 
    \begin{align*}
        \subopt(\widehat{\pi};\zeta,\Prho)&\le 96\sqrt{\frac{d^2H^4\Crob\log(3HK/\delta)}{K}}
    \end{align*}
    as long as $K\ge c_0 \log(KH/\delta)/d_{\min}^{\textb}$ for some sufficiently large universal constant $c_0>0$. In other words, the learned policy $\widehat{\pi}$ is $\epsilon$-optimal if the total number of sample trajectories satisfies
    \begin{equation}\label{eq:sample_complexity_partial}
        K \ge \widetilde{O}\Big(\frac{\Crob d^2 H^4}{\epsilon^2}\Big).
    \end{equation}
\end{corollary}
Notice that Corollary \ref{corollary:sufficient_coverage} implies the sub-optimality bound of \alg is comparable to that of the prior art in standard linear MDPs \citep[Corollary 4.5]{jin2021pessimism}, in terms of the feature dimension $d$ and the horizon length $H$.
 \paragraph{Comparisons to prior art for Lin-RMDPs.} 
    With high probability, the existing Assumption 6.2 proposed in \citet{blanchet2023double} can be transferred into the following condition (see \eqref{eq:corollary1_Nh_sa}-\eqref{eq:Lambda_h_lowerbound} in Appendix \ref{proof:corollary_sufficient_coverage}): 
    \begin{equation}\label{eq:CCrob_doublerob}
        \max_{(u,h,i,P)\in\mathbb{R}^d\times[H]\times[d]\times\Prho(P^0)} \frac{u^{\top} \left(\E_{d_h^{\star,P}}\phi^2_i(s,a)\cdot \1_{i,i}\right) u}{u^{\top}\left(\mathbb{E}_{d^{\textb}_h}[\phi(s,a)\phi(s,a)^{\top}]\right)u} \le C^{\star}_{1} \in[1,\infty).
    \end{equation}

Thanks to the clipping operation in \eqref{eq:Crob}, $\Crob \le d\cdot C^{\star}_{1}$ for any given batch dataset $\D$. Notice that both $\Crob$ and $C^{\star}_{1}$ are lower bounded by $1$. 
 The proposed algorithm in \citet{blanchet2023double} can achieve  $\epsilon$-accuracy provided that the total number of sample trajectories obeys
\begin{equation*}
    K\ge \widetilde{O}\Big(\frac{C^{\star}_{1}d^4H^4}{\epsilon^2}\Big).
\end{equation*}
It shows that the sample complexity \eqref{eq:sample_complexity_partial} of \alg improves the one in \citet{blanchet2023double} by at least $\widetilde{O}(d)$.

\subsubsection{The case of  full feature coverage}\label{sec:alg_full}
Then, we introduce the following full feature coverage assumption that is also widely used in standard offline linear MDPs \citep{jin2021pessimism,xiong2023nearly,yin2022near,ma2023distributionally}, which requires the behavior policy exploring the feature space uniformly well.
\begin{assumption}[Well-explored data coverage]\label{assump:full_coverage} We assume $\kappa = \min_{h\in [H]} \lambda_{\min} (\E_{d_h^{\textb}} [\phi(s,a)\phi(s,a)^{\top}])>0$.
\end{assumption}
Compared to Assumption \ref{assump:Crob}, Assumption \ref{assump:full_coverage} necessitates the behavior policy to be more exploratory to reach every feature dimension, which is a stronger assumption requiring full coverage of all feature dimensions.
The following corollary provides the sample complexity guarantee of \alg~under the full feature coverage, where the proof is postponed to Appendix \ref{proof:corollary_full_coverage}.
\begin{corollary}[Full feature coverage]\label{corollary:full_coverage}
    With Theorem \ref{thm:main} and Assumption \ref{assump:full_coverage} hold and consider any $\delta\in (0,1)$. Let $d_{\min}^{\textb} = \min_{h,s,a}\{d_h^{\textb}(s,a):d_h^{\textb}(s,a)>0\}$.
    With probability at least $1-5\delta$,  the policy $\widehat{\pi}$ returned by Algorithm \ref{alg:DRLSVI} achieves 
    \begin{equation*}
        \subopt(\widehat{\pi};\zeta,\Prho)\le 96\sqrt{\frac{dH^4\log(3HK/\delta)}{\kappa K}},
    \end{equation*}
    as long as $K\ge \max\{c_0 \log(2Hd/\delta)/\kappa^2,c_0\log(KH/\delta)/d_{\min}^{\textb}\}$ for some sufficiently large universal constant $c_0$.  In other words, the learned policy $\widehat{\pi}$ is $\epsilon$-optimal if the total number of sample trajectories satisfies
    \begin{equation}\label{eq:sample_complexity_full}
        K \ge \widetilde{O}\Big(\frac{d H^4}{\kappa\epsilon^2}\Big).
    \end{equation}
\end{corollary}
Notice that the sample complexity in \eqref{eq:sample_complexity_full} matches the prior arts in standard linear MDPs \citep{yin2022near,xiong2023nearly} when robustness is not considered. A careful reader may observe that \eqref{eq:sample_complexity_full} has better dependency on $d$ compared to  \eqref{eq:sample_complexity_partial}. While noting that the upper bound of $\kappa$ is $1/d$ \citep{wang2020statistical}, the sample complexity of \alg (cf. \eqref{eq:sample_complexity_full}) is at least $\widetilde{O}(d^2H^4/\epsilon^2)$. 

%% file: section/variance_est.tex
\section{Tightening the Sample Complexity by Leveraging Variance Estimation}

To tighten our results, we further explore the variance information to reweight the ridge regression in \alg and develop its variance-aware variant called \algvar. The key idea is to design a tighter penalty term with the estimated variance,  which is widely used in standard linear MDPs \citep{zhou21variance,min2021variance,yin2022near,xiong2023nearly} to achieve near-optimal results.

\subsection{\algvar: a variance-aware variant of \alg}\label{sec:var_est}
We first highlight the design features of \algvar that are different from \alg, which can boiled down to the following two steps. 

\paragraph{Constructing a variance estimator.} 
First, we run the Algorithm \ref{alg:DRLSVI} on a sub-dataset $\tilde{\D}^0\in\D$ to obtain the estimated value function $\{\tilde{V}_h\}_{h=1}^{H+1}$. 
Then with  $\{\tilde{V}_h\}_{h=1}^{H+1}$ at our hands, we design the variance estimator $\varest^2_h: \S\times \A\to [1,H^2]$ by 
\begin{equation}\label{eq:var_est}
    \varest^2_h(s,a)=\max\{[\phi(s,a)^{\top}\nu_{h,1}]_{[0,H^2]} -\left([\phi(s,a)^{\top}\nu_{h,2}]_{[0,H]}\right)^2,1\}, \quad\forall (s,a,h)\in\S\times\A\times[H],
\end{equation}
where $\nu_{h,1}$ and $\nu_{h,2}$ are obtained from ridge regression as follows:
\begin{align}
    \nu_{h,1} 
    &= \arg\min_{\nu\in\mathbb{R}^d} \sum_{\tau\in \Dtilde^0_h} \left(\phi(s_h^{\tau},a_h^{\tau})^{\top}\nu - ({\tilde{V}_{h+1}}(s_{h+1}^{\tau}))^2 \right)^2 + \|\nu\|_2^2,\label{eq:var_estimator_1}\\
    \nu_{h,2}  &= \arg\min_{\nu\in\mathbb{R}^d} \sum_{\tau\in\Dtilde^0_h} \left(\phi(s_h^{\tau},a_h^{\tau})^{\top}\nu - {\tilde{V}_{h+1}}(s_{h+1}^{\tau})\right)^2 + \|\nu\|_2^2.\label{eq:var_estimator_2}
\end{align}
Notice that the technique of variance estimation is also used in \cite{xiong2023nearly,yin2022near} for standard linear MDPs, while we address the temporal dependency via a carefully designed three-fold subsampling approach detailed in Appendix \ref{appendix:three_fold}.
\paragraph{Variance-weighted ridge regression.}
Incorporating with the variance estimator $\varest^2_h$ computed on $\tilde{\D}^0$, we replace the ridge regression updates (i.e. \eqref{eq:update_theta} and \eqref{eq:update_bar_nu}) by their variance-weighted counterparts as follows: 
\begin{align}
    &\widehat{\theta}_{h}^{\sigma} = \arg\min_{\theta\in\mathbb{R}^d} \sum_{\tau\in\D_h^0} \frac{\big(\phi(s_h^\tau,a_h^{\tau})^{\top}\theta - r_h^\tau \big)^2}{\varest^2_h(s_h^{\tau},a_h^{\tau})} + \lambda_1\|\theta\|_2^2 = \Sigma_h^{-1} \Big( \sum_{\tau \in \D^0_h} \frac{\phi(s_h^{\tau}, a_h^{\tau}) r_h^{\tau}}{\varest^2_h(s_h^{\tau},a_h^{\tau})} \Big), \label{eq:update_theta_var}\\
    &\bar{\nu}_h^{\rho,\sigma,{V}}(\alpha) = \arg\min_{\nu\in\mathbb{R}^d}\! \sum_{\tau\in\D_h^0}\!\!\frac{\big(\phi(s_h^{\tau},a_h^{\tau})^{\top}\nu - [{{V}}]_\alpha(s_{h+1}^{\tau})\big)^2}{\varest^2_h(s_h^{\tau},a_h^{\tau})} \!+\!\lambda_1\|\nu\|_2^2= \Sigma_h^{-1}\! \Big(\!\sum_{\tau\in\D^0_h}\!\!\!\frac{\phi(s_h^{\tau},a_h^{\tau})[{{V}}]_{\alpha}(s_{h+1}^\tau)}{\varest^2_h(s_h^{\tau},a_h^{\tau})}\Big), \label{eq:update_bar_nu_var}
\end{align}
for any value function $V:\S\to[0,H]$ and $h\in[H]$, where $$\Sigma_h \defeq \sum_{\tau\in \D_h^0} \dfrac{\phi(s_h^{\tau},a_h^{\tau})\phi(s_h^{\tau},a_h^{\tau})^{\top}}{\varest^2_h(s_h^{\tau},a_h^{\tau})}+ \lambda_1 I_d$$ with the regularization coefficient $\lambda_1$ and $\D^0$ is another sub-dataset of $\D$ that is independent of $\tilde{\D}^0$. 
Accordingly, \algvar constructs an empirical variance-aware robust Bellman operator as
\begin{equation}\label{eq:empirical_bellman_var}
    (\widehat{\B}^{\rho,\sigma}_h {V}) (s,a)= \phi(s,a)^{\top}(\widehat{\theta}_h^{\sigma}+\widehat{\nu}_h^{\rho,\sigma,{V}}),
\end{equation}
where the $i$-th coordinate of $\widehat{\nu}_h^{\rho,\sigma,{V}}$ is defined as
\begin{equation}
    \widehat{\nu}_{h,i}^{\rho,\sigma,{V}} = \max_{\alpha \in [\min_s {{V}}(s),\max_s {{V}}(s)]} \big\{ \bar{\nu}_{h,i}^{\rho,\sigma,{V}}(\alpha) - \rho(\alpha - \min_{s'} {V}(s'))\big\}.\label{eq:update_nu_var}
\end{equation}

Similar to \alg, we also perform the pessimistic value iterations, where the estimated $Q$ function is updated by
\begin{equation*}
    \bar{Q}_h(s,a) = \widehat{\B}^{\rho,\sigma}_h(\widehat{V}_{h+1}) (s,a)-\underset{\text{penalty function}~\Gamma_h^{\sigma}:\S\times\A\to\mathbb{R}}{\underbrace{\gamma_1\sum_{i=1}^d \|\phi_i(s,a)\1_i\|_{\Sigma_h^{-1}}}},\quad\forall (s,a,h)\in \S\times \A\times[H],
\end{equation*}
with a fixed, estimated $\widehat{V}_{h+1}$ obtained from previous iteration and some coefficient $\gamma_1>0$. 

The rest of \algvar follows the procedure described in Algorithm \ref{alg:DRLSVI}. To avoid redundancy,  the detailed implementation of \algvar is provided in Appendix \ref{sec:implementation_algvar}.

\subsection{Performance guarantees of \algvar}

Then, we are ready to present our improved results, where the proof is postponed to Appendix \ref{proof:thm_var}. 
\begin{theorem}\label{thm:var_est}
    Suppose that Assumption \ref{assump:linear-mdp}, \ref{assump:d_rectangular}, and \ref{assump:full_coverage} hold and consider any $\delta\in (0,1)$. In \algvar, we set 
    \begin{equation*}
        \lambda_1 = 1/H^2, \gamma_1=\xi_1\sqrt{d}, \quad\text{where}~\xi_1 = 66\log(3HK/\delta).
    \end{equation*}
     Provided that $\sqrt{d}\ge H$ and $K\ge \widetilde{O}\big(\frac{H^4+ H^6d\kappa}{\kappa^2} + \frac{1}{d_{\min}^{\textb}}\big)$, then with probability exceeding $1-\delta$, the robust policy $\widehat{\pi}$ learned by \algvar~satisfies
    \begin{align}\label{eq:thm_var}
        \subopt(\widehat{\pi};\zeta,\Prho)
        &\le \widetilde{O}(\sqrt{d})\sum_{h=1}^H\sum_{i=1}^d \max_{P\in\Prho(P^0)}\E_{\pi^{\star}, P}\left[\|\phi_i(s_h,a_h)\1_i\|_{(\Sigma_h^{\star})^{-1}} \right],
       \end{align}
       where the variance-weighted cumulative sample covariance matrix $\Sigma_h^{\star}$ is defined by 
       \begin{equation}\label{eq:Sigma_star}
        \Sigma_h^{\star} = \sum_{\tau\in \D_h^0} \dfrac{\phi(s_h^{\tau},a_h^{\tau})\phi(s_h^{\tau},a_h^{\tau})^{\top}}{\max\{1,\Var_{P_{h}^0}[V^{\star,\rho}_{h+1}](s_h^{\tau},a_h^{\tau})\}}+\frac{1}{H^2} I_d,
    \end{equation}
    and $\Var_{P_{h}^0}[V](s,a) = \int_{\S}P^0_{h,s,a}(s')V^2(s') \intd s' - (\int_{\S}P^0_{h,s,a}(s')V(s') \intd s')^2$ represents the  conditional variance for any value function $V:\S\to[0,H]$ and any $(s,a,h)\in \S\times\A\times[H]$.
\end{theorem}

Compared to the instance-dependent sub-optimality bound of \alg (cf. Theorem \ref{thm:main}), the above guarantee for \algvar is tighter since 
$H^2 \Lambda_h^{-1} \succeq (\Sigma_h^{\star})^{-1}$. The underlying reason for the improvement is the use of the variance estimator to control the conditional variance and the tighter penalty function designed via the Bernstein-type inequality, while that of \alg depends on the Hoeffding-type counterpart.

%% file: appendix/preliminaries.tex
\section{Technical Lemmas}
\input{appendix/proof_linearity.tex}
\input{appendix/technical_lemmas.tex}

%% file: appendix/proof_linearity.tex
\subsection{Proof of Lemma \ref{lemma:linearity}}\label{proof:lemma_linearity}
    For a Lin-RMDP satisfying Assumption \ref{assump:d_rectangular}, the robust linear transition operator defined in \eqref{eq:robust_transition_operator} obeys: for any time step $h\in[H]$, 
        \begin{align*}
            (\mathbb{P}^{\rho}_hf)(s, a) &= \inf_{\mu_h \in \Urho(\mu_h^0)} \int_{\S} \phi(s, a)^{\top} \mu_h(s') f(s') \intd s' \\
            &= \inf_{\mu_h \in \Urho(\mu_h^0)}  \sum_{i = 1}^d \phi_{i}(s, a) \int_{\S} \mu_{h, i}(s') f(s') \intd s' \\
            &= \sum_{i = 1}^d \phi_{i}(s, a) \inf_{\mu_{h, i}\in \Urho(\mu_{h,i}^{0})} \int_\S \mu_{h, i}(s') f(s') \intd s' \\
            &= \phi(s, a)^{\top} \left( \inf_{\mu_h \in \Urho(\mu_h^0)} \int_\S \mu_h(s') f(s') \intd s' \right),
        \end{align*}
where the penultimate equality is due to $\phi_i(s, a) \geq 0, \forall (i,s, a) \in [d]\times\S \times \A$ and $\Urho(\mu_{h}^0) \defeq \otimes_{[d]} \Urho(\mu_{h, i}^0)$. 
Therefore, the robust linear Bellman operator defined in \eqref{eq:robust_bellman_operator} is linear in the feature map $\phi$: for $(h,s,a)\in[H]\times\S\times\A$,
        \begin{align*}
            (\mathbb{B}^\rho_h f)(s, a) &= r_h(s, a) + [\mathbb{P}^{\rho}_hf](s, a) \\
            &=\phi(s, a)^{\top}\theta_h + \phi(s, a)^{\top} \left( \inf_{\mu_h \in \Urho(\mu_h^0)} \int_\S \mu_h(s') f(s') \intd s' \right)\\
            &= \phi(s, a)^{\top}\underset{\defeq w^{\rho}_h}{\underbrace{\left(\theta_h + \inf_{\mu_h\in \Urho(\mu_h^0)} \int_{ \S} \mu_h(s') f(s')\intd s'\right)}},
        \end{align*}
where the second equality is due to Assumption \ref{assump:linear-mdp}.

%% file: appendix/technical_lemmas.tex
\subsection{Preliminary facts}
\subsubsection{Useful concenrtation inequalities}

\begin{lemma}[Hoeffding-type inequality for self-normalized process \citep{abbasi2011improved}] \label{lemma:Hoeffding}
    Let $\{\eta_t\}_{t=1}^{\infty}$ be a real-valued stochastic process and let $\{\mathcal{F}_t\}_{t=0}^{\infty}$ be a filtration such that $\eta_t$ is $\mathcal{F}_t$-measurable. Let $\{x_t\}_{t=1}^{\infty}$ be an $\mathbb{R}^d$-valued stochastic process where $x_t$ is $\mathcal{F}_{t-1}$ measurable and $x_t\le L$. Let $\Lambda_t = \lambda I_d+ \sum_{s=1}^t x_sx_s^{\top}$. Assume that conditioned on $\F_{t-1}$, $\eta_t$ is mean-zero and $R$-sub-Gaussian. Then, for any $\delta>0$, with probability at least $1-\delta$, for all $t>0$, we have
    \begin{equation*}
        \|\sum_{s=1}^tx_s\eta_s\|_{\Lambda_t^{-1}} \le R\sqrt{d\log(1+tL/\lambda) + 2\log(1/\delta)}.
    \end{equation*}
\end{lemma}

\begin{lemma}[Bernstein-type inequality for self-normalized process \citep{zhou21variance}] \label{lemma:Bernstein}
    Let $\{\eta_t\}_{t=1}^{\infty}$ be a real-valued stochastic process and let $\{\mathcal{F}_t\}_{t=0}^{\infty}$ be a filtration such that $\eta_t$ is $\mathcal{F}_t$-measurable. Let $\{x_t\}_{t=1}^{\infty}$ be an $\mathbb{R}^d$-valued stochastic process where $x_t$ is $\mathcal{F}_{t-1}$ measurable and $x_t\le L$. Let $\Lambda_t = \lambda I_d+ \sum_{s=1}^t x_sx_s^{\top}$. Assume that 
    \begin{equation*}
        |\eta_t|\le R, \quad \mathbb{E}[\eta_t|\F_{t-1}] = 0,\quad \mathbb{E}[\eta_t^2|\F_{t-1}]\le\sigma^2.
    \end{equation*}
Then for any $\delta>0$, with probability at least $1-\delta$, for all $t>0$, we have
\begin{equation*}
    \|\sum_{s=1}^tx_s\eta_s\|_{\Lambda_t^{-1}} \le 8\sigma \sqrt{d\log\left(1+\frac{tL^2}{\lambda d}\right) \log (\frac{4t^2}{\delta})} + 4R\log (\frac{4t^2}{\delta}).
\end{equation*}
\end{lemma}

\begin{lemma}[Lemma H.5, \cite{min2021variance}]\label{lemma:H5_Min}
    Let $\phi:\S\times\A\to \mathbb{R}^d$ be a bounded function such that $\|\phi(s,a)\|_2\le C$ for all $(s,a)\in\S\times\A$. For any $K>0$ and $\lambda>0$, define $\bar{G}_K = \sum_{k=1}^K \phi(s_k,a_k)\phi(s_k,a_k)^{\top} + \lambda I_d$ where $(s_k,a_k)$ are i.i.d. samples from some distribution $\nu$ over $\S\times\A$. Let $G = \mathbb{E}_v[\phi(s,a)\phi(s,a)^{\top}]$. Then for any $\delta\in(0,1)$, if $K$ satisfies that
    \begin{equation*}
        K\ge\max\{512C^4\|G^{-1}\|^2\log(2d/\delta),4\lambda\|G^{-1}\|\}.
    \end{equation*}
    Then with probability at least $1-\delta$, it holds simultaneously for all $u\in\mathbb{R}^d$ that
    \begin{equation*}
        \|u\|_{\bar{G}_k^{-1}} \le \frac{2}{\sqrt{K}}\|u\|_{G^{-1}}.
    \end{equation*}
\end{lemma}

\begin{lemma}[Lemma 5.1, \cite{jin2021pessimism}]\label{lemma:jin21}
    Under the condition that with probability at least $1-\delta$, the penalty function $\Gamma_h:\S\times \A\to\mathbb{R}$ in Algorithm \ref{alg:DRLSVI} and satisfying
    \begin{equation}\label{eq:penalty_condition}
        | (\widehat{\B}_h^{\rho}\widehat{V}_{h+1})(s,a) - (\B_h^{\rho} \widehat{V}_{h+1})(s,a)|\le \Gamma_h(s,a),\quad\forall (s,a,h)\in \S\times\A\times[H],
    \end{equation}  
    we have 
    \begin{equation*}
        0\le \iota_h(s,a)\le 2\Gamma_h(s,a), \quad\forall (s,a,h)\in \S\times\A\times[H].
    \end{equation*}
    \end{lemma}

\subsubsection{Useful facts}
\begin{lemma}\label{prop:max_difference}
    For any function $f_1:\mathcal{C}\subseteq\mathbb{R} \to \mathbb{R}$ and $f_2:\mathcal{C}\subseteq\mathbb{R} \to \mathbb{R}$, we have
    \begin{equation*}
        \max_{\alpha\in \mathcal{C}} f_1(\alpha) - \max_{\alpha\in \mathcal{C}} f_2(\alpha) \le \max_{\alpha\in \mathcal{C}} \left(f_1(\alpha) - f_2(\alpha)\right).
    \end{equation*}
\end{lemma}
\begin{proof}
    Let $\alpha_1^{\star} = \arg\max_{\alpha\in\mathcal{C}} f_1(\alpha)$. Then,
    \begin{align*}
        \max_{\alpha\in \mathcal{C}} f_1(\alpha) - \max_{\alpha\in \mathcal{C}} f_2(\alpha) &\le f_1(\alpha_1^{\star}) - \max_{\alpha\in \mathcal{C}} f_2(\alpha)\\
        &\le f_1(\alpha_1^{\star}) -  f_2(\alpha_1^{\star})\\
        & \le  \max_{\alpha\in \mathcal{C}} \left(f_1(\alpha) - f_2(\alpha)\right).
    \end{align*}
\end{proof}

    \begin{lemma}\label{lemma:trace}
        For any positive semi-definite matrix $A\in\mathbb{R}^{d\times d}$ and any constant $c\ge 0$, we have 
        \begin{equation}
            \Tr\left(A(I+cA)^{-1}\right) \le \sum_{i=1}^{d} \frac{\lambda_i}{1+c\lambda_i},
        \end{equation}
        where $\{\lambda_i\}_{i=1}^{d}$ are the eigenvalues of $A$ and $\Tr(\cdot)$ denotes the trace of the given matrix.
    \end{lemma}
    \begin{proof}
        Note that
        \begin{align*}
            A(I+cA)^{-1} &= A(I+cA)^{-1} + \frac{1}{c}(I+cA)^{-1} -\frac{1}{c}(I+cA)^{-1}\\
            &= \frac{1}{c}I - \frac{1}{c}(I+cA)^{-1}.
        \end{align*}
        In addition, the eigenvalues of $(I+cA)^{-1}$ are $\{\frac{1}{1+ c\lambda_i}\}_{i=1}^d$. Therefore,
        \begin{equation*}
            \Tr\left(A(I+cA)^{-1}\right) =\sum_{i=1}^d \frac{\lambda_i}{1+c\lambda_i}.
        \end{equation*}
    \end{proof}

    \begin{lemma}[modified Lemma 4, \cite{shi2023curious}]\label{lemma:TV_strong}
        Consider any probability vector $\mu^0\in\Delta(\S)$, any fixed uncertainty level $\rho$ and the uncertainty set $\Urho(\mu^0)$ satisfying Assumption \ref{assump:d_rectangular}. For any value function $V:\S\to[0,H]$, recalling the definition of $[V]_{\alpha}$ in \eqref{eq:nu}, one has
        \begin{equation}
            \inf_{\mu\in\Urho(\mu^0)} \int_{\S}\mu(s')V(s')\intd s' = \max_{\alpha\in[\min_s V(s),\max_s V(s)]} \{\E_{s'\sim\mu^0}[V]_{\alpha}(s') - \rho (\alpha-\min_{s'}[V]_{\alpha}(s'))\}.
        \end{equation}
    \end{lemma}

%% file: appendix/proof_thm1.tex
\section{Analysis for \alg: Algorithm~\ref{alg:DRLSVI}} 

\subsection{Proof of equation \eqref{eq:TV_dual_linear}}\label{sec:TV_dual}
Recall that for any $(s,a,h)\in\S\times\A\times[H]$, one has
\begin{align*}
    (\bellman^{\rho}_hV)(s, a) &= r_h(s, a) + \inf_{\mu_{h} \in \Urho(\mu_h^0)} \int_{\S} \phi(s, a)^{\top} \mu_h(s') V(s')\intd s'\\
    &= \phi(s,a)^{\top}\theta_h + \sum_{i=1}^d \phi_i(s,a)\inf_{\mu_{h,i}\in\Urho(\mu_{h,i}^0)}\int_{\S} \mu_{h,i}(s')V(s')\intd s'.
\end{align*}
According to Lemma \ref{lemma:TV_strong}, for any value function $V:\S\to [0,H]$ and any uncertainty set $\Urho(\mu_{h,i}^0),\forall (h,i)\in[H]\times[d]$ that satisfies Assumption \ref{assump:d_rectangular}, one has
\begin{align*}
    \inf_{\mu_{h,i}\in\Urho(\mu_{h,i}^0)}\int_S \mu_{h,i}(s')V(s')\intd s' &= \max_{\alpha\in [\min_s V(s),\max_s V(s)]} \big\{ \int_{\S} \mu_{h,i}(s') [V]_{\alpha}(s')\intd s' - \rho(\alpha - \min_{s'} [V]_{\alpha}(s'))\big\}.
\end{align*}

Denote $\nu_{h}^{\rho,V} = [\nu_{h,1}^{\rho,V},\nu_{h,2}^{\rho,V}\ldots,\nu_{h,d}^{\rho,V}]^{\top}\in\mathbb{R}^d$, where $\nu_{h, i}^{\rho,V}=\max_{\alpha\in [\min_s V(s),\max_s V(s)]} \big\{ \int_{\S} \mu_{h,i}(s') [V]_{\alpha}(s')\intd s' - \rho(\alpha - \min_{s'} [V]_{\alpha}(s'))\big\}$ for any $i\in[d]$. Therefore,
\begin{equation*}
    (\bellman^{\rho}_hV)(s, a) = \phi(s, a)^{\top} \big(  \theta_h + \nu_h^{\rho,V} \big).
\end{equation*}

\subsection{Two-fold subsampling method} \label{sec:two-fold}

To tackle the temporal dependency in batch dataset $\D$, we apply the insights from the subsampling approach inspired by \cite{li2022settling}. The key idea is to utilize half of the data to establish a valid lower bound of the number of samples, which is employed to achieve the statistical independence in the remaining half of the dataset.
The detailed implementation of the two-fold subsampling can be summarized in the Algorithm \ref{alg:two_fold_subsampling}.
\begin{algorithm}[!htbp]
    \caption{\twofold}\label{alg:two_fold_subsampling}
    \begin{algorithmic}[1]
    \Input Batch dataset $\mathcal{D}$; 
    \State \textbf{Split Data:} Split $\D$ into two haves $\Dmain$ and $\Daux$, where $|\Dmain| = |\Daux| = {N_h}/2$. Denote $\Nmain_h(s)$ (resp. $\Naux_h(s)$) as the number of sample transitions from state $s$ at step $h$ in  $\Dmain$ (resp. $\Daux$).
    \State \textbf{Construct the high-probability lower bound $\Ntrim_h(s)$ by $\Daux$:} For each $s\in\S$ and $1\le h\le H$, compute 
    \begin{equation}\label{eq:update_Ntrim}
        \Ntrim_h(s) = \max\{\Naux_h(s) - 10\sqrt{\Naux_h(s) \log\frac{KH}{\delta},0}\}.
    \end{equation}
    \State \textbf{Construct the almost temporally statistically independent $\Dtrim$:}  Let $\Dmain_h(s)$ denote the dataset containing all transition-reward sample tuples at the current state $s$ and step $h$ from $\Dmain$. For any $(s,h)\in \S\times [H]$, subsample $\min\{\Ntrim_h(s),\Nmain_h(s)\}$ transition-reward sample tuples randomly from $\Dmain_h(s)$, denoted as $\Dmainsub$.
    \Output  $\D^0 =\Dmainsub$.
    \end{algorithmic}
    \label{alg:two-sampling}
    \end{algorithm}

    Recall that we assume the sample trajectories in $\D$ are generated independently. Then, the following lemma shows that \eqref{eq:update_Ntrim} is a valid lower bound of $\Nmain_h(s)$ for any $s\in \S$ and $h\in [H]$, which can be obtained by a slight modification on Lemma 3 and Lemma 7 in \cite{li2022settling}

\begin{lemma}\label{lemma:Ntrim_bound}
    With probability at least $1-2\delta$, if $\Ntrim_h(s)$ satisfies \eqref{eq:update_Ntrim} for every $s\in\S$ and $h\in[H]$, then $\D^0\defeq \Dmainsub$ contains temporally statistically independent samples and the following bound holds simultaneously, \ie, 
    \begin{align*}
        \Ntrim_h(s)\le \Nmain_h(s),\qquad \forall (s,h)\in \S\times [H].
    \end{align*}
    In addition, with probability at least $1-3\delta$, the following lower bound also holds, \ie,
    \begin{equation*}
        \Ntrim_h(s,a) \ge \frac{Kd_h^b(s,a)}{8} - 5\sqrt{Kd_h^b(s,a)\log(\frac{KH}{\delta})}, \quad\forall (s,a,h)\in \S\times\A\times[H].
    \end{equation*}
\end{lemma}
\begin{proof}
    Let $\S_D$ be the collection of all the states appearing in any batch dataset $\D$, where $|\S_D|\le K$. By changing the union bound over $\S$ to $\S_D$ in the proof of \citet[Lemma 3]{li2022settling}, the remaining proof still holds, since $\Ntrim_h(s) = \Nmain_h(s)=0, \forall s\notin \S_D$. Together with \citet[Lemma 7]{li2022settling}, $\D^0$ contains temporally statistically independent samples if $ \Ntrim_h(s)\le \Nmain_h(s), \forall (s,h)\in \S\times [H]$.
\end{proof}

\subsection{Proof of Theorem \ref{thm:main}}\label{proof:thm_main}

\paragraph{Notations.}
Before starting the proof of Theorem \ref{thm:main}, we introduce several notations for the convenience of the following analysis. First, we use
\begin{align}\label{eq:model_eval_err}
    \iota_h(s, a) = \mathbb{B}^{\rho}_h \widehat{V}_{h + 1}(s, a) - \widehat{Q}_h(s, a),\quad\forall (h,s,a)\in[H]\times\S\times\A,
    \end{align}
to represent the model evaluation error at the $h$-th step of \alg. In addition, we denote the estimated weight of the transition kernel at the $h$-th step by
    \begin{equation}\label{eq:hat_mu}
      \forall (s,h) \in \S \times [H]: \quad \widehat{\mu}_{h}(s) = \Lambda_{h}^{-1} \left( \sum_{\tau\in \D^0_h} \phi(s_h^{\tau}, a_h^{\tau}) \1(s_{h+1}^{\tau} = s)\right)\in\mathbb{R}^{d},
    \end{equation}
where $\indicator(\cdot) $ is the indicator function.
Accordingly, it holds that $\bar{\nu}_h^{\widehat{V}}(\alpha)= \int_\S \widehat{\mu}_{h}(s') [\widehat{V}_{h+1}(s')]_{\alpha}\intd s' \in \mathbb{R}^d$. 
We denote the set of all the possible state occupancy distributions associated with the optimal policy $\pi^{\star}$ and any $P\in\Prho(P^0)$ as
\begin{align}
    \D_h^{\star} &= \left\{ \left[d_h^{\star,P}(s)\right]_{s\in\S}:P\in\Prho(P^0)\right\} = \left\{ \left[d_h^{\star,P}(s,\pi_h^{\star}(s))\right]_{s\in\S}:P\in\Prho(P^0)\right\},\label{eq:D_h^star}
\end{align}
for any time step $h\in[H]$. 
\subsubsection{Proof sketch for Theorem \ref{thm:main}}
We first claim that Theorem \ref{thm:main} holds as long as the following theorem can be established.
\begin{theorem}\label{thm:main_iid}
    Consider $\delta\in(0,1)$.
    Suppose that the dataset $\D_0$ in Algorithm \ref{alg:DRLSVI} contains $N_h<K$ transition-reward sample tuples at every $h\in[H]$. Assume that conditional on $\{N_h\}_{h\in[H]}$, all the sample tuples in $\D^0_h$ are statistically independent. Suppose that Assumption \ref{assump:linear-mdp} and \ref{assump:d_rectangular} hold. In \alg, we set 
    \begin{equation}\label{eq:thm3_parameters}
        \lambda_0 = 1, \quad \gamma_0 = 6\sqrt{d\xi_0}H,\quad \text{where }\xi_0 = \log(3HK/\delta).
       \end{equation}
       Here, $\delta\in(0,1)$ is the confidence parameter and $K$ is the upper bound of $N_h$ for any $h\in[H]$. Then, $\{\widehat{\pi}_h\}_{h=1}^H$ generated by Algorithm \ref{alg:DRLSVI}, with the probability at least $1-\delta$, satisfies
       \begin{equation*}
        \subopt(\widehat{\pi};\zeta,\Prho)\le \tilde{O}(\sqrt{d}H)\sum_{h=1}^H \sum_{i=1}^d \max_{d_h^{\star}\in \D_h^{\star}}\mathbb{E}_{d_h^{\star}}\left[\|\phi_i(s_h,a_h)\1_i\|_{\Lambda_h^{-1}} \right].
       \end{equation*}
\end{theorem}
As the construction in Algorithm \ref{alg:two_fold_subsampling}, $\{\Ntrim_h(s)\}_{s\in\S,h\in[H]}$ is computed using $\Daux$ that is independent of $\D^0$. From Lemma \ref{lemma: Ntrim_bound_var}, $\Ntrim_h(s)$ is a valid sampling number for any $s\in\S$ and $h\in[H]$ such that $|\D^0_h| = \sum_{s\in\S} \Ntrim_h(s) \le K$, and $\D^0_h$ can be treated as containing temporally statistically independent samples with probability exceeding $1-2\delta$. Therefore, by invoking Theorem \ref{thm:main_iid} with $N_h \defeq |\D^0_h|$, we have 
\begin{align*}
    \subopt(\widehat{\pi};\zeta,\Prho)
    \le \tilde{O}(\sqrt{d}H)\sum_{h=1}^H \sum_{i=1}^d \max_{d_h^{\star}\in \D_h^{\star}}\mathbb{E}_{d_h^{\star}}\left[\|\phi_i(s_h,a_h)\1_i\|_{\Lambda_h^{-1}} \right],
   \end{align*}
with probability exceeding $1-3\delta$.

 \subsubsection{Proof of Theorem \ref{thm:main_iid}}\label{proof:thm_main_iid}

The proof of Theorem \ref{thm:main_iid} can be summarized as following key steps.
\paragraph{Step 1: establishing the pessimistic property.} To substantiate the pessimism, we heavily depend on the following lemma, where the proof is postponed to Appendix \ref{proof:lemma_key_lemma}.
\begin{lemma}\label{lemma:key_lemma}
    Suppose all the assumptions in Theorem \ref{thm:main_iid} hold and follow all the parameters setting in \eqref{eq:thm3_parameters}. Then for any $(s,a,h)\in\S\times\A\times[H]$, with probability at least $1-\delta$, the value function $\{\widehat{V}\}_{h=1}^{H}$ generated by \alg satisfies
    \begin{equation}\label{eq:important_event}
        | (\widehat{\mathbb{B}}_{h}^{\rho} \widehat{V}_{h + 1})(s, a) - (\mathbb{B}_h^{\rho} \widehat{V}_{h + 1})(s, a) | \le \Gamma_h(s,a) \defeq  \gamma_0\sum_{i=1}^d \|\phi_i(s,a)\1_i\|_{\Lambda_h^{-1}}.
    \end{equation}
\end{lemma}
In the following, we will show that the following relations hold:
\begin{equation}\label{eq:pessimism}
    Q^{\star,\rho}_h(s,a) \ge Q_h^{\widehat{\pi},\rho}(s,a) \ge \widehat{Q}_h(s,a) \quad \text{and} \quad  V^{\star,\rho}_h(s) \ge V_h^{\widehat{\pi},\rho}(s)  \ge  \widehat{V}_h(s),\quad\forall (s,a,h)\in\S\times\A\times[H],
\end{equation}
if the condition \eqref{eq:important_event} holds. It implies that $\widehat{Q}_h(s,a)$ and $\widehat{V}_h(s)$ is the pessimistic estimates of $Q_h^{\widehat{\pi},\rho}(s,a)$ and $ V_h^{\widehat{\pi},\rho}(s)$ for any $s\in\S$, respectively. Notice that if $Q_h^{\widehat{\pi},\rho}(s,a) \ge \widehat{Q}_h(s,a)$ for all $(s,a)\in\S\times\A$, one simultaneously has the following relation:
\begin{equation*}
    V_h^{\widehat{\pi},\rho}(s)  =  Q_h^{\widehat{\pi},\rho}(s,\widehat{\pi}_h(s)) \ge \widehat{Q}_h(s,\widehat{\pi}_h(s))= \widehat{V}_h(s),\qquad \forall (s,h)\in\S\times[H].
\end{equation*}
Therefore, we shall verify that 
\begin{equation}\label{eq:pessimism_Q}
    Q_h^{\widehat{\pi},\rho}(s,a) \ge \widehat{Q}_h(s,a),\quad\forall (s,a)\in\S\times\A,
\end{equation}
by induction, and $V_h^{\widehat{\pi},\rho}(s)  \ge  \widehat{V}_h(s)$ will spontaneously hold for $s\in\S$.
\begin{itemize}
    \item \textit{At step} $h=H+1$: From the initialization step in Algorithm \ref{alg:DRLSVI}, we have $Q_{H+1}^{\widehat{\pi},\rho}(s,a) = \widehat{Q}_{H+1}(s,a) = 0$, for any $(s,a)\in\S\times\A$, and \eqref{eq:pessimism_Q} holds.
    \item \textit{For any step $h\le H$}: Suppose $Q_{h+1}^{\widehat{\pi},\rho}(s,a) \ge \widehat{Q}_{h+1}(s,a)$. From \eqref{eq:pessimism}, we have $V_{h+1}^{\widehat{\pi},\rho}(s)  \ge  \widehat{V}_{h+1}(s)$. 
    Therefore, if $\widehat{Q}_h(s,a) = 0$,  $Q_h^{\widehat{\pi},\rho}(s,a) \ge 0 = \widehat{Q}_h(s,a)$, for any $(s,a)\in\S\times\A$. 
    Otherwise, 
    \begin{align*}
        \widehat{Q}_h(s,a) &\le (\widehat{\B}_h^{\rho}\widehat{V}_{h+1})(s,a) - \Gamma_h(s,a)\\
        & = (\B_h^{\rho}\widehat{V}_{h+1})(s,a) + (\widehat{\B}_h^{\rho}\widehat{V}_{h+1})(s,a) - (\B_h^{\rho}\widehat{V}_{h+1})(s,a) -  \Gamma_h(s,a)\\
        &\le (\B_h^{\rho}\widehat{V}_{h+1})(s,a) + |(\widehat{\B}_h^{\rho}\widehat{V}_{h+1})(s,a) - (\B_h^{\rho}\widehat{V}_{h+1})(s,a)| -  \Gamma_h(s,a)\\
        &\le (\B_h^{\rho}\widehat{V}_{h+1})(s,a) + \Gamma_h(s,a) - \Gamma_h(s,a)\\
        &\le (\B_h^{\rho}\widehat{V}_{h+1})(s,a)\\
        &\le (\B_h^{\rho}V_{h+1}^{\widehat{\pi},\rho})(s,a)  = Q_h^{\widehat{\pi},\rho}(s,a),
    \end{align*}
    where the first inequality is from the definition of $\widehat{Q}_h(s,a)$ (cf. Line 10 in Algorithm \ref{alg:DRLSVI}), and third inequality is based on the condition \eqref{eq:important_event}.
\end{itemize}
Combining these two cases, for any $h\in[H+1]$, we could verify the pessimistic property, i.e. the equation \eqref{eq:pessimism}.

\paragraph{Step 2: bounding the suboptimality gap.}
Notice that for any $h\in[H]$ and any $s\in\S$,
\begin{align}
    V_h^{\star,\rho}(s) -  V_h^{\widehat{\pi},\rho}(s) 
    &= V_h^{\star,\rho}(s)  -  \widehat{V}_h(s) + \widehat{V}_h(s) - V_h^{\widehat{\pi},\rho}(s)\nonumber\\
    &\le V_h^{\star,\rho}(s)  -  \widehat{V}_h(s),\label{eq:suboptimality_gap_h_1}
\end{align}
where the inequality is due to the pessimistic property in \eqref{eq:pessimism}. 

In the following, we will control the value difference, \ie, $V_h^{\star,\rho}(s)  -  \widehat{V}_h(s)$, for any $(s,h)\in\S\times[H]$. First, recall the definition of $\widehat{V}_h$ (cf. Line 12 in Algorithm \ref{alg:DRLSVI}) and the robust Bellman consistency equation \eqref{eq:robust_bellman_consistency}. For all $s\in\S$, 
\begin{align}
 V_h^{\star,\rho}(s) - \widehat{V}_h(s)
 & = Q^{\star,\rho}_h(s, \pi^{\star}_h(s)) -  \widehat{Q}_h(s,  \widehat{\pi}_h(s))\nonumber\\
 &\le  Q^{\star,\rho}_h(s, \pi^{\star}_h(s)) - \widehat{Q}_h(s,   \pi^{\star}_h(s))\label{eq:V_difference0},
\end{align}
where the last inequality is from $\widehat{\pi}_h$ is the greedy policy with respect to $\widehat{Q}_h$ (cf. Line 11 in Algorithm \ref{alg:DRLSVI}.)
From the definition of the model evaluation error (i.e., equation \eqref{eq:model_eval_err}) and the robust Bellman optimality equation \eqref{eq:robust_bellman_optimality_Q}, we have
\begin{align*}
    \widehat{Q}_h(s,a) &= (\bellman_h^\rho \widehat{V}_{h+1})(s,a) - \iota_h(s,a), & &\forall (s,a)\in \S\times\A, \\
    Q_h^{\star,\rho}(s,a) 
    &= (\bellman^{\rho}_h V^{\star,\rho}_{h+1})(s,a) & &\forall (s,a)\in \S\times\A,
\end{align*}
which leads to
\begin{equation}\label{eq:Q_difference0}
     Q_h^{\star,\rho}(s, \pi_h^{\star}(s)) -\widehat{Q}_h(s,\pi_h^{\star}(s))  = (\mathbb{P}^{\rho}_h V_{h+1}^{\star,\rho})(s,\pi_h^{\star}(s)) - (\mathbb{P}^{\rho}_h\widehat{V}_{h+1})(s,\pi_h^{\star}(s))  + \iota_h(s,\pi_h^{\star}(s)), \quad\forall s\in \S.
\end{equation}
Denote 
\begin{align}
    P_{h,s,\pi_h^{\star}(s)}^{\inf,\widehat{V}}(\cdot) \defeq \argmin_{P(\cdot) \in \Prho(P^0_{h,s,\pi_h^{\star}(s)})} \int_{S}P(s')\widehat{V}_{h+1}(s') \intd s'.
\end{align}
 Therefore, \eqref{eq:Q_difference0} becomes
\begin{equation}\label{eq:Q_difference1}
     Q_h^{\star,\rho}(s, \pi_h^{\star}(s)) - \widehat{Q}_h(s,\pi_h^{\star}(s)) \le \int_{\S} P_{h,s,\pi_h^{\star}(s)}^{\inf,\widehat{V}}(s')\left(V_{h+1}^{\star,\rho}(s') - \widehat{V}_{h+1}(s')\right)\intd s' + \iota_h(s,\pi_h^{\star}(s)), ~~\forall (s,a)\in \S\times\A.
\end{equation}
Substituting \eqref{eq:Q_difference1} into \eqref{eq:V_difference0}, one has 
\begin{equation}
    V_h^{\star,\rho}(s) - \widehat{V}_h(s) \le \int_{\S}P_{h,s,\pi_h^{\star}(s)}^{\inf,\widehat{V}}(s')\left(V_{h+1}^{\star,\rho}(s') - \widehat{V}_{h+1}(s')\right)\intd s' + \iota_h(s,\pi_h^{\star}(s)).
\end{equation}
For any $h\in[H]$, define $\widehat{P}^{\inf}_{h,s}:\S\to\S$ and $\iota_h^{\star}\in \S\to\mathbb{R}$ by
\begin{equation}
    \widehat{P}^{\inf}_{h} (s) = P^{\inf,\widehat{V}}_{h,s,\pi_h^{\star}(s)}(\cdot)\quad \text{and} \quad\text{and}\quad
    \iota_h^{\star}(s) \defeq \iota_h(s,\pi^{\star}(s)),\quad\forall s\in\S.
\end{equation}
By telescoping sum, we finally obtain that for any $s\in\S$,
    \begin{align*}
        V_h^{\star,\rho}(s) - \widehat{V}_h(s) = \langle \1_s, V_h^{\star,\rho} - \widehat{V}_h \rangle &\le \left(\prod_{t=h}^{H}\widehat{P}^{\inf}_j\right)\left(V_{H+1}^{\star,\rho} - \widehat{V}_{H+1}\right)(s) + 
            \sum_{t=h}^H  \left(\prod_{j=h}^{t-1} \widehat{P}^{\inf}_j\right) \iota_t^{\star}(s)\\
        &= \sum_{t=h}^H  \left(\prod_{j=h}^{t-1} \widehat{P}^{\inf}_j\right) \iota_t^{\star}(s),
    \end{align*}
where the equality is from $V_{H+1}^{\star,\rho}(s)=\widehat{V}_{H+1}(s)=0$ and $ \left(\prod_{j=t}^{t-1} \widehat{P}^{\inf}_j\right)(s) = \1_s$.

\paragraph{Step 3: finishing up.} For any $d_h^{\star}\in\D_h^{\star}$, denote
\begin{align*}
    d_{h:t}^{\star} = d_h^{\star}\left(\prod_{j=h}^{t-1} \widehat{P}^{\inf}_j\right)\in\D_t^{\star}.
\end{align*}
Together with \eqref{eq:suboptimality_gap_h_1}, the sub-optimality gap defined in \eqref{eq:suboptimality_gap} satisfies
\begin{equation}\label{eq:suboptimality_two_term}
    \subopt (\widehat{\pi};\zeta,\P^{\rho}) \le \E_{s_1\sim\zeta} V_1^{\star,\rho}(s_1)  -  \E_{s_1\sim\zeta} \widehat{V}_1(s_1) \le \sum_{t=1}^H  \E_{s_t\sim d_{1:t}^{\star}}\iota_t^{\star}(s_t).
 \end{equation}
 For any $h\in[H]$, we let $\Gamma_h^{\star}:\S\to\mathbb{R}$ satisfy
 \begin{equation}
    \Gamma_h^{\star}(s) = \Gamma_h(s,\pi_h^{\star}(s)),\quad\forall s\in\S.
 \end{equation}
 Combining Lemma \ref{lemma:jin21} together with Lemma \ref{lemma:key_lemma} will lead to
 \begin{align*}
    \subopt (\widehat{\pi};\zeta,\P^{\rho}) 
    &\le 2\sum_{h=1}^H \E_{s_h\sim d_{1:h}^{\star}} \Gamma_h^{\star}(s_h).
\end{align*}
Note that $\Gamma_h^{\star}(s) = \gamma_0\sum_{i=1}^d\|\phi_i(s,\pi_h^{\star}(s))\1_i\|_{\Lambda_h^{-1}}$ for any $(s,h)\in\S\times[H]$.
 Following the definition \eqref{eq:D_h^star}, we have $d_{1:h}^{\star}\in \D_h^{\star}$ and correspondingly,
 \begin{align*}
    \subopt (\widehat{\pi};\zeta,\P^{\rho}) 
    &\le2\sum_{h=1}^H \E_{s_h\sim d_{1:h}^{\star}} \Gamma_h^{\star}(s_h)\\
    &\le 2\gamma_0\sum_{h=1}^H  \max_{d_h^{\star}\in \D_h^{\star}}\mathbb{E}_{(s_h,a_h)\sim d_h^{\star}}  \left[ \sum_{i=1}^d\|\phi_i(s_h,a_h)\1_i\|_{\Lambda_h^{-1}}\right],
\end{align*} 
with probability exceeding $1-\delta$.

%% file: appendix/proof_thm1_lemmas.tex
\subsubsection{Proof of Lemma \ref{lemma:key_lemma}}\label{proof:lemma_key_lemma}
To control $| (\widehat{\mathbb{B}}_{h}^{\rho} \widehat{V}_{h + 1})(s, a) - (\mathbb{B}_h^{\rho} \widehat{V}_{h + 1})(s, a)|$ for any $(s,a,h)\in\S\times\A\times[H]$, we first show the following lemma, where the proof can be found in Appendix \ref{proof:control_suboptimality_gap}.
    \begin{lemma}\label{lemma:control_suboptimality_gap}
        Suppose Assumption \ref{assump:linear-mdp} and \ref{assump:d_rectangular} hold. Then, for any $(s,a,h)\in \S\times\A\times [H]$, the estimated value function $\widehat{V}_{h+1}$ generated by \alg satisfies
        \begin{equation}\label{eq:control_suboptimality_gap}
        \begin{aligned}
            &| (\widehat{\mathbb{B}}_{h}^{\rho} \widehat{V}_{h + 1})(s, a) - (\mathbb{B}_h^{\rho} \widehat{V}_{h + 1})(s, a) |\\
            &\le \left(\sqrt{\lambda_0 d} H + \underset{T_{1,h}}{\underbrace{\max_{\alpha\in [\min_s \widehat{V}_{h+1}(s),\max_s \widehat{V}_{h+1}(s)]}  \|\sum_{\tau\in \D_h^0}\phi(s_h^{\tau},a_h^\tau) \epsilon_{h}^{\tau}(\alpha,\widehat{V}_{h+1})\|_{\Lambda_h^{-1}}}}\right) \sum_{i=1}^d\|\phi_i(s,a)\1_i\|_{\Lambda_h^{-1}},
        \end{aligned}
        \end{equation}
        where $\epsilon_h^{\tau}(\alpha,V) = \int_{\S} P^0_h(s'|s_h^{\tau},a_h^{\tau}) [V]_{\alpha}(s') \intd s'- [ V]_{\alpha}(s_{h+1}^{\tau})$ for any value function $V:\S\to[0,H]$, $\alpha \in [\min_s V(s),\max_s V(s)]$ and $\tau\in\D_h^0$.
    \end{lemma}
   We observe that the second term (\ie, $T_{1,h}$) in \eqref{eq:control_suboptimality_gap} will become dominating, as long as $\lambda_0$ is sufficiently small. In the following analysis, we will control
   $T_{1,h}$ via uniform concentration and the concentration of self-normalized process.
    
    Notice that $\alpha$ and $\widehat{V}_{h+1}$ are coupled with each other, which makes controlling $T_{1,h}$ intractable. To this end, we propose the minimal $\epsilon_0$-covering set for $\alpha$. Since $\widehat{V}_{h+1}(s)\in [0,H]$ for any $s\in \S$, we construct $\N(\epsilon_0,H)$ as  the minimal $\epsilon_0$-cover of the $[0,H]$ whose size satisfies $|\N(\epsilon_0,H)|\le \frac{3 H}{\epsilon_0}$. In other words, for any $\alpha \in [0,H]$, there exists $\alpha^{\dag}\in \N(\epsilon_0,H)$, we have
    \begin{equation}\label{eq:def_covering_set_for_alpha}
        |\alpha - \alpha^{\dag}| \le \epsilon_0.
    \end{equation}
    Then we can rewrite $T_{1,h}^2$ as 
    \begin{align}
       T_{1,h}^2 = & \max_{\alpha\in [\min_s \widehat{V}_{h+1}(s),\max_s \widehat{V}_{h+1}(s)]} \|\sum_{\tau\in \D^0_h}\phi(s_h^{\tau},a_h^\tau) \left(\epsilon_{h}^{\tau}(\alpha,\widehat{V}_{h+1}) - \epsilon_{h}^{\tau}(\alpha^{\dag},\widehat{V}_{h+1})+\epsilon_{h}^{\tau}(\alpha^{\dag},\widehat{V}_{h+1})\right)\|_{\Lambda_h^{-1}}^2 \nonumber\\
        &\le \max_{\alpha\in [0,H]} 2\|\sum_{\tau\in \D^0_h}\phi(s_h^{\tau},a_h^\tau) \left(\epsilon_{h}^{\tau}(\alpha,\widehat{V}_{h+1}) - \epsilon_{h}^{\tau}(\alpha^{\dag},\widehat{V}_{h+1})\right)\|_{\Lambda_h^{-1}}^2 +2\|\sum_{\tau\in \D^0_h}\phi(s_h^{\tau},a_h^\tau)\epsilon_{h}^{\tau}(\alpha^{\dag},\widehat{V}_{h+1})\|_{\Lambda_h^{-1}}^2\nonumber\\
        &\le 8\epsilon_0^2 K^2/\lambda_0 + 2\underset{T_{2,h}}{\underbrace{\|\sum_{\tau\in \D^0_h}\phi(s_h^{\tau},a_h^\tau)\epsilon_{h}^{\tau}(\alpha^{\dag},\widehat{V}_{h+1})\|_{\Lambda_h^{-1}}^2}}, \label{eq:diff_alpha}
    \end{align}
    for some $\alpha^{\dag}\in \N(\epsilon_0,H)$, where the proof of the last inequality is postponed to Appendix \ref{proof:diff_alpha}. Alternatively, 
    \begin{equation}\label{eq:T_2_uniform}
        T_{2,h}\le \sup_{\alpha\in \N(\epsilon_0,H)} \|\sum_{\tau\in \D^0_h}\phi(s_h^{\tau},a_h^\tau)\epsilon_{h}^{\tau}(\alpha,\widehat{V}_{h+1})\|_{\Lambda_h^{-1}}^2.
    \end{equation}
    
    Noted that the samples in $\D^0$ are temporally statistically independent, i.e., $\widehat{V}_{h+1}$ is independent of $\D_h^0$, or to say, $\widehat{\mu}_h$. Therefore, we can directly control $T_{2,h}$ via the following lemma. 
    \begin{lemma}[Concentration of self-normalized process]\label{lemma:self_normal}
    Let $V:\S\to[0,H]$ be any fixed vector that is independent with $\widehat{\mu}_h$ and $\alpha\in[0,H]$ be a fixed constant. For any fixed $h\in[H]$ and any $\delta\in(0,1)$, we have
    \begin{equation*}
        P_{\mathcal{D}}\left(\|\sum_{\tau\in \D^0_h} \phi(s_h^{\tau},a_h^\tau)\epsilon_{h}^{\tau}(\alpha,V)\|^2_{\Lambda_h^{-1}}>H^2\left(2\cdot\log(1/\delta)+d\cdot\log(1+{N_h}/\lambda_0)\right)\right)\le \delta.
    \end{equation*}
    \end{lemma}
    The proof of Lemma \ref{lemma:self_normal} is postponed to Appendix  \ref{proof:self_normal}.
    Then applying Lemma \ref{lemma:self_normal} and the union bound over $\N(\epsilon_0,H)$, we have 
    \begin{align*}
        &P_{\D}\left(\sup_{\alpha\in \N(\epsilon_0,H)} \|\sum_{\tau\in \D^0_h} \phi(s_h^{\tau},a_h^\tau)\epsilon_{h}^{\tau}(\alpha,\widehat{V}_{h+1})\|^2_{\Lambda_h^{-1}} \ge H^2(2\log(H |\mathcal{N}(\epsilon_0,H)|/\delta) + d\log(1+{N_h}/\lambda_0))\right)\\
        &\le \delta/H,
    \end{align*}
    for any fixed $h\in [H]$.
    According to \cite{vershynin2018high}, one has $|\N(\epsilon_0,H)|\le\frac{3H}{\epsilon_0}$. Taking the union bound for any $h\in [H]$, we arrive at
    \begin{equation}\label{eq:bound_T3}
    \begin{aligned}
        \sup_{\alpha\in \N(\epsilon_0,H)} \|\sum_{\tau\in \D^0_h} \phi(s_h^{\tau},a_h^\tau)\epsilon_{h}^{\tau}(\alpha,\widehat{V}_{h+1})\|^2_{\Lambda_h^{-1}} &\le  2H^2\log(\frac{3H^2}{\epsilon_0\delta})+H^2d\log(1+\frac{{K}}{\lambda_0}),
    \end{aligned}
    \end{equation}
    with probability exceeding $1-\delta$, where we utilize $N_h\le K$ for every $h\in[H]$ on the right-hand side.
    
    Combining \eqref{eq:diff_alpha}, \eqref{eq:T_2_uniform} and \eqref{eq:bound_T3}, we have 
    \begin{equation*}
        T_{1,h}^2 \le 8\epsilon_0^2 {K}^2/\lambda_0 + 4H^2\log(\frac{3H^2}{\epsilon_0\delta})+2H^2d\log(1+\frac{{K}}{\lambda_0}),
    \end{equation*} 
    with probability at least $1-\delta$. Let $\epsilon_0 = H/K$ and $\lambda_0 = 1$. Then,
    \begin{align*}
        T_{1,h}^2 &\le 8H^2 + 4H^2\log(\frac{3HK}{\delta})+2H^2d\log(1+K)\\
        &\le 8H^2 + 4H^2\log(3HK/\delta) + 2dH^2\log(2K).
    \end{align*}
    Let $\xi_0 = \log(3HK/\delta)\ge 1$. Note that $\log(2K)\le \log(3HK/\delta)= \xi_0$. Then, we have
    \begin{equation*}
        T_{1,h}^2 \le 8H^2 +4H^2 \xi_0 + 2dH^2\xi_0 \le 16dH^2\xi_0.
    \end{equation*}
    Therefore, with probability exceeding $1-\delta$, one has
    \begin{align*}
        | (\widehat{\mathbb{B}}_{h}^{\rho} \widehat{V}_{h + 1})(s, a) - (\mathbb{B}_h^{\rho} \widehat{V}_{h + 1})(s, a) | &\le \left( \sqrt{d}H+ 4\sqrt{d\xi_0}H \right)\sum_{i=1}^d \|\phi_i(s,a)\|_{\Lambda_h^{-1}}\\
        &\le \gamma_0\sum_{i=1}^d \|\phi_i(s_h,a_h)\|_{\Lambda_h^{-1}} = \Gamma_h(s,a),
    \end{align*}
    where $\gamma_0 = 6\sqrt{d\xi_0}H$ and the above inequality satisfies \eqref{eq:penalty_condition}.

\subsubsection{Proof of Lemma \ref{lemma:control_suboptimality_gap}}\label{proof:control_suboptimality_gap}
It follows Lemma \ref{lemma:linearity} and \eqref{eq:empirical_bellman} that
\begin{align}
    &| (\widehat{\mathbb{B}}_{h}^{\rho} \widehat{V}_{h + 1})(s, a) - (\mathbb{B}_h^{\rho} \widehat{V}_{h + 1})(s, a) |= |\phi(s, a)^{\top}(\widehat{w}_{h}^{\rho,\widehat{V}} - w_{h}^{\rho,\widehat{V}})|\nonumber\\
    &= \underbrace{|\phi(s, a)^{\top} \left( \widehat{\theta}_h - \theta_h \right)|}_{\texttt{(i)}} + \underbrace{|\phi(s, a)^{\top} \left( \widehat{\nu}_h^{\rho,\widehat{V}} - \nu_h^{\rho,\widehat{V}} \right)|}_{\texttt{(ii)}},\qquad \forall (s,a,h)\times \S\times\A\times[H]. \label{eq:direcly_decompose_B_diff}
\end{align}

We first bound the term $\texttt{(i)}$, for any $h\in [H]$. 

By the update \eqref{eq:update_theta}, we have
\begin{align*}
    \texttt{(i)} &= |\phi(s, a)^{\top} \Lambda_h^{-1} \left( \sum_{\tau\in \D^0_h} \phi(s_h^{\tau}, a_h^{\tau}) r_h^{\tau} \right) - \phi(s, a)^{\top} \theta_h| \\
    &= |\phi(s, a)^{\top} \Lambda_h^{-1} \left(\Lambda_h - \lambda_0 I \right)\theta_h   - \phi(s, a)^{\top} \theta_h|\\
    & = |\lambda_0 \phi(s, a)^{\top} \Lambda_h^{-1} \theta_h|, 
\end{align*}
where the second equality is from Assumption \ref{assump:linear-mdp} and  \eqref{eq:Lambda}. Applying the Cauchy-Schwarz inequality leads to
\begin{equation}\label{eq:control_term_i}
    \texttt{(i)} \leq \lambda_0 \norm{\phi(s, a)}_{\Lambda_h^{-1}} \norm{\theta_h}_{\Lambda_h^{-1}} \leq \sqrt{d\lambda_0}\sum_{i=1}^d \|\phi_i(s,a)\1_i\|_{\Lambda_h^{-1}}, 
\end{equation}
where the last inequality is $\|\theta_h\|\le \sqrt{d}$ and $\|\Lambda_h^{-1}\|\le 1/\lambda_0$ such that
\begin{equation*}
    \norm{\theta_h}_{\Lambda_h^{-1}} \le \|\Lambda_h^{-1}\|^{1/2}\|\theta_h\| \le \sqrt{d/\lambda_0},\qquad \forall h\in[H].
\end{equation*}

Next, to bound the term $\texttt{(ii)}$, we define the following notations for simplicity. Let $\epsilon_h^{\tau}(\alpha,V) = \int_{\S} P^0_h(s'|s_h^{\tau},a_h^{\tau}) [V]_{\alpha}(s') \intd s'- [ V]_{\alpha}(s_{h+1}^{\tau})$ for any $V:\S\to[0,H]$ and $\alpha \in [\min_s V(s),\max_s V(s)]$. Also, we define two auxiliary functions:
\begin{align*}
    \widehat{g}_{h,i}(\alpha) &= \int_\S \widehat{\mu}_{h,i}(s') [\widehat{V}_{h+1}]_{\alpha}(s')\intd s'-\rho(\alpha - \min_{s'} [\widehat{V}_{h+1}]_{\alpha}(s')),\\
   g^0_{h,i}(\alpha)  &= \int_\S {\mu}^0_{h,i}(s') [\widehat{V}_{h+1}]_{\alpha}(s')\intd s'-\rho(\alpha - \min_{s'} [\widehat{V}_{h+1}]_{\alpha}(s')).
\end{align*}
With the new notations in hand, we can bound $\texttt{(ii)}$ by
\begin{align}
        \texttt{(ii)} &= \left\vert\sum_{i=1}^ d \phi_i(s, a) \left(\widehat{\nu}_{h,i}^{\rho,\widehat{V}} - \nu_{h,i}^{\rho,\widehat{V}} \right)\right\vert\nonumber\\
        &\le \sum_{i=1}^d \phi_i(s,a)\max_{\alpha \in [\min_s \widehat{V}_{h+1}(s),\max_s \widehat{V}_{h+1}(s)]} \left\vert\widehat{g}_{h,i}(\alpha) - g^0_{h,i}(\alpha)\right\vert\nonumber\\
        &\le \sum_{i=1}^d \phi_i(s,a) \max_{\alpha \in [\min_s \widehat{V}_{h+1}(s),\max_s \widehat{V}_{h+1}(s)]} \left\vert \int_{\S}\left(\widehat{\mu}_{h,i}(s') - \mu^0_{h,i}(s')\right)[\widehat{V}_{h+1}]_{\alpha}(s')\intd s' \right\vert\nonumber\\
        &\le \sum_{i=1}^d \max_{\alpha \in [\min_s \widehat{V}_{h+1}(s),\max_s \widehat{V}_{h+1}(s)]} \left\vert\phi_i(s,a)\int_{\S}\left(\widehat{\mu}_{h,i}(s') - \mu^0_{h,i}(s')\right)[\widehat{V}_{h+1}]_{\alpha}(s')\intd s'\right\vert,\label{eq:T_2_after_max}
    \end{align}
where the first inequality is due to \eqref{eq:nu}, \eqref{eq:update_nu} as well as Lemma \ref{prop:max_difference},
and the last inequality is based on $\phi_i(s,a)\ge 0$ for any $(s,a)\in\S\times\A$ from Assumption \ref{assump:linear-mdp}. 
Moreover,
\begin{align}
    & \left\vert\int_\S {\mu}^0_{h,i}(s') [\widehat{V}_{h+1}]_{\alpha}(s')\intd s'- \int_\S \widehat{\mu}_{h,i}(s') [\widehat{V}_{h+1}]_{\alpha}(s')\intd s'\right\vert \nonumber\\
    &= \left\vert \int_\S {\mu}^0_{h,i}(s') [\widehat{V}_{h+1}]_{\alpha}(s')\intd s'- \1_i^{\top}\Lambda_h^{-1} \left(\sum_{\tau\in \D^0_h}\phi(s_h^{\tau},a_h^{\tau})[\widehat{V}_{h+1}]_{\alpha}(s_{h+1}^\tau)\right)\right\vert\nonumber\\
    &= \left\vert \1_i^{\top}\Lambda_{h}^{-1} \left(\Lambda_h \int_\S {\mu}^0_{h}(s') [\widehat{V}_{h+1}]_{\alpha}(s')\intd s' - \sum_{\tau\in \D^0_h}\phi(s_h^{\tau},a_h^{\tau})[\widehat{V}_{h+1}]_{\alpha}(s^\tau_{h+1})\right)\right\vert\nonumber\\
    &= \left\vert\1_i^{\top}\Lambda_h^{-1}\left[\lambda_0\int_{\S} {\mu}^0_{h}(s') [\widehat{V}_{h+1}]_{\alpha}(s')\intd s' + \sum_{\tau\in \D^0_h}\phi(s_h^{\tau},a_h^\tau)\left(\int_{\S} P_h^0(s'|s_h^{\tau},a_h^{\tau})[\widehat{V}_{h+1}]_{\alpha}(s')\intd s'  - [\widehat{V}_{h+1}]_{\alpha}(s_{h+1}^\tau)\right) \right]\right\vert\nonumber\\
    &= \left\vert\1_i^{\top}\Lambda_h^{-1}\left[\lambda_0\int_{\S} {\mu}^0_{h}(s') [\widehat{V}_{h+1}]_{\alpha}(s')\intd s' + \sum_{\tau\in \D^0_h}\phi(s_h^{\tau},a_h^\tau)\epsilon_h^{\tau}(\alpha,\widehat{V}_{h+1}) \right]\right\vert\label{eq:muV_esterror1}
\end{align}
where the first equality is from \eqref{eq:hat_mu}, the third one is due to \eqref{eq:Lambda} and we let
$$\epsilon_h^{\tau}(\alpha,V) = \int_{\S} P^0_h(s'|s_h^{\tau},a_h^{\tau}) [V]_{\alpha}(s') \intd s'- [ V]_{\alpha}(s_{h+1}^{\tau}),$$ for any $V:\S\to[0,H]$ and $\alpha \in [\min_s V(s),\max_s V(s)]$. 
Then, we have
\begin{align}
    &\left\vert\phi_i(s,a)\cdot(\widehat{\mu}_{h,i} - \mu^0_{h,i})[\widehat{V}_{h+1}]_{\alpha}\right\vert \nonumber\\
    &\le \left\vert \phi_i(s,a)\1_i^{\top}\Lambda_h^{-1}\left(\lambda_0\int_{\S} {\mu}^0_{h}(s') [\widehat{V}_{h+1}]_{\alpha}(s')\intd s' + \sum_{\tau\in \D^0_h}\phi(s_h^{\tau},a_h^\tau) \epsilon_h^\tau(\alpha,\widehat{V}_{h+1}) \right) \right\vert\nonumber\\
    &\le \|\phi_i(s,a)\1_i\|_{\Lambda_h^{-1}} \left( \underset{\texttt{(iii)}}{\underbrace{\lambda_0 \Big\|\int_{\S} {\mu}^0_{h}(s') [\widehat{V}_{h+1}]_{\alpha}(s')\intd s'\Big\|_{\Lambda_h^{-1}}}} + \Big\|\sum_{\tau\in \D^0_h}\phi(s_h^{\tau},a_h^\tau)\epsilon_h^\tau(\alpha,\widehat{V}_{h+1})\Big\|_{\Lambda_h^{-1}}\right),\label{eq:max_term}
\end{align}
where the last inequality holds due to the Cauchy-Schwarz inequality.

Moreover, the term \texttt{(iii)} in \eqref{eq:max_term} can be further simplified to
\begin{equation*}
    \texttt{(iii)} \le \lambda_0\|\Lambda_h^{-1}\|^{\frac{1}{2}} \| \int_{\S} {\mu}^0_{h}(s') [\widehat{V}_{h+1}]_{\alpha}(s')\intd s'\| \le \sqrt{\lambda_0}H,
\end{equation*}
since $|\widehat{V}_{h+1}(s)|\le H$ for any $s\in\S$ and $\|\Lambda_h^{-1}\| \le 1/\lambda_0$.
Then we have 
\begin{equation}\label{eq:decomposeT2}
    \texttt{(ii)} \le \left(\sqrt{\lambda_0}H +\max_{\alpha\in [\min_s \widehat{V}_{h+1}(s),\max_s \widehat{V}_{h+1}(s)]}  \|\sum_{\tau\in \D^0_h}\phi(s_h^{\tau},a_h^\tau) \epsilon_{h}^{\tau}(\alpha,\widehat{V}_{h+1})\|_{\Lambda_h^{-1}}\right) \sum_{i=1}^d \|\phi_i(s,a)\1_i\|_{\Lambda_h^{-1}},
\end{equation}
for any $(s,a,h)\in\S\times\A\times[H]$. Combining \eqref{eq:decomposeT2} with \eqref{eq:control_term_i} finally leads to \eqref{eq:control_suboptimality_gap}, which completes our proof.

\subsubsection{Proof of \eqref{eq:diff_alpha}}\label{proof:diff_alpha}
Since $\epsilon_h^{\tau}(\alpha,V)$ is $2$-Lipschitz with respect to $\alpha$ for any $V:\S\to[0,H]$, i.e.
\begin{align*}
    |\epsilon_{h}^{\tau}(\alpha,V) - \epsilon_h^{\tau}(\alpha^{\dag},V)|
    \le & 2|\alpha-\alpha^{\dag}|\le 2\epsilon_0.
\end{align*}
Therefore, for any $\alpha\in[0,H]$, one has
\begin{align*}
    &\|\sum_{\tau\in \D^0_h} \phi(s_h^{\tau},a_h^\tau)\left(\epsilon_{h}^{\tau}(\alpha,V) - \epsilon_h^{\tau}(\alpha^{\dag},V)\right)\|^2_{\Lambda_h^{-1}}\\
    =& \sum_{\tau,\tau'\in \D^0_h} \phi(s_h^{\tau},a_h^\tau)^{\top}\Lambda_h^{-1}\phi(s_h^{\tau'},a_h^{\tau'})\left[\left(\epsilon_{h}^{\tau}(\alpha,V) - \epsilon_h^{\tau}(\alpha^{\dag},V)\right)\left(\epsilon_{h}^{\tau'}(\alpha,V) - \epsilon_h^{\tau'}(\alpha^{\dag},V)\right)\right]\\
    \le&\sum_{\tau,\tau'\in \D^0_h} \phi(s_h,a_h^\tau)^{\top}\Lambda_h^{-1}\phi(s_h^{\tau'},a_h^{\tau'})\cdot 4\epsilon_0^2 \\
    \le & 4\epsilon_0^2 N_h^2/\lambda_0,
\end{align*}
where the last inequality is based on $\|\phi(s,a)\|\le 1$ and $\lambda_{\min}(\Lambda_h) \ge \lambda_0$ for any $(s,a,h)\in\S\times\A\times[H]$ such that
\begin{equation}\label{eq:quad_phi_Lambda}
\begin{aligned}
    \sum_{\tau,\tau'\in \D^0_h} \phi(s_h^{\tau},a_h^\tau)^{\top}\Lambda_h^{-1}\phi(s_h^{\tau'},a_h^{\tau'}) &= \sum_{\tau,\tau'\in \D^0_h} \|\phi(s_h^{\tau},a_h^\tau)\|_2\cdot\|\phi(s_h^{\tau'},a_h^{\tau'})\|_2\cdot\|\Lambda_h^{-1}\|\le  N_h^2/\lambda_0.
\end{aligned}
\end{equation}
Thus,
\begin{equation*}
    \max_{\alpha\in [0,H]} \|\sum_{\tau\in \D^0_h}\phi(s_h^{\tau},a_h^\tau) \left(\epsilon_{h}^{\tau}(\alpha,\widehat{V}_{h+1}) - \epsilon_{h}^{\tau}(\alpha^{\dag},\widehat{V}_{h+1})\right)\|_{\Lambda_h^{-1}}^2 \le 4\epsilon_0^2 N_h^2/\lambda_0 \le 4\epsilon_0^2 K^2/\lambda_0,
\end{equation*}
due to the fact $N_h\le K$ for any $h\in[H]$, which completes the proof of \eqref{eq:diff_alpha}.

\subsubsection{Proof of Lemma \ref{lemma:self_normal}}\label{proof:self_normal}
For any fixed $h\in [H]$ and $\tau\in \D_h^0$, we define the $\sigma$-algebra
\begin{equation*}
    \F_{h,\tau} = \sigma(\{(s_h^j,a_h^j)\}_{j=1}^{(\tau+1)\wedge |N_h|},\{r_h^j,s_{h+1}^j\}_{j=1}^\tau).
\end{equation*}
As shown in \citet[Lemma B.2]{jin2021pessimism}, for any $\tau\in\D_h^0$, we have $\phi(s_h^\tau,a_h^{\tau})$ is $\F_{h,\tau-1}$-measurable and $\epsilon_h^{\tau}(\alpha,V)$ is $\F_{h,\tau-1}$-measurable. Hence $\{\epsilon_h^{\tau}(\alpha,V)\}_{\tau\in\D_h^0}$ is stochastic process adapted to the filtration $\{\F_{h,\tau}\}_{\tau\in\D_h^0}$. Then, we have
\begin{align*}
    \E_{{\D_h^0}} [\epsilon_h^{\tau}(\alpha,V)|\F] 
    &= \int_{\S} P^0_h(s'|s_h^{\tau},a_h^{\tau}) [V]_{\alpha} - \E_{\D_h^0}\left[ [V(s_{h+1}^{\tau})]_{\alpha}|\{(s_h^j,a_h^j)\}_{j=1}^{(\tau)\wedge N_h},\{r_h^j,s_{h+1}^j\}_{j=1}^{\tau-1}\right]\\
    &= \int_{\S} P^0_h(s'|s_h^{\tau},a_h^{\tau}) [V]_{\alpha} - \E_{\D_h^0}\left[ [V(s_{h+1}^{\tau})]_{\alpha}\right] = 0.
\end{align*}
Note that $\epsilon_h^{\tau}(\alpha,V) = \int_{\S} P^0_h(s'|s_h^{\tau},a_h^{\tau}) [V]_{\alpha} - [ V(s_{h+1}^{\tau})]_{\alpha}$ for any $V\in[0,H]^{\S}$ and $\alpha \in [0,H]$. Then, we have
\begin{align*}
    |\epsilon_h^{\tau}(\alpha,V)| \le H.
\end{align*}
Hence, for the fixed $h\in[H]$ and all $\tau\in[H]$, the random variable $\epsilon_h^{\tau}(\alpha,V)$ is mean-zero and $H$-sub-Gaussian conditioning on $\F_{h,\tau-1}$. 
Then, we invoke the Lemma \ref{lemma:Hoeffding}
with $\eta_\tau = \epsilon_{h}^{\tau}(\alpha,V)$ and $x_{\tau} = \phi(s_h^{\tau},a_h^{\tau})$. For any $\delta>0$, we have
\begin{equation*}
    P_{\mathcal{D}}\left(\|\sum_{\tau\in \D^0_h} \phi(s_h^{\tau},a_h^\tau)\epsilon_{h}^{\tau}(\alpha,V)\|^2_{\Lambda_h^{-1}}> 2H^2\log (\frac{\det (\Lambda_h^{1/2})}{\delta\det (\lambda_0 I_d)^{1/2}})\right)\le \delta.
\end{equation*}
Together with the facts that $\det (\Lambda_h^{1/2}) = (\lambda_0+N_h)^{d/2}$ and $\det (\lambda_0 I_d)^{1/2} = \lambda_0^{d/2}$, we can conclude the proof of Lemma \ref{lemma:self_normal}.

%% file: appendix/proof_corollary.tex
\subsection{Proof of Corollary \ref{corollary:sufficient_coverage}} \label{proof:corollary_sufficient_coverage}
Before continuing, we introduce some additional notations that will be used in the following analysis. For any $(h,i) \in[H] \times [d]$, define $\Phi_{h,i}^{\star} : \S\to \mathbb{R}^{ d\times d}$ and $b_{h,i}^{\star}: \S\to \mathbb{R}$ by
\begin{align}
    \Phi_{h,i}^{\star}(s) &= (\phi_i(s,\pi_h^{\star}(s))\1_i)(\phi_i(s,\pi_h^{\star}(s))\1_i)^{\top} \in\mathbb{R}^{d\times d},\label{eq:Phi_hi_star}\\
    b_{h,i}^{\star}(s) &=(\phi_i(s,\pi_h^{\star}(s))\1_i)^{\top}\Lambda_h^{-1}(\phi_i(s,\pi_h^{\star}(s))\1_i)\label{eq:b_hi_star}.
\end{align}

With these notations in hand and recalling \eqref{eq:thm_main} in Theorem~\ref{thm:main}, one has
\begin{align}
  V_1^{\star,\rho}(\zeta) - V_1^{\widehat{\pi},\rho}(\zeta)
  &\le  2\gamma_0\sum_{h=1}^H \sum_{i=1}^d\sup_{d_h^{\star}\in \D_h^{\star}} \E_{s\sim d_h^{\star}}\sqrt{b_{h,i}^{\star}(s)}\nonumber\\
  &\le 2\gamma_0\sum_{h=1}^H \sum_{i=1}^d\sup_{d_h^{\star}\in \D_h^{\star}} \sqrt{ \E_{s\sim d_h^{\star}} b_{h,i}^{\star}(s)}\nonumber\\
  &= 2\gamma_0\sum_{h=1}^H \sup_{d_h^{\star}\in \D_h^{\star}} \sum_{i=1}^d\sqrt{ \E_{s\sim d_h^{\star}} b_{h,i}^{\star}(s)}, \label{eq:corollary1_Vdifference}
\end{align}
where the second inequality is due to the Jensen's inequality and concavity.

In the following, we will control the key term $\sum_{i=1}^d\sqrt{ \E_{s\sim d_h^{\star}} b_{h,i}^{\star}(s)}$ for any $d_h^{\star}\in \D_h^{\star}$. Before continuing, we first denote 
\begin{equation*}
    \mathcal{C}_h^{\textb} = \{(s,a):d_h^{\textb}(s,a)>0\}.
\end{equation*}
Considering any $(s,a)$ s.t. $d_h^{\textb}(s,a)>0$ and from Lemma \ref{lemma:Ntrim_bound}, the following lower bound holds with probability at least $1-3\delta$, \ie,
\begin{equation}\label{eq:corollary1_Nh_sa}
    N_h(s,a) \ge \frac{Kd_h^{\textb}(s,a)}{8} - 5\sqrt{Kd_h^{\textb}(s,a)\log(\frac{KH}{\delta})}\ge \frac{Kd_h^{\textb}(s,a)}{16},
\end{equation}
as long as 
\begin{equation}\label{eq:two_fould_burnin_K}
    K\ge c_0 \frac{\log(KH/\delta)}{d_{\min}^{\textb}}\ge c_0 \frac{\log(KH/\delta)}{d_h^{\textb}(s,a) }
\end{equation}
for some sufficiently large $c_0$ and $d_{\min}^{\textb} = \min_{h,s,a}\{d_h^{\textb}(s,a):d_h^{\textb}(s,a)>0\}$. Therefore, 
\begin{align*}
    \Lambda_h& = \sum_{(s,a)\in\mathcal{C}_h^{\textb}} N_h(s,a)\phi(s,a)\phi(s,a)^{\top}+I_d\\
    &\succeq \sum_{(s,a)\in\mathcal{C}_h^{\textb}}  \frac{Kd_h^{\textb}(s,a)}{16}\phi(s,a)\phi(s,a)^{\top} + I_d\\
    &\succeq \frac{K}{16} \mathbb{E}_{d^{\textb}_h}[\phi(s,a)\phi(s,a)^{\top}] + I_d.
\end{align*}
From Assumption \ref{assump:Crob}, 
        \begin{equation*}
            \mathbb{E}_{d^{\textb}_h}[\phi(s,a)\phi(s,a)^{\top}] \succeq \max_{P\in\Prho(P^0)} \frac{d\cdot\min\{\E_{d_h^{\star,P}}\phi^2_i(s,a),1/d\}}{\Crob} \1_{i,i}, \quad \forall i\in[d]
        \end{equation*}
Thus, for any $i\in [d]$,
    \begin{equation}\label{eq:Lambda_h_lowerbound}
        \Lambda_h \succeq I_d + \frac{Kd\cdot \min\{\E_{d_h^{\star}}\phi^2_i(s,\pi_h^{\star}(s)),1/d\}}{16\Crob}\cdot \1_{i,i}.
    \end{equation}
    Here, $\1_{i,j}$ represents a matrix with the $(i,j)$-th coordinate as $1$ and all other elements as $0$.
    Consequently, 
    \begin{align}\label{eq:corollary1_inner_term}
        \E_{s\sim d_h^{\star}} b_{h,i}^{\star}(s) = \E_{s\sim d_h^{\star}}\Tr(\Phi_{h,i}^{\star}(s)\Lambda_h^{-1}) &= \Tr(\E_{s\sim d_h^{\star}}\Phi_{h,i}^{\star}(s) \Lambda_h^{-1}) \nonumber\\
        &\le \frac{\E_{d_h^{\star}}\phi^2_i(s,\pi_h^{\star}(s))}{1+Kd\cdot \min\{\E_{d_h^{\star}}\phi^2_i(s,\pi_h^{\star}(s)),1/d\}/16\Crob} ,
    \end{align}
    where the second equality is because the trace is a linear mapping and the last inequality holds by Lemma \ref{lemma:trace}. We further define $\mathcal{E}_{h,\mathsf{larger}} = \{i:\E_{(s,a)\sim d_h^{\star}}\phi^2_i(s,a)\ge \frac{1}{d}\}$. Due to Assumption \ref{assump:linear-mdp}, we first claim that 
    \begin{equation}\label{eq:E_h,larger}
        |\mathcal{E}_{h,\mathsf{larger}}|\le \sqrt{d},
    \end{equation}
    where the proof can be found at the end of this subsection.

    By utilizing Assumption \ref{assump:Crob}, we discuss the following three cases.
    \begin{itemize}
    \item If $\E_{(s,a)\sim d_h^{\star}}\phi^2_i(s,a) = 0$ ($i\notin\mathcal{E}_{h,\mathsf{larger}}$), it is easily observed that \eqref{eq:corollary1_inner_term} can be controlled by
    $\langle d_h^{\star}, b_{h,i}^{\star} \rangle \le  0 $.
    \item If $0<\E_{(s,a)\sim d_h^{\star}}\phi^2_i(s,a)\le \frac{1}{d}$ ($i\notin\mathcal{E}_{h,\mathsf{larger}}$),we have
        \begin{equation}
            \eqref{eq:corollary1_inner_term} \le \frac{16\Crob \cdot \E_{d_h^{\star}}\phi^2_i(s,\pi_h^{\star}(s))}{Kd\cdot \E_{d_h^{\star}}\phi^2_i(s,\pi_h^{\star}(s))} =  \frac{16\Crob}{Kd}.
        \end{equation}
    \item If $i\in\mathcal{E}_{h,\mathsf{larger}}$, i.e., $\frac{1}{d}\le \E_{(s,a)\sim d_h^{\star}}\phi^2_i(s,a) \le 1 $, we have
        \begin{equation}
           \eqref{eq:corollary1_inner_term} \le \frac{16\Crob \cdot \E_{d_h^{\star}}\phi^2_i(s,\pi_h^{\star}(s))}{K}\le \frac{16\Crob}{K},
        \end{equation}
        where the last inequality holds due to $\phi^2_i(s,\pi_h^{\star}(s))\le 1$.
    \end{itemize}
    Summing up the above three cases and \eqref{eq:E_h,larger}, we have
    \begin{align*}
        \sum_{i=1}^d  \sqrt{\E_{s\sim d_h^{\star}} b_{h,i}^{\star}(s)}
         &\le \sum_{i\in \mathcal{E}_{h,\mathsf{larger}}}  \sqrt{\E_{s\sim d_h^{\star}} b_{h,i}^{\star}(s)} + \sum_{i\notin \mathcal{E}_{h,\mathsf{larger}}} \sqrt{\E_{s\sim d_h^{\star}} b_{h,i}^{\star}(s)} \\
        &\le |\mathcal{E}_{h,\mathsf{larger}}| \sqrt{\frac{16\Crob}{K}} + |d-\mathcal{E}_{h,\mathsf{larger}}| \sqrt{\frac{16\Crob}{Kd}}\\
        &\le 8 \sqrt{\Crob}\sqrt{\frac{d}{K}}. 
    \end{align*}
Together with \eqref{eq:corollary1_Vdifference} and setting $\gamma_0=6\sqrt{d}H\sqrt{\log(3HK/\delta)}$, one obtains
\begin{align*}
    V_1^{\star,\rho}(\zeta) - V_1^{\widehat{\pi},\rho}(\zeta)&\le  2\gamma_0\sum_{h=1}^H \sup_{d_h^{\star}\in \D_h^{\star}} \sum_{i=1}^d\sqrt{ \E_{s\sim d_h^{\star}} b_{h,i}^{\star}(s)}\\
    &\le 96dH^2\sqrt{\Crob/K}\sqrt{\log(3HK/\delta)}, 
\end{align*}
with probability at least $1-4\delta$, as long as $K\ge c_0 \frac{\log(KH/\delta)}{d_{\min}^{\textb}}$ for some universal constant $c_0$.

\paragraph*{Proof of \eqref{eq:E_h,larger}.} 
Let $\tilde{\mathcal{E}}_{h,\mathsf{larger}} = \{i:\E_{(s,a)\sim d_h^{\star}}\phi_i(s,a)\ge \frac{1}{\sqrt{d}}\}$. 
\begin{itemize}
    \item We first show that $|\tilde{\mathcal{E}}_{h,\mathsf{larger}}|$ should be no larger than $\sqrt{d}$ by contradiction. Suppose $ |\tilde{\mathcal{E}}_{h,\mathsf{larger}}| > \sqrt{d}$. Then, there are more than $\sqrt{d}$ coordinates of $\E_{(s,a)\sim d_h^{\star}}\phi(s,a)\in\mathbb{R}^d$ that is larger than $1/\sqrt{d}$. In other words,
    \begin{equation}
        \sum_{i\in \tilde{\mathcal{E}}_{h,\mathsf{larger}}} \E_{(s,a)\sim d_h^{\star}}\phi_i(s,a) > 1,
    \end{equation}
    which is equivalent to 
    \begin{equation}
        \max_{(s,a)\in\S\times\A} \|\phi(s,a)\|_1 \ge \E_{(s,a)\sim d_h^{\star}}\|\phi(s,a)\|_1 \ge \E_{(s,a)\sim d_h^{\star}} \sum_{i\in \tilde{\mathcal{E}}_{h,\mathsf{larger}}} \phi_i(s,a)  > 1,
    \end{equation}
    where the last inequality is from the linearity of the expectation mapping.
    It contradicts to our Assumption \ref{assump:d_rectangular}, which implies $\|\phi(s,a)\|_1 =  1$ for any $(s,a)\in\S\times\A\times[H]$.   
    \item Then, we show that $\tilde{\mathcal{E}}_{h,\mathsf{larger}}\subseteq \mathcal{E}_{h,\mathsf{larger}}$: For every element $i\in \tilde{\mathcal{E}}_{h,\mathsf{larger}}$, we have
        \begin{equation*}
            \frac{1}{d} \le (\E_{(s,a)\sim d_h^{\star}}\phi_i(s,a))^2  \le \E_{(s,a)\sim d_h^{\star}}\phi^2_i(s,a),
        \end{equation*}
        where the second inequality is due to the Jensen's inequality.
        Thus, $\tilde{\mathcal{E}}_{h,\mathsf{larger}}\subseteq \mathcal{E}_{h,\mathsf{larger}}$.
\end{itemize}
Combining these two arguments, we show that $|\mathcal{E}_{h,\mathsf{larger}}|\ge \sqrt{d}$.

\subsection{Proof of Corollary \ref{corollary:full_coverage}}\label{proof:corollary_full_coverage}
We first establish the following lemma to control the sub-optimality, under the full feature coverage.
\begin{lemma}\label{lemma:penalty_term_full}
    Consider $\delta\in(0,1)$.
    Suppose Assumption \ref{assump:d_rectangular}, Assumption \ref{assump:full_coverage} and all conditions in Lemma \ref{lemma:H5_Min} hold. For any $h\in[H]$, if $N_h\ge \max\{512 \log(2Hd/\delta)/\kappa^2,4/\kappa\}$, we have 
    \begin{equation*}
        \sum_{i=1}^d\|\phi_i(s,a)\1_i\|_{\Lambda_h^{-1}} \le \frac{2}{\sqrt{N_h\kappa}},\quad\forall (s,a)\in\S\times\A,
    \end{equation*}
    with probability exceeding $1-\delta$. 
\end{lemma}
\begin{proof}
    From Lemma \ref{lemma:H5_Min} and Assumption \ref{assump:full_coverage}, one has
\begin{equation*}
    \|\phi_i(s,a)\1_i\|_{\Lambda_h^{-1}} \le \frac{2\phi_i(s,a)}{\sqrt{N_h\kappa}}, \quad\forall (i,s,a)\in[d]\times\S\times\A,
\end{equation*}
as long as $N_h\ge \max\{512 \log(2Hd/\delta)/\kappa^2,4/\kappa\}$. In addition, 
\begin{equation}
    1 = \int_{\S} P^0_h(s'|s,a)\intd s' = \int_{\S} \phi(s,a)^{\top}\mu_{h}^0(s')\intd s' = \sum_{i=1}^d \phi_i(s,a)\int_{\S}\mu_{h,i}^0(s') \intd s' =  \sum_{i=1}^d \phi_i(s,a),
\end{equation}
where the last equality is implied by Assumption \ref{assump:d_rectangular}.
Therefore, 
\begin{equation*}
    \sum_{i=1}^d\|\phi_i(s,a)\1_i\|_{\Lambda_h^{-1}}\le \sum_{i=1}^d \frac{2\phi_i(s,a)}{\sqrt{N_h\kappa}} \le \frac{2}{\sqrt{N_h\kappa}}.
\end{equation*}
\end{proof}
From \eqref{eq:corollary1_Nh_sa}, we have $N_h\ge \frac{K}{16}$ with probability exceeding $1-3\delta$, as long as $K$ obeys \eqref{eq:two_fould_burnin_K}.  Together will Lemma \ref{lemma:penalty_term_full}, with probability exceeding $1-4\delta$, one has
\begin{equation*}
    \sum_{i=1}^d\|\phi_i(s,a)\1_i\|_{\Lambda_h^{-1}} \le \frac{8}{\sqrt{K\kappa}}, \forall (s,a,h)\in\S\times\A\times[H],
\end{equation*}
as long as $K\ge \max\{c_0 \log(2Hd/\delta)/\kappa^2,c_0\log(KH/\delta)/d_{\min}^{\textb}\}$ for some sufficiently large universal constant $c_0$. It follows Theorem \ref{thm:main} that
\begin{equation*}
    \subopt(\widehat{\pi};\zeta,\Prho)\le 96\sqrt{d}H^2 \sqrt{\frac{\log(3HK/\delta)}{K\kappa}},
\end{equation*}
which completes the proof.

%% file: appendix/proof_thm2_varest.tex
\section{Analysis for \algvar: Algorithm~\ref{alg:DRLSVI_V}}\label{appendix:analysis_algvar}
\input{appendix/three_fold.tex}

\subsection{Proof of Theorem \ref{thm:var_est}}\label{proof:thm_var}
To show that Theorem \ref{thm:var_est} holds, we first establish the following theorem that considers the temporally independent dataset, where the proof is deferred to the next subsection.
\begin{theorem}\label{thm:var_iid}
    Consider the dataset $\D_0$ and $\tilde{D}^0$ used for constructing the variance estimator in \algvar and $\delta\in(0,1)$. Suppose that both $\D^0$ and $\tilde{D}^0$ contain $N_h < K$ sample tuples at every $h\in[H]$. Assume that conditional on $\{N_h\}_{h\in[H]}$, the sample tuples in $\D^0$ and $\tilde{\D}^0$ are statistically independent, where 
    \begin{equation*}
        N_h(s,a)\ge \frac{Kd_h^{\textb}(s,a)}{24},\qquad \forall (s,a,h)\in \S\times\A\times[H].
    \end{equation*}
    Suppose that Assumption \ref{assump:linear-mdp}, \ref{assump:d_rectangular}, and \ref{assump:Crob} hold. In \algvar, we set 
    \begin{equation}\label{eq:thm4_parameters}
        \lambda_1 = 1/H^2, \gamma_1=\xi_1\sqrt{d}, \quad\text{where}~\xi_1 = 66\log(3HK/\delta).
    \end{equation}
    Then, with probability at least $1-7\delta$, $\{\widehat{\pi}_h\}_{h=1}^H$ generated by \algvar~satisfies
    \begin{equation*}
        \subopt(\widehat{\pi};\zeta,\Prho)\le \tilde{O}(\sqrt{d})\sum_{h=1}^H\sum_{i=1}^d \max_{d_h^{\star}\in\D_h^{\star}} \mathbb{E}_{d_h^{\star}}\left[\|\phi_i(s_h,a_h)\1_i\|_{(\Sigma^{\star}_{h})^{-1}}\right],
       \end{equation*}
    if $\sqrt{d}\ge H$ and $K\ge \max\{\tilde{O}(H^4/\kappa^2),\tilde{O}(H^6d/\kappa)\}$, where $\Sigma_h^{\star}$ is defined in \eqref{eq:Sigma_star}.
\end{theorem}
As the construction in Algorithm \ref{alg:three_fold_subsampling}, $\{\Ntrim_h(s)\}_{s\in\S,h\in[H]}$ is computed using $\Daux$ that is independent of $\D^0\defeq \Dmainsub$ and $\tilde{\D}^0\defeq \Dvarsub$. Moreover, from Lemma \ref{lemma: Ntrim_bound_var} and Lemma \ref{lemma:threefold_independent} in the Section \ref{appendix:three_fold}, $\{\Ntrim_h(s)\}_{s\in\S}$ is a valid sampling number and $\D^0_h$ and $\tilde{\D}^0_h$ can be treated as being temporally statistically independent samples and 
\begin{equation*}
    \sum_{s\in\S} \Ntrim_h(s)\ge K/24,
\end{equation*}
 with probability exceeding $1-4\delta$, as long as $K\ge c_1\log\frac{KH}{\delta}/d_{\min}^{\textb}$ for some sufficiently large $c_1$. 
 
 Therefore, by invoking Theorem \ref{thm:main_iid} with $N_h \defeq \sum_{s\in\S} \Ntrim(s)$, we have 
\begin{align*}
        \subopt(\widehat{\pi};\zeta,\Prho)
        &\le \tilde{O}(\sqrt{d})\sum_{h=1}^H\sum_{i=1}^d \max_{d_h^{\star}\in\D_h^{\star}} \mathbb{E}_{d_h^{\star}}\left[\|\phi_i(s_h,a_h)\1_i\|_{(\Sigma_{h}^{\star})^{-1}}\right],
\end{align*}
with probability exceeding $1-11\delta$, if $\sqrt{d}\ge H$ and $K\ge \max\{\widetilde{O}(H^4/\kappa^2),\widetilde{O}(H^6d/\kappa),\widetilde{O}(1/d_{\min}^{\textb})\}$.

\subsection{Proof of Theorem \ref{thm:var_iid}}
Before starting, we first introduce some notations that will be used in the following analysis. First, we use
\begin{align}\label{eq:model_eval_err_var}
    \iota_h^{\sigma}(s, a) = \mathbb{B}^{\rho,\sigma}_h \widehat{V}_{h + 1}(s, a) - \widehat{Q}_h(s, a),\quad\forall (s,a,h)\in\S\times\A\times[H],
    \end{align}
to represent the model evaluation error at the $h$-th step of our proposed Algorithm \ref{alg:DRLSVI_V}. In addition,
For any $h\in[H]$, we let $\Gamma_h^{\star,\sigma}:\S\to\mathbb{R}$ satisfy
 \begin{equation}
    \Gamma_h^{\star,\sigma}(s) = \Gamma_h^{\sigma}(s,\pi_h^{\star}(s)),\quad\forall s\in\S.
 \end{equation}
 Also, denote
 $\mathbb{V}_{P_h^0} V(s,a) =  \max\{1,\Var_{P_{h}^0}[V](s,a)\}$ for any $V:\S\to[0,H]$ and any $(s,a,h)\in \S\times\A\times[H]$.
Similar to Lemma \ref{lemma:key_lemma}, we have the following key lemma, where the proof can be found in Appendix \ref{proof:key_lemma_var}.
\begin{lemma}\label{lemma:key_lemma_var}
    Suppose all the assumptions in Theorem \ref{thm:var_iid} hold and follow all the parameters setting in \eqref{eq:thm4_parameters}. In addition, suppose that the number of trajectories $K\ge \max\{\tilde{O}(H^4/\kappa^2),\tilde{O}(H^6d/\kappa)\}$.Then for any $(s,a,h)\in\S\times\A\times[H]$, with probability at least $1-7\delta$, one has 
    \begin{align}\label{eq:important_event_var}
        | (\widehat{\mathbb{B}}_{h}^{\rho,\sigma} \widehat{V}_{h + 1})(s, a) - (\mathbb{B}_h^{\rho} \widehat{V}_{h + 1})(s, a) | \le \Gamma^{\sigma}_h(s,a) 
        &\defeq  \gamma_1\sum_{i=1}^d \|\phi_i(s,a)\1_i\|_{\Sigma_h^{-1}}.
    \end{align}
    In addition, 
    \begin{equation*}
        \gamma_1\sum_{i=1}^d \|\phi_i(s,a)\1_i\|_{\Sigma_h^{-1}} \le 2\gamma_1 \sum_{i=1}^d \|\phi_i(s,a)\1_i\|_{(\Sigma_h^{\star})^{-1}}.
    \end{equation*}
\end{lemma}
Next, following the same steps in Appendix \ref{proof:thm_main_iid}, with probability exceeding $1-7\delta$, one has
\begin{align*}
    \subopt(\widehat{\pi};\zeta,\Prho) 
    &\le 2\gamma_1\sum_{h=1}^H\sum_{i=1}^d \max_{d_h^{\star}\in\D_h^{\star}} \mathbb{E}_{d_h^{\star}}\left[\|\phi_i(s_h,a_h)\1_i\|_{\Sigma_h^{-1}}\right]\\
    &\le 4\gamma_1\sum_{h=1}^H\sum_{i=1}^d \max_{d_h^{\star}\in\D_h^{\star}} \mathbb{E}_{d_h^{\star}}\left[\|\phi_i(s_h,a_h)\1_i\|_{(\Sigma_h^{\star})^{-1}}\right].
\end{align*}

\subsection{Proof of Lemma \ref{lemma:key_lemma_var}}\label{proof:key_lemma_var}
Similar to Lemma \ref{lemma:control_suboptimality_gap}, we first establish the following lemma, which proof is postponed to Appendix \ref{proof:control_suboptimality_gap_varest}.
\begin{lemma}\label{lemma:control_suboptimality_gap_varest}
    Suppose the Assumption \ref{assump:linear-mdp} and \ref{assump:d_rectangular} hold. Then, for any $(s,a,h)\in\S\times\A\times[H]$ and any $V_{h+1}:\S\to[0,H]$, we have
    \begin{equation}\label{eq:control_suboptimality_gap_varest}
    \begin{aligned}
        &|(\widehat{\mathbb{B}}_{h}^{\rho,\sigma} V_{h + 1})(s, a) - (\mathbb{B}_h^{\rho} V_{h + 1})(s, a) |\\
        &\le \left(2\sqrt{\lambda_1 d} H + \max_{\alpha\in [\min_s V_{h+1}(s),\max_s V_{h+1}(s)]}  \|\sum_{\tau\in\D_h^0} \frac{\phi(s_h^{\tau},a_h^\tau)}{\varest_h(s_h^{\tau},a_h^{\tau})} \epsilon_h^{\tau,\sigma}(\alpha,V_{h+1})\|_{\Sigma_h^{-1}}\right) \sum_{i=1}^d\|\phi_i(s,a)\1_i\|_{\Sigma_h^{-1}},
    \end{aligned}
    \end{equation}
    where $\epsilon_h^{\tau,\sigma}(\alpha,{V}) = \frac{\int_{\S} P^0_h(s'|s_h^{\tau},a_h^{\tau}) [V]_{\alpha}(s') \intd s'- [ V]_{\alpha}(s_{h+1}^{\tau})}{\varest_h(s_h^{\tau},a_h^{\tau})}$ for any $V:\S\to[0,H]$, any $\tau\in\D_h^0$ and $\alpha \in [\min_s V(s),\max_s V(s)]$.
\end{lemma}
Letting $\lambda_1 = 1/H^2$ in \eqref{eq:control_suboptimality_gap_varest}, then Lemma \ref{lemma:control_suboptimality_gap_varest} becomes 
\begin{equation}\label{eq:control_suboptimality_gap_varest_1}
    \begin{aligned}
        &|(\widehat{\mathbb{B}}_{h}^{\rho,\sigma} \widehat{V}_{h + 1})(s, a) - (\mathbb{B}_h^{\rho,\sigma} \widehat{V}_{h + 1})(s, a) |\\
        &\le \left(2\sqrt{d} + \underset{M_{1,h}}{\underbrace{\max_{\alpha\in [\min_s \widehat{V}_{h+1}(s),\max_s \widehat{V}_{h+1}(s)]}  \|\sum_{\tau\in\D_h^0} \frac{\phi(s_h^{\tau},a_h^\tau)}{\varest(s_h^{\tau},a_h^{\tau})} \epsilon_h^{\tau,\sigma}(\alpha,\widehat{V}_{h+1})\|_{\Sigma_h^{-1}}}}\right) \sum_{i=1}^d\|\phi_i(s,a)\1_i\|_{\Sigma_h^{-1}}.
    \end{aligned}
    \end{equation}
Due to the correlation between $\alpha$ and $\widehat{V}_{h+1}$, we also apply the uniform concentration with the minimal $\epsilon_1$-covering set $\N(\epsilon_1,H)$ for $\alpha$ defined in \eqref{eq:def_covering_set_for_alpha}. Similar to \eqref{eq:diff_alpha}, there exists $\alpha^{\dag}\in \N(\epsilon_1,H)$ s.t.
    \begin{equation}\label{eq:decompose_M1}
        M^2_{1,h} \le 8\epsilon_1^2 H^2 K^2+ 2 \underset{M_{2,h}}{\underbrace{\left\|\sum_{\tau\in\D_h^0} \frac{\phi(s_h^{\tau},a_h^\tau)}{\varest_h(s_h^{\tau},a_h^{\tau})} \epsilon_h^{\tau,\sigma}(\alpha^{\dag},\widehat{V}_{h+1})\right\|^2_{\Sigma_h^{-1}}}} \le 8 + 2M_{2,h},
    \end{equation}
    where the second inequality holds if $\epsilon_1\le \frac{1}{HK}$. Without the loss of generality, we let $\epsilon_1 = \frac{1}{HK}$ in the following analysis. The detailed proof of \eqref{eq:decompose_M1} is postponed to Appendix \ref{proof:decompose_M1}.

    Next, we will focus on bound the term $M_{2,h}$. Before proceeding, we first define the $\sigma$-algebra
    \begin{equation*}
        \F_{h,\tau} = \sigma(\{(s_h^j,a_h^j)\}_{j=1}^{(\tau+1)\wedge N_h},\{r_h^j,s_{h+1}^j\}_{j=1}^\tau), 
    \end{equation*}
    for any fixed $h\in[H]$ and $\tau\in\D_h^0$. 
     Noted that the samples in $\D^0$ are temporally statistically independent, i.e., $\widehat{V}_{h+1}$ is independent of $\D_h^0$ for any $h\in[H]$. In addition, $\{\varest_h^2\}_{h\in[H]}$ is constructed using an additional dataset $\tilde{\D}^0$, which is also independent of $\D^0$. Thus, for any $h\in[H]$ and $\tau\in \D_h^0$, we have $ \frac{\phi(s_h^{\tau},a_h^\tau)}{\varest_h(s_h^{\tau},a_h^{\tau})}$ is $\F_{h,\tau-1}$-measurable and $| \frac{\phi(s_h^{\tau},a_h^\tau)}{\varest_h(s_h^{\tau},a_h^{\tau})}|\le 1$. Also, $\epsilon_h^{\tau,\sigma}(\alpha^{\dag},\widehat{V}_{h+1})$ is $\F_{h,\tau}$-measurable,
     \begin{equation*}
        \mathbb{E}[\epsilon_h^{\tau,\sigma}(\alpha^{\dag},\widehat{V}_{h+1})| \mathcal{F}_{h,\tau-1}] = 0,\qquad |\epsilon_h^{\tau,\sigma}(\alpha^{\dag},\widehat{V}_{h+1})|\le H,
    \end{equation*}

    It follows the independence between $\varest_h^2$ and $\D_h^0$ that
    \begin{align}
        \Var_{P_{h}^0}\left[\frac{\int_{\S} P^0_h(s'|s_h^{\tau},a_h^{\tau}) [\widehat{V}_{h+1}]_{\alpha}(s') \intd s'- [\widehat{V}_{h+1}]_{\alpha}(s)}{\varest_h(s_h^{\tau},a_h^{\tau})} \right](s_h^{\tau},a_{h}^{\tau}) 
        & = \frac{\Var_{{P_{h}^0}}[\widehat{V}_{h+1}]_{\alpha}(s_h^{\tau},a_{h}^{\tau})}{\varest_h^2(s_h^{\tau},a_h^{\tau})}\nonumber\\
        &\le \frac{\mathbb{V}_{P_h^0}\widehat{V}_{h+1}(s_h^{\tau},a_h^{\tau})}{\varest_h^2(s_h^{\tau},a_h^{\tau})},\label{eq:Var_error}
    \end{align}
    for any $h\in[H]$ and $\tau\in\D_h^0$, where the inequality is from $\mathbb{V}_{P_h^0}\widehat{V}_{h+1}(s_h^{\tau},a_h^{\tau}) =  \max\{1,\Var_{{P_{h}^0}}[\widehat{V}_{h+1}](s_h^{\tau},a_{h}^{\tau})\}$.

   The analysis of the improvement on sample complexity heavily relies on the following lemma about the variance estimation error, where the proof is deferred to Appendix \ref{proof:lemma_varest_error}.
   \begin{lemma} \label{lemma:varest_error}
    Suppose that $\D^0$ and $\tilde{\D}^0$ satisfy all the conditions imposed in Theorem \ref{thm:var_iid}.
    Assume the Assumption \ref{assump:linear-mdp}, \ref{assump:d_rectangular} and \ref{assump:Crob} hold .
    For any $h\in[H]$ and given the nominal transition kernel $P_h^0:\S\times\A\to\S$, the $\widehat{V}_{h+1}$ generated by the \algvar~on $\D^0$ and $\varest_h^2$ generated by \alg~on $\tilde{\D}^0$ satisfies
    \begin{align}
        \left|\mathbb{V}_{P_h^0} V^{\star,\rho}_{h+1}(s_h^{\tau},a_h^{\tau}) - \varest_h^2(s_h^{\tau},a_h^{\tau}) \right| &\le \frac{70c_b H^3 \sqrt{d}}{\sqrt{K\kappa}}, \quad \forall \tau\in\D_h^0, \label{eq:claim_1}\\
        \left|\mathbb{V}_{P_h^0} \widehat{V}_{h+1}(s_h^{\tau},a_h^{\tau}) - \mathbb{V}_{P_h^0} V^{\star,\rho}_{h+1}(s_h^{\tau},a_h^{\tau}) \right| &\le \frac{320c_b H^3 \sqrt{d}}{\sqrt{K\kappa}}, \quad \forall \tau\in\D_h^0, \label{eq:claim_2}
    \end{align}
    where  $c_b =  12\log(3HK/\delta)$ and $K\ge c_1\log(2Hd/\delta)H^4/\kappa^2$ for some sufficiently large universal constant $c_1$, with probability at least $1-6\delta$.
\end{lemma}
    Notice that $1\le \varest_h^2(s,a)\le H^2$ for any $(s,a,h)\in\S\times\A\times[H]$.
   Invoking the Lemma \ref{lemma:varest_error}, we have
   \begin{align*}
    \frac{\mathbb{V}_{P_h^0}\widehat{V}_{h+1}(s_h^{\tau},a_h^{\tau})}{\varest_h^2(s_h^{\tau},a_h^{\tau})}& = \frac{\mathbb{V}_{P_h^0}V^{\star,\rho}_{h+1}(s_h^{\tau},a_h^{\tau})}{\varest_h^2(s_h^{\tau},a_h^{\tau})} +\frac{\left|\mathbb{V}_{P_h^0} \widehat{V}_{h+1}(s_h^{\tau},a_h^{\tau}) - \mathbb{V}_{P_h^0} V^{\star,\rho}_{h+1}(s_h^{\tau},a_h^{\tau}) \right|}{\varest_h^2(s_h^{\tau},a_h^{\tau})}\\
    &\le 1+ \frac{\frac{70c_b H^3 \sqrt{d}}{\sqrt{K\kappa}}}{\varest_h^2(s_h^{\tau},a_h^{\tau})} + \frac{320c_b H^3 \sqrt{d}}{\sqrt{K\kappa}\cdot \varest_h^2(s_h^{\tau},a_h^{\tau})} \le 1+  \frac{400c_b H^3 \sqrt{d}}{\sqrt{K\kappa}}\le 2
   \end{align*}
   where the penultimate inequality uses $1\le \varest_h^2(s,a)$ for any $(s,a,h)\in\S\times\A\times[H]$ and the last inequality holds as long as $K\ge c_kH^6d/\kappa$ for some sufficiently large universal constant $c_k$.
   Therefore, combining with \eqref{eq:Var_error} leads to $\Var[\epsilon_h^{\tau,\sigma}(\alpha^{\dag},V_{h+1})|\F_{h,\tau-1}]\le 2$.

   Suppose that $\sqrt{d}\ge H$. 
   From \cite{vershynin2018high}, one has $|\N(\epsilon_1,H)|\le \frac{3H}{\epsilon_1} = 3H^2K$. By the union bound and invoking Lemma \ref{lemma:Bernstein}, we have
   \begin{align*}
    \sup_{\alpha\in\N(\epsilon_1,H)}\left\|\sum_{\tau\in\D_h^0} \frac{\phi(s_h^{\tau},a_h^\tau)}{\varest(s_h^{\tau},a_h^{\tau})} \epsilon_h^{\tau,\sigma}(\alpha,\hatV_{h+1})\right\|_{\Sigma_h^{-1}}  &\le 16\sqrt{d\log(1+ H^2K/d)\log(12H^3K^3/\delta)} + 4H\log(12H^3K^3/\delta)\\
    &\le c_{1}\sqrt{d}
   \end{align*}
   with probability $1-7\delta$ and for a fixed $\alpha\in\N(\epsilon_1,\hatV_{h+1})$, for $c_1 =40\log(3HK/\delta)$.
   Then, the equation \eqref{eq:control_suboptimality_gap_varest} becomes
   \begin{align*}
    |(\widehat{\mathbb{B}}_{h}^{\rho,\sigma} \widehat{V}_{h + 1})(s, a) - (\mathbb{B}_h^{\rho} \widehat{V}_{h + 1})(s, a) | &\le (2\sqrt{d} + 2\sqrt{2}+\sqrt{2}c_1\sqrt{d}) \sum_{i=1}^d\|\phi_i(s,a)\1_i\|_{\Sigma_h^{-1}}\\
    &= \gamma_1\sum_{i=1}^d \|\phi_i(s,a)\1_i\|_{\Sigma_h^{-1}} \defeq    \Gamma^{\sigma}_h(s,a) .
   \end{align*}
   Moreover, from \eqref{eq:claim_1}, we have
   \begin{align*}
     \sum_{\tau\in\D_h^0} \frac{\phi_h^{\tau}(s_h^{\tau},a_h^{\tau})\phi_h^{\tau}(s_h^{\tau},a_h^{\tau})^{\top}}{\varest_h^2(s_h^{\tau},a_h^{\tau})} 
     &\succeq  \sum_{\tau\in\D_h^0} \frac{\phi_h^{\tau}(s_h^{\tau},a_h^{\tau})\phi_h^{\tau}(s_h^{\tau},a_h^{\tau})^{\top}}{\mathbb{V}_{P_h^0} V^{\star,\rho}_{h+1}(s_h^{\tau},a_h^{\tau})+\frac{70c_b H^3 d}{\sqrt{Kji\kappa}}}\\
     &\succeq  \sum_{\tau\in\D_h^0} \frac{\phi_h^{\tau}(s_h^{\tau},a_h^{\tau})\phi_h^{\tau}(s_h^{\tau},a_h^{\tau})^{\top}}{2\mathbb{V}_{P_h^0} V^{\star,\rho}_{h+1}(s_h^{\tau},a_h^{\tau})}
   \end{align*}
   where the last inequality is from $\frac{70c_b H^3 \sqrt{d}}{\sqrt{K\kappa}} \le \frac{1}{2}$. Then, we obtain $\Sigma_h \succeq \frac{1}{2}\Sigma_h^{\star}$ for any $h\in[H]$, which completes the proof.

\input{appendix/proof_thm2_lemmas.tex}

\input{appendix/proof_varest_error.tex}

%% file: appendix/three_fold.tex
\input{appendix/implementation_algvar.tex}

\subsection{Theoretical guarantee for \threefold}\label{appendix:three_fold}
As the three-fold subsampling method presented in Appendix \ref{sec:implementation_algvar}, it is slightly different from the two-fold variant. Thus, we establish the following lemma to show that \eqref{eq:update_Ntrim_three} is a valid high-probability lower bound of $\Nmain(s)$ for any $s\in \S$ and $h\in [H]$, which follows the proof of Lemma 3 in \cite{li2022settling}. 
    \begin{lemma}\label{lemma: Ntrim_bound_var}
        Consider $\delta\in(0,1)$.
        With probability at least $1-3\delta$, if $\Ntrim_h(s)$ satisfies \eqref{eq:update_Ntrim_three} for every $s\in\S$ and $h\in[H]$, then the following bounds hold, \ie, 
        \begin{equation}\label{eq:Ntrim_upper}
            \Ntrim_h(s)\le \Nmain_h(s),\qquad \Ntrim_h(s)\le \Nvar_h(s), \qquad \forall (s,h)\in \S\times [H].
        \end{equation}
        In addition, with probability at least $1-4\delta$, the following bound also holds:
        \begin{equation}\label{eq:Ntrim_lower}
            \Ntrim_h(s,a) \ge \frac{Kd_h^{b}(s,a)}{12} - \sqrt{6Kd_h^{\textb}(s,a)\log \frac{KH}{\delta}},  \forall (s,a,h)\in\S\times\A\times [H].
        \end{equation}
    \end{lemma}
    \begin{proof}
         We begin with proving the first claim \eqref{eq:Ntrim_upper}.
        Let $\S_{\Daux}\subset \S$ be the collection of all the states appearing for the dataset $\Daux$, where $|\S_{\Daux}|\le K/3$.
        Without loss of generality, we assume that $\Daux$ contains the first $K/3$ trajectories and satisfies
        \begin{equation*}
            \Naux_h(s) = \sum_{k=1}^{K/3} \1 (s_h^k = s), \quad \forall (s,h)\in\S\times[H],
        \end{equation*}
        which can be viewed as the sum of $K/3$ independent Bernoulli random variables. By the union bound and the Bernstein inequality, 
        \begin{align*}
            P \left(\exists(s,h)\in\S_{\Daux}\times [H]:\left| \Naux_h(s) - \frac{K}{3}d_h^{b}(s)\right|\ge t\right)
            &\le \sum_{s\in \S_{\Daux},h\in[H]} P\left(\left| \Naux_h(s) - \frac{K}{3}d_h^{b}(s)\right|\ge t\right)\\
            &\le \frac{2KH}{3}\exp\left(-\frac{t^2/2}{v_{s,h}+t/3}\right),
        \end{align*}
        for any $t\ge 0$, where
        \begin{equation*}
            v_{s,h} = \frac{K}{3}\Var[\1 (s_h^k=s)] \le \frac{Kd_h^{\textb}(s)}{3}.
        \end{equation*} 
        Here, we abuse the notation $\Var$ to represent the variance of the Bernoulli distributed $\1 (s_h^k=s)$.
        Then, with probability at least $1-2\delta/3$, we have
        \begin{align}
            \left| \Naux_h(s) - \frac{K}{3}d_h^{b}(s)\right|&\le \sqrt{2v_{s,h}\log(\frac{KH}{\delta})} + \frac{2}{3}\log(\frac{KH}{\delta}) \nonumber\\
            &\le \sqrt{Kd_h^{\textb}(s)\log(\frac{KH}{\delta})} +\log(\frac{KH}{\delta}),~\forall (s,h)\in\S\times[H].\label{eq:Naux-mean}
        \end{align}
        Similarly, with probability at least $1-2\delta/3$, we have
        \begin{equation}\label{eq:Nmain-mean}
            \left| \Nmain_h(s) - \frac{K}{3}d_h^{b}(s)\right| \le \sqrt{Kd_h^{\textb}(s)\log(\frac{KH}{\delta})} + \log(\frac{KH}{\delta}),\quad\forall (s,h)\in\S\times[H].
        \end{equation}
        Therefore, combining \eqref{eq:Naux-mean} and \eqref{eq:Nmain-mean} leads to 
        \begin{equation}\label{eq:Nmain-Naux}
            \left|\Nmain_h(s) -\Naux_h(s)\right| \le 2\sqrt{Kd_h^{\textb}(s)\log(\frac{KH}{\delta})} + 2\log(\frac{KH}{\delta}),\quad\forall (s,h)\in\S\times[H],
        \end{equation}
        with probability at least $1-4\delta/3$.
        Then, we consider the following two cases
        \begin{itemize}
            \item Case 1: $\Naux_h(s)\le 36\log\frac{KH}{\delta}$. One has 
            \begin{equation*}
                \Ntrim_h(s) = \max\{\Naux_h(s)-6\sqrt{\Naux_h(s)\log\frac{HK}{\delta}},0\} =0 \le \Nmain_h(s).
            \end{equation*}
            \item Case 2: $\Naux_h(s)>36\log\frac{KH}{\delta}$.  From \eqref{eq:Naux-mean}, we have 
            \begin{equation*}
                \frac{K}{3}d_h^{b}(s) + \sqrt{Kd_h^{\textb}(s)\log(\frac{KH}{\delta})} + \log(\frac{KH}{\delta})\ge\Naux_h(s) \ge 36\log\frac{KH}{\delta},
            \end{equation*}
            implying  
            \begin{equation*}
                Kd_h^{\textb}(s) \ge 72\log\frac{KH}{\delta}.
            \end{equation*}
            Also from \eqref{eq:Naux-mean},
            \begin{equation*}
                \Naux_h(s)\ge \frac{K}{3}d_h^{b}(s) - \sqrt{Kd_h^{\textb}(s)\log(\frac{KH}{\delta})} - \log(\frac{KH}{\delta})\ge \frac{K}{6}d_h^{b}(s).
            \end{equation*}
            Therefore, with probability exceeding $1-4\delta/3$
            \begin{align*}
                \Ntrim_h(s) &= \Naux_h(s)-6\sqrt{\Naux_h(s)\log\frac{HK}{\delta}}\le 
                \Naux_h(s)-\sqrt{6}\sqrt{Kd_h^{\textb}(s)\log\frac{HK}{\delta}}\\
                &\le  \Naux_h(s)-2\sqrt{Kd_h^{\textb}(s)\log\frac{HK}{\delta}} - \frac{1}{3}\sqrt{Kd_h^{\textb}(s)\log\frac{HK}{\delta}}\\
                &\le \Naux_h(s) - 2\sqrt{Kd_h^{\textb}(s)\log\frac{HK}{\delta}} - 2\log\frac{HK}{\delta}\\
                &\le \Nmain_h(s),
            \end{align*}
            where the last inequality is from \eqref{eq:Nmain-Naux}.
        \end{itemize}
        Following the same arguments,  we also have 
        \begin{equation}\label{eq:Nvar-mean}
            \left|\Nvar_h(s) -\Naux_h(s)\right| \le 2\sqrt{Kd_h^{\textb}(s)\log(\frac{KH}{\delta})} + 2\log(\frac{KH}{\delta}),\quad\forall (s,h)\in\S\times[H].
        \end{equation}
        holds, with probability at least $1-4\delta/3$. Therefore, we can also guarantee that $\Ntrim_h(s)\le \Nvar_h(s)$ with probability at least $1-4\delta/3$, for any $(s,h)\in\S\times[H]$.

        Putting these two results together, we prove the first claim \eqref{eq:Ntrim_upper}.

       Next, we will establish the second claim \eqref{eq:Ntrim_lower}. To begin with, we claim the following statement holds with probability exceeding $1-2\delta/3$,
       \begin{equation}\label{eq:Ntrim_lowerbound_claim}
        \Ntrim_h(s,a)\ge \Ntrim_h(s)\pi_h^{\textb}(a|s) - \sqrt{2\Ntrim_h(s)\pi_h^{\textb}(a|s)\log(\frac{KH}{\delta})} - \log\frac{KH}{\delta}, \quad\forall (s,a,h)\in\S\times\A\times[H],
       \end{equation}
       conditioned on the high-probability event that the first part \eqref{eq:Ntrim_upper} holds.
       In the sequel, we discuss the following two cases, provided that the inequality \eqref{eq:Ntrim_lowerbound_claim} holds.
       \begin{itemize}
        \item Case 1:$Kd_h^{\textb}(s,a) = Kd_h^{\textb}(s)\pi_h^{\textb}(a|s)> 864\log\frac{KH}{\delta}$. From \eqref{eq:Naux-mean}, with probability exceeding $1-2\delta/3$, one has
        \begin{equation*}
            \Naux_h(s)\ge \frac{K}{3}d_h^{b}(s) - \sqrt{Kd_h^{\textb}(s)\log(\frac{KH}{\delta})} - \log(\frac{KH}{\delta})\ge \frac{K}{6}d_h^{b}(s) \ge 144\log\frac{KH}{\delta}.
        \end{equation*}
        Together with the definition \eqref{eq:update_Ntrim_three}, we have 
        \begin{align*}
            \Ntrim_h(s) \ge \Naux_h(s) - 6\sqrt{\Naux_h(s) \log\frac{KH}{\delta}} \ge  \frac{1}{2}\Naux_h(s) \ge \frac{K}{12}d_h^{b}(s).
        \end{align*}
        Therefore,
        \begin{equation*}
            \Ntrim_h(s)\pi_h^{\textb}(a|s) \ge \frac{K}{12}d_h^{b}(s) \pi_h^{\textb}(a|s) \ge 72\log\frac{KH}{\delta}.
        \end{equation*}
        Combining with \eqref{eq:Ntrim_lowerbound_claim}, one can derive
        \begin{align*}
            \Ntrim_h(s,a) &\ge\frac{Kd_h^{b}(s,a)}{12}  - \sqrt{\frac{1}{6}Kd_h^{b}(s,a)\log(\frac{KH}{\delta})} - \log\frac{KH}{\delta}\\
            &\ge \frac{Kd_h^{b}(s,a)}{12}  - \sqrt{6Kd_h^{\textb}(s,a)\log \frac{KH}{\delta}}.
        \end{align*}
        with probability exceeding $1-4\delta/3$
        \item Case 2: $Kd_h^{\textb}(s,a) \le 864\log\frac{KH}{\delta}$. From \eqref{eq:update_Ntrim_three}, one has 
            \begin{equation*}
                \Ntrim_h(s,a) \ge 0 \ge \frac{Kd_h^{b}(s,a)}{12} - \sqrt{6Kd_h^{\textb}(s,a)\log \frac{KH}{\delta}}.
            \end{equation*}
       \end{itemize}
       By integrating these two cases, we can claim \eqref{eq:Ntrim_lower} is valid with probability exceeding $1-4\delta/3$, as long as the inequality \eqref{eq:Ntrim_lowerbound_claim} holds under the condition of the high-probability event described in the the first part \eqref{eq:Ntrim_upper}. Thus, the second claim \eqref{eq:Ntrim_lower} holds with probability at least $1-4\delta$.
       \paragraph{Proof of inequality \eqref{eq:Ntrim_lowerbound_claim}.} First, we can observe that the inequality \eqref{eq:Ntrim_lowerbound_claim} holds if $\Ntrim_h(s)\pi_h^{\textb}(s,a) \le 2\log\frac{KH}{\delta}$. Thus, we focus on the other case that  $\Ntrim_h(s)\pi_h^{\textb}(s,a) > 2\log\frac{KH}{\delta}$. Denote that 
       \begin{equation*}
            \mathcal{E} = \{(s,a,h)\in\S\times\A\times[H]\vert~\Ntrim_h(s)\pi_h^{\textb}(a|s)>2\log(\frac{KH}{\delta})\}.
       \end{equation*}
        Noticed that from Algorithm \ref{alg:three_fold_subsampling}, one has that $|\mathcal{E}| \le \frac{KH}{3}$. Supposing that the first claim \eqref{eq:Ntrim_upper} holds, one has $\Ntrim_h(s) = \min\{\Ntrim_h(s),\Nmain_h(s),\Nvar_h(s)\}$. Therefore, $\Ntrim_h(s,a)$ can be viewed as the sum of $\Ntrim_h(s)$ independent Bernoulli random variables, where each is with the mean $\pi_h^{\textb}(a|s)$. Then, by the union bound and the Bernstein inequality,
        \begin{align*}
            &P\left(\exists (s,a,h)\in\mathcal{E}: \left\vert \Ntrim_h(s,a) - \Ntrim_h(s)\pi_h^{\textb}(a|s)  \right\vert\ge t\right)\\
            &\le \sum_{(s,a,h)\in\mathcal{E}} P\left(\left\vert \Ntrim_h(s,a) - \Ntrim_h(s)\pi_h^{\textb}(a|s)  \right\vert\ge t\right)\le \frac{2KH}{3}\exp\left(-\frac{t^2/2}{v_{s,h}+t/3}\right),
        \end{align*}
        for any $t\ge 0$, where
        \begin{equation*}
            v_{s,h} = \Ntrim_h(s)\Var[\1((s_h^k,a_h^k)= (s,a))] \le \Ntrim_h(s) \pi_h^{\textb}(a|s)
        \end{equation*} 
        A little algebra yields that with probability at least $1-2\delta/3$, one can obtain
        \begin{align}
            \left| \Ntrim_h(s,a) - \Ntrim_h(s)\pi_h^{\textb}(a|s)\right|&\le \sqrt{2v_{s,h}\log(\frac{KH}{\delta})} + \frac{2}{3}\log(\frac{KH}{\delta}) \nonumber\\
            &\le \sqrt{2\Ntrim_h(s) \pi_h^{\textb}(a|s)\log(\frac{KH}{\delta})} +\log(\frac{KH}{\delta}),~\forall (s,h)\in\S\times[H].
        \end{align}
        Therefore, with probability $1-2\delta/3$, one can obtain
        \begin{equation*}
            \Ntrim_h(s,a) \ge \Ntrim_h(s)\pi_h^{\textb}(a|s)- \sqrt{2\Ntrim_h(s) \pi_h^{\textb}(a|s)\log(\frac{KH}{\delta})} - \log(\frac{KH}{\delta}),
        \end{equation*}
        for any $(s,a,h)\in \mathcal{E}$, conditioned on the first claim \eqref{eq:Ntrim_upper} holds.
    \end{proof}

    In addition, the following lemma guarantees that the samples in $\Dmainsub$ and $\Dvarsub$ are statistically independent with probability exceeding $1-3\delta$. Before continuing, we denote $\Diid$ as the dataset containing $\Ntrim_h(s)$ independent transition-reward sample tuples for every $(s,h)\in\S\times[H]$,  following $\pi_h^{\textb}$ and $P_h^0$.
    \begin{lemma}[Modified Lemma 7, \cite{li2022settling}]\label{lemma:threefold_independent}
        With probability exceeding $1-3\delta$, $\Dmainsub$ and $\Dvarsub$ generated by Algorithm \ref{alg:three_fold_subsampling} as well as $\Diid$ have the same distributions.
    \end{lemma}

%% file: appendix/implementation_algvar.tex
\subsection{The implementation of \algvar}\label{sec:implementation_algvar}

\begin{algorithm}[!h]
    \caption{Distributionally Robust Pessimistic Least-Squares Value Iteration with Variance Estimation (\algvar)}\label{alg:DRLSVI_V}
    \begin{algorithmic}[1]
    \Input Datasets $\tilde{\D}^0,\D^0\gets\threefold(\D)$; feature map $\phi(s,a)$ for $(s,a)\in\S\times\A$; $\gamma_1,\lambda_1>0$.
    \ConstructVar Obtain $(\tilde{V}, \tilde{\pi}) \gets \alg(\tilde{\D}^0, \phi)$
    \State For every $h\in[H]$, compute $\tilde{\Lambda}_h =  \sum_{\tau\in \D_h^0}  \phi(s_h^{\tau}, a_h^{\tau}) \phi(s_h^{\tau}, a_h^{\tau})^{\top} + I_d$ and
    \begin{equation*}
        \nu_{h,1} = (\tilde{\Lambda}_h)^{-1}\bigg(\sum_{\tau\in\Dtilde_h^0} \phi(s_h^{\tau},a_h^{\tau}) \tilde{V}_{h+1}^2(s_{h+1}^{\tau})\bigg),\quad\nu_{h,2}  = (\tilde{\Lambda}_h)^{-1}\bigg(\sum_{\tau\in\Dtilde_h^0} \phi(s_h^{\tau},a_h^{\tau}) \tilde{V}_{h+1}(s_{h+1}^{\tau})\bigg).
    \end{equation*}
    \State Update $\varest_h^2 (s,a)$ via \eqref{eq:var_est}, for any $ (s,a)\in\S\times\A$.
    \Init Set $\widehat{Q}_{H + 1}(\cdot,\cdot) = 0$ and $\widehat{V}_{H + 1}(\cdot) = 0$.
    \For{step $h = H, H - 1, \cdots, 1$}
        \State $\Sigma_h = \sum_{\tau\in \D^0_h} \frac{\phi(s_h^{\tau},a_h^{\tau})\phi(s_h^{\tau},a_h^{\tau})^{\top}}{\varest^2_h(s_h^{\tau},a_h^{\tau})}+\lambda_1 I_d$.
        \State $\widehat{\theta}_h^{\sigma} =\Sigma_h^{-1} \big( \sum_{\tau \in \D_h^0} \frac{\phi(s_h^{\tau}, a_h^{\tau}) r_h^{\tau}}{{\varest}^2_h(s_h^{\tau},a_h^{\tau})} \big)$.
        \For{feature $i = 1, \cdots, d$}
        \State Update $\widehat{\nu}_{h,i}^{\rho,\sigma,\widehat{V}}$ via \eqref{eq:update_nu_var}.
        \EndFor
        \State $\widehat{w}_h^{\rho,\sigma,\widehat{V}} = \widehat{\theta}_h + \widehat{\nu}_h^{\rho,\sigma,\widehat{V}}$.
        \State $\bar{Q}_h(\cdot, \cdot) = \phi(\cdot, \cdot)^{\top}\widehat{w}_h^{\rho,\sigma,\widehat{V}} - \gamma_1\sum_{i=1}^d \|\phi_i(\cdot,\cdot)\1_i\|_{\Sigma_h^{-1}}$.
        \State $\widehat{Q}_h(\cdot, \cdot) = \min\left\{\bar{Q}_h, H - h + 1 \right\}_{+} $.
        \State $ \widehat{\pi}_h(\cdot) = \argmax_{a\in\A} \widehat{Q}_h(\cdot,a)$.
        \State $\widehat{V}_h(\cdot) =  \widehat{Q}_h(\cdot,\widehat{\pi}_h(\cdot))$.
    \EndFor
    \Output $\widehat{\pi} \defeq \{\widehat{\pi}_h\}_{h = 1}^H$
    \end{algorithmic}
    \end{algorithm}

    \begin{algorithm}[!htbp]
        \caption{\threefold}\label{alg:three_fold_subsampling}
        \begin{algorithmic}[1]
        \Input Batch dataset $\mathcal{D}$; 
        \State \textbf{Split Data:} Split $\D$ into three haves $\Daux$, $\Dmain$ and $\Dvar$, where $|\Daux|= |\Dmain| =|\Dvar| = K/3$. Denote $\Nmain_h(s)$ (resp. $\Naux_h(s)$ or $\Nvar_h(s)$) as the number of sample transitions from state $s$ at step $h$ in  $\Dmain$ (resp. $\Daux$ or $\Dvar$).
        \State \textbf{Construct the high-probability lower bound $\Ntrim_h(s)$ by $\Daux$:} For each $s\in\S$ and $1\le h\le H$, compute 
        \begin{equation}\label{eq:update_Ntrim_three}
            \Ntrim_h(s) = \max\{\Naux_h(s) - 6\sqrt{\Naux_h(s) \log\frac{KH}{\delta}},0\}.
        \end{equation}
        \State \textbf{Construct the almost temporally statistically independent $\Dmainsub$ and $\Dvarsub$:}  Let $\Dmain_h(s)$ (resp. $\Dvar_h(s)$) be the set of all transition-reward sample tuples at state $s$ and step $h$ from $\Dmain$ (resp. $\Dvar$). For any $(s,h)\in \S\times [H]$, subsample $\min\{\Ntrim_h(s),\Nmain_h(s)\}$ (resp. $\min\{\Ntrim_h(s),\Nvar_h(s)\}$) sample tuples randomly from $\Dmain_h(s)$ (resp. $\Dvar_h(s)$), denoted as $\Dmainsub$ (resp. $\Dmainsub$).
        \Output $\Dmainsub,~\Dvarsub$.
        \end{algorithmic}
        \end{algorithm}

The implementation of \algvar is detailed in Algorithm \ref{alg:DRLSVI_V}, which can be divided into three steps. First, we carefully design \threefold (cf. Algorithm \ref{alg:three_fold_subsampling}), to generate two almost temporally statistically independent datasets, $\tilde{\D}^0,\D^0$, which are also independent from each other. The theoretical analysis of \threefold is postponed to Appendix \ref{appendix:three_fold}.
The second step is to construct a variance estimator $\varest_h^2$ for any $h\in[H]$ via $\tilde{\D}^0$, which is independent of $\D_h^0$. The key idea is to utilize the intermediate results $\{\tilde{V}_{h}\}_{h=1}^{H+1}$ of running \alg~on $\tilde{\D}^0$ to approximate the variance as \eqref{eq:var_est}. With the variance estimator at our hands, the last step is to apply the weighted ridge regression to construct the empirical variance-aware robust Bellman operator via \eqref{eq:update_theta_var}-\eqref{eq:update_nu_var}, which is slightly different from \alg.

%% file: appendix/proof_thm2_lemmas.tex
\subsubsection{Proof of Lemma \ref{lemma:control_suboptimality_gap_varest}}\label{proof:control_suboptimality_gap_varest}
Following \eqref{eq:direcly_decompose_B_diff}-\eqref{eq:T_2_after_max}, we have 
\begin{align*}
    &|(\widehat{\mathbb{B}}_{h}^{\rho,\sigma} V_{h + 1})(s, a) - (\mathbb{B}_h^{\rho} V_{h + 1})(s, a) |\\
    &\le \sqrt{d\lambda_1} \sum_{i=1}^d \|\phi_i(s,a)\1_i\|_{\Sigma_h^{-1}} +  \underset{\texttt{(i)}}{\underbrace{\sum_{i=1}^d \max_{\alpha \in [\min_s V_{h+1}(s),\max_s V_{h+1}(s)]} \left\vert\phi_i(s,a)\int_{\S}(\widehat{\mu}^{\sigma}_{h,i}(s') - \mu_{h,i}(s'))[V_{h+1}]_{\alpha}(s')\intd s'\right\vert}},
\end{align*}
for $\forall (s,a,h)\times \S\times\A\times[H]$, 
where $\widehat{\mu}^{\sigma}_{h,i}(s)$ is the $i$-th coordinate of
    \begin{equation*}
        \widehat{\mu}^{\sigma}_{h}(s) = \Sigma_{h}^{-1} \left( \sum_{\tau\in \D^0_h} \frac{\phi(s_h^{\tau}, a_h^{\tau}) \1(s = s_{h+1}^{\tau})}{\varest_h^2(s_h^{\tau},a_h^{\tau})} \right)\in\mathbb{R}^{d}
    \end{equation*}
such that $\bar{\nu}_h^{\widehat{V}}(\alpha)= \int_\S \widehat{\mu}_{h}^{\sigma}(s') [\widehat{V}_{h+1}(s')]_{\alpha}\intd s'$ defined in the update \eqref{eq:update_bar_nu_var}. 
Similar to \eqref{eq:muV_esterror1}, by letting  $\epsilon_h^{\tau,\sigma}(\alpha,{V}) = \frac{\int_{\S} P_h^0(s'|s_h^{\tau},a_h^{\tau}) [V]_{\alpha}(s') \intd s'- [ V]_{\alpha}(s_{h+1}^{\tau})}{\varest(s_h^{\tau},a_h^{\tau})}$ for any $V:\S\to[0,H]$, any $\tau\in\D_h^0$ and $\alpha \in [\min_s V(s),\max_s V(s)]$, we have 
\begin{align*}
    & \left\vert \int_{\S}\mu^0_{h,i}(s') [V_{h+1}]_{\alpha}(s')\intd s'- \int_{\S}\widehat{\mu}^{\sigma}_{h,i}(s') [V_{h+1}]_\alpha(s')\intd s'\right\vert\\
    &= \left\vert\1_i^{\top}\Sigma_h^{-1}\left[\lambda_1\int_{\S} {\mu}^0_{h}(s') [V_{h+1}]_{\alpha}(s')\intd s' + \sum_{\tau\in\D_h^0}\frac{\phi(s_h^{\tau},a_h^{\tau})}{\varest_h^2(s_h^{\tau},a_h^{\tau})}\left(\int_{\S} P_h^0(s'|s_h^{\tau},a_h^{\tau})[V_{h+1}]_{\alpha}(s')\intd s'  - [V_{h+1}]_{\alpha}(s_{h+1}^\tau)\right) \right]\right\vert\\
    &= \left\vert\1_i^{\top}\Sigma_h^{-1}\left[\lambda_1\int_{\S} {\mu}^0_{h}(s') [V_{h+1}]_{\alpha}(s')\intd s' + \sum_{\tau\in\D_h^0}\frac{\phi(s_h^{\tau},a_h^{\tau})}{\varest_h(s_h^{\tau},a_h^{\tau})}\epsilon_h^{\tau,\sigma}(\alpha,V_{h+1})\right]\right\vert
\end{align*}
Then, we obtain
\begin{align}
    &\left\vert\phi_i(s,a)\int_{\S}(\widehat{\mu}^{\sigma}_{h,i}(s') - \mu_{h,i}(s'))[V_{h+1}]_{\alpha}(s')\intd s'\right\vert\nonumber\\
    &\le \left\vert \phi_i(s,a)\1_i^{\top}\Sigma_h^{-1}\left(\lambda_1\int_{\S} {\mu}^0_{h}(s') [V_{h+1}]_{\alpha}(s')\intd s' +\sum_{\tau\in\D_h^0}\frac{\phi(s_h^{\tau},a_h^\tau)}{\varest(s_h^{\tau},a_h^{\tau})} \epsilon_h^{\tau,\sigma}(\alpha,V_{h+1}) \right) \right\vert\nonumber\\
    &\le \|\phi_i(s,a)\1_i\|_{\Sigma_h^{-1}} \left( \underset{\texttt{(ii)}}{\underbrace{\lambda_1\|\int_{\S} {\mu}^0_{h}(s') [V_{h+1}]_{\alpha}(s')\intd s'\|_{\Sigma_h^{-1}}}} + \|\sum_{\tau\in\D_h^0}\frac{\phi(s_h^{\tau},a_h^\tau)}{\varest(s_h^{\tau},a_h^{\tau})}\epsilon_h^{\tau,\sigma}(\alpha,V_{h+1})\|_{\Sigma_h^{-1}}\right),\label{eq:max_term_varest}
\end{align}
where the last inequality follows Cauchy-Schwarz inequality.
Moreover, the term \texttt{(ii)} in \eqref{eq:max_term_varest} can be further simplified
\begin{equation*}
    \texttt{(ii)} \le \lambda_1\|\Sigma_h^{-1}\|^{\frac{1}{2}} \| \int_{\S} {\mu}^0_{h}(s') [V_{h+1}]_{\alpha}(s')\intd s'\|\le \sqrt{\lambda_1}H,
\end{equation*}
since $V(s)\le H$ for any $s\in\S$ and $\|\Sigma_h^{-1}\| \le 1/\lambda_1$.
Then we have 
\begin{equation}\label{eq:lemma16}
\begin{aligned}
    \texttt{(i)} \le \left(\sqrt{\lambda_1}H +\max_{\alpha\in [\min_s V_{h+1}(s),\max_s V_{h+1}(s)]}  \|\sum_{\tau\in\D_h^0}\frac{\phi(s_h^{\tau},a_h^\tau)}{\varest(s_h^{\tau},a_h^{\tau})}\epsilon_h^{\tau,\sigma}(\alpha,V_{h+1})\|_{\Sigma_h^{-1}}\right) \sum_{i=1}^d &\|\phi_i(s,a)\1_i\|_{\Sigma_h^{-1}},
\end{aligned}
\end{equation}
which concludes our proof of \eqref{eq:control_suboptimality_gap_varest}.

\subsubsection{Proof of \eqref{eq:decompose_M1}} \label{proof:decompose_M1}
Due to the semi-positiveness of $\Sigma_h^{-1}$, one can control $M^2_{1,h}$ for any $h\in[H]$ as
\begin{align*}
    & \max_{\alpha\in [\min_s \widehat{V}_{h+1}(s),\max_s \widehat{V}_{h+1}(s)]}  \|\sum_{\tau\in \D_h^0}\frac{\phi(s_h^{\tau},a_h^\tau)}{\varest_h(s_h^{\tau},a_h^{\tau})} \epsilon_h^{\tau,\sigma}(\alpha,\widehat{V}_{h+1})\|^2_{\Sigma_h^{-1}} \\
    & \le \max_{\alpha\in [0,H]} 2\|\sum_{\tau\in \D^0_h}
    \frac{\phi(s_h^{\tau},a_h^\tau)}{\varest_h(s_h^{\tau},a_h^{\tau})}  \left(\epsilon_{h}^{\tau,\sigma}(\alpha,\widehat{V}_{h+1}) - \epsilon_{h}^{\tau,\sigma}(\alpha^{\dag},\widehat{V}_{h+1})\right)\|_{\Sigma_h^{-1}}^2 +2\|\sum_{\tau\in \D^0_h}\frac{\phi(s_h^{\tau},a_h^\tau)}{\varest_h(s_h^{\tau},a_h^{\tau})} \epsilon_{h}^{\tau,\sigma}(\alpha^{\dag},\widehat{V}_{h+1})\|_{\Sigma_h^{-1}}^2,
\end{align*}
for some $\alpha^{\dag} \in\N(\epsilon_1,H)$.
Note that $\epsilon_{h}^{\tau,\sigma}(\alpha,V)$ is 2-Lipschitz w.r.t. $\alpha$ for any $V:\S\to [0,H]$, i.e.,
\begin{align*}
    |\epsilon_{h}^{\tau,\sigma}(\alpha,V) - \epsilon_h^{\tau,\sigma}(\alpha^{\dag},V)|
    \le & 2|\alpha-\alpha^{\dag}|\le 2\epsilon_1.
\end{align*}
Therefore, for any $\alpha\in [0,H]$, we have
\begin{align*}
    &\|\sum_{\tau\in \D^0_h} \frac{\phi(s_h^{\tau},a_h^\tau)}{\varest_h(s_h^{\tau},a_h^{\tau})}  \left(\epsilon_{h}^{\tau,\sigma}(\alpha,V) - \epsilon_h^{\tau,\sigma}(\alpha^{\dag},V)\right)\|^2_{\Sigma_h^{-1}}\\
    =& \sum_{\tau,\tau'\in \D^0_h} \frac{\phi(s_h^{\tau},a_h^\tau)}{\varest_h(s_h^{\tau},a_h^{\tau})}^{\top}\Sigma_h^{-1}\frac{\phi(s_h^{\tau'},a_h^{\tau'})}{\varest_h(s_h^{\tau'},a_h^{\tau'})} \left[\left(\epsilon_{h}^{\tau,\sigma}(\alpha,V) - \epsilon_h^{\tau,\sigma}(\alpha^{\dag},V)\right)\left(\epsilon_{h}^{\tau',\sigma}(\alpha,V) - \epsilon_h^{\tau',\sigma}(\alpha^{\dag},V)\right)\right]\\
    \le&\sum_{\tau,\tau'\in \D^0_h} \frac{\phi(s_h^{\tau},a_h^\tau)}{\varest_h(s_h^{\tau},a_h^{\tau})} ^{\top}\Sigma_h^{-1}\frac{\phi(s_h^{\tau'},a_h^{\tau'})}{\varest_h(s_h^{\tau'},a_h^{\tau'})} \cdot 4\epsilon_1^2 \\
    \le & 4\epsilon_1^2 N_h^2/\lambda_1,
\end{align*}
where the last inequality is based on $\|\phi(s,a)\|_2\le 1$, $\varest_h(s,a)\ge 1$ for any $(s,a,h)\in\S\times\A\times [H]$ and $\lambda_{\min}(\Sigma_h)\ge\lambda_1 = \frac{1}{H^2}$ for any $h\in[H]$ such that 
\begin{equation}\label{eq:quad_phi_Sigma}
\begin{aligned}
    \sum_{\tau,\tau'\in \D^0_h} \phi(s_h^{\tau},a_h^\tau)^{\top}\Sigma_h^{-1}\phi(s_h^{\tau'},a_h^{\tau'}) &= \sum_{\tau,\tau'\in \D^0_h} \|\phi(s_h^{\tau},a_h^\tau)\|_2\cdot\|\phi(s_h^{\tau'},a_h^{\tau'})\|_2\cdot\|\Sigma_h^{-1}\|
    &\le  N_h^2/\lambda_1.
\end{aligned}
\end{equation}
Due to the fact $N_h\le K$ for nay $h\in[H]$, we conclude that
\begin{equation}
    M^2_{1,h} \le 8\epsilon_1^2 H^2 K^2+ 2 \underset{M_{2,h}}{\underbrace{\left\|\sum_{\tau\in\D_h^0} \frac{\phi(s_h^{\tau},a_h^\tau)}{\varest_h(s_h^{\tau},a_h^{\tau})} \epsilon_h^{\tau,\sigma}(\alpha^{\dag},\widehat{V}_{h+1})\right\|^2_{\Sigma_h^{-1}}}} \le 8 + 2M_{2,h},
\end{equation}
where the second inequality holds if $\epsilon_1\le \frac{1}{HK}$.

%% file: appendix/proof_varest_error.tex
 \subsubsection{Proof of Lemma \ref{lemma:varest_error}}\label{proof:lemma_varest_error}
 Recall that in Section \ref{sec:var_est}, $\{\varest^2_h\}_{h=1}^H$ is constructed via $\{\tilde{V}_{h+1}\}_{h=1}^H$ generated by \alg~on $\tilde{\D}^0$. Before starting, we define 
 $$\widehat{\Var}_h \tilde{V}_{h+1}(s,a) = [\phi(s,a)^{\top}{\nu}_{h,1}]_{[0,H^2]} -\left([\phi(s,a)^{\top}\nu_{h,2}]_{[0,H]}\right)^2,\quad \forall (s,a,h)\in\S\times\A\times[H],$$
  such that $\varest_h^2(s,a) = \max\{1,\widehat{\Var}_h\tilde{V}_{h+1}(s,a)\}$. In addition, recall that $\mathbb{V}_{P_h^0}V(s,a) =  \max\{1,\Var_{{P_{h}^0}}V(s,a)\}$ for  any value function $V:\S\to[0,H]$ and any $(s,a,h)\in\S\times\A\times[H]$. Then, we can decompose the target terms by
  \begin{align}
    &\left|\mathbb{V}_{P_h^0} V^{\star,\rho}_{h+1}(s_h^{\tau},a_h^{\tau}) - \varest_h^2(s_h^{\tau},a_h^{\tau}) \right|\nonumber\\
    &\le  \left|\mathbb{V}_{P_h^0}\tilde{V}_{h+1}(s_h^{\tau},a_h^{\tau}) - \varest_h^2(s_h^{\tau},a_h^{\tau})\right| +  \left|\mathbb{V}_{P_h^0} V^{\star,\rho}_{h+1}(s_h^{\tau},a_h^{\tau}) - \mathbb{V}_{P_h^0}\tilde{V}_{h+1}(s_h^{\tau},a_h^{\tau})\right| \nonumber\\
    &\le  \underset{(a)}{\underbrace{\left|\Var_{P_h^0}\tilde{V}_{h+1}(s_h^{\tau},a_h^{\tau}) - \widehat{\Var}_h \tilde{V}_{h+1}(s_h^{\tau},a_h^{\tau})\right|}} +  \underset{(b)}{\underbrace{\left|\Var_{P_h^0}V^{\star,\rho}_{h+1}(s_h^{\tau},a_h^{\tau}) - \Var_{P_h^0}\tilde{V}_{h+1}(s_h^{\tau},a_h^{\tau})\right|}} \nonumber,
\end{align}
for every $h\in[H]$, where the last inequality is based on the non-expansiveness of $\max\{1,\cdot\}$. 
Similarly,
\begin{align*}
    \left|\mathbb{V}_{P_h^0} \widehat{V}_{h+1}(s_h^{\tau},a_h^{\tau}) - \mathbb{V}_{P_h^0} V^{\star,\rho}_{h+1}(s_h^{\tau},a_h^{\tau}) \right| \le \underset{(c)}{\underbrace{\left|\Var_{P_h^0}V^{\star,\rho}_{h+1}(s_h^{\tau},a_h^{\tau}) - \Var_{P_h^0}\widehat{V}_{h+1}(s_h^{\tau},a_h^{\tau}) \right|}}.
\end{align*}

In the sequel, we will control $(a)$, $(b)$ and $(c)$ respectively.
\paragraph{Step 1: controlling $(a)$.} For each $\tau\in\D_h^0$, we decompose the term $(a)$ by 
        \begin{align*}
        (a) &= \left|\widehat{\Var}_h \tilde{V}_{h+1}(s_h^{\tau},a_h^{\tau}) - \Var_{P_h^0} \tilde{V}_{h+1}(s_h^{\tau},a_h^{\tau})\right|\\
        &\le \left|[\phi(s_h^{\tau},a_h^{\tau})^{\top}\nu_{h,1}]_{[0,H^2]} - \int_{\S} P^0_{h,s_h^{\tau},a_h^{\tau}}(s')\tilde{V}_{h+1}^2(s')\intd s'\right|\\
        &\quad + \left|\left([\phi(s,a)^{\top}\nu_{h,2}]_{[0,H]}\right)^2 - \left[\int_{\S} P^0_{h,s_h^{\tau},a_h^{\tau}}(s')\tilde{V}_{h+1}(s')\intd s'\right]^{2}\right|\\
        &\le  \underset{(a_1)}{\underbrace{\left|\phi(s_h^{\tau},a_h^{\tau})^{\top}\left(\nu_{h,1} - \int_{\S}\mu_h^0(s')\tilde{V}_{h+1}^2(s')\intd s'\right)\right|}} + 2H \underset{(a_2)}{\underbrace{\left|\phi(s_h^{\tau},a_h^{\tau})^{\top} \left(\nu_{h,2} - \int_{\S}\mu_{h}^0(s')\tilde{V}_{h+1}(s')\intd s'\right)\right|}},
        \end{align*}
        where the last inequality is based on $a^2-b^2 = (a+b)(a-b)$ for any $a,b\in\mathbb{R}$. In the sequel, we control $(a_1)$ and $(a_2)$, respectively. Before continuing, we first define
        $\tilde{\mu}_{h,i}:\S\to\mathbb{R}$ is the $i$-th coordinate of 
        \begin{equation*}
            \tilde{\mu}_{h}(s) = (\tilde{\Lambda}_h)^{-1} \left( \sum_{\tau'\in \tilde{\D}_h^0} \phi(s_h^{\tau'}, a_h^{\tau'})\1(s = s_{h+1}^{\tau'}) \right)\in\mathbb{R}^{d}
        \end{equation*}
        such that $\nu_{h,1}= \int_{\S}\tilde{\mu}_{h}(s')\tilde{V}^2_{h+1}(s')\intd s'\in \mathbb{R}^d$ and $\nu_{h,2} = \int_{\S}\tilde{\mu}_{h}(s')\tilde{V}_{h+1}(s')\intd s'$. With this new notation, we reformulate $(a_1)$ as
        \begin{equation*}
            (a_1) = |\phi(s_h^{\tau},a_h^{\tau})\int_{\S}(\tilde{\mu}_{h}(s')-\mu_{h}^0(s')) \tilde{V}^2_{h+1}(s')\intd s'|.
        \end{equation*}
        Following the steps in Lemma \ref{lemma:control_suboptimality_gap}, i.e., the equations \eqref{eq:muV_esterror1}-\eqref{eq:decomposeT2}  with $\lambda_0 = 1$, we can obtain 
        \begin{align}
            (a_1) &\le \left(H + \max_{\alpha\in [\min_s \tilde{V}_{h+1}(s),\max_s \tilde{V}_{h+1}(s)]} \|\sum_{\tau'\in\tilde{\D}_h^0} \phi(s_h^{\tau'},a_h^{\tau'})\epsilon_h^{\tau'}(\alpha,\tilde{V}^2_{h+1})\|_{(\tilde{\Lambda}_h)^{-1}}\right)\sum_{i=1}^d \|\phi_i(s_h^{\tau},a_h^{\tau})\1_i\|_{\tilde{\Lambda}_h^{-1}}\nonumber\\
            &\le \left(H + 2\sqrt{2}H + \sqrt{2} \sup_{\alpha\in \N(\epsilon_0,H)}\|\sum_{\tau'\in\tilde{\D}_h^0} \phi(s_h^{\tau'},a_h^{\tau'})\epsilon_h^{\tau'}(\alpha,\tilde{V}^2_{h+1})\|_{(\tilde{\Lambda}_h)^{-1}} \right) \sum_{i=1}^d \|\phi_i(s_h^{\tau},a_h^{\tau})\1_i\|_{\tilde{\Lambda}_h^{-1}},\label{eq:a1_bound_1}
        \end{align}
       where $\epsilon_h^{\tau}(\alpha,V) = \int_{\S} P^0_h(s'|s_h^{\tau},a_h^{\tau}) [V]_{\alpha}(s') \intd s'- [ V]_{\alpha}(s_{h+1}^{\tau})$ for any $V:\S\to[0,H]$, $\tau'\in \tilde{\D}_h^0$ and $\alpha \in [\min_s V(s),\max_s V(s)]$. Since $\tilde{V}_{h+1}$ is independent of $\tilde{\D}_h^0$, we can directly apply Lemma \ref{lemma:Hoeffding} following the same arguments in Lemma \ref{lemma:self_normal}. Therefore, with probability exceeding $1-\delta$, we have
       \begin{equation}
        \sup_{\alpha\in \N(\epsilon_0,H)}\|\sum_{\tau'\in\tilde{\D}_h^0} \phi(s_h^{\tau'},a_h^{\tau'})\epsilon_h^{\tau'}(\alpha,\tilde{V}^2_{h+1})\|_{(\tilde{\Lambda}_h)^{-1}} \le H^2 \sqrt{ 2\log(3HK/\delta) + d\log(1+K)} \le c_{a}H^2\sqrt{d}, \label{eq:a1_bound_2}
       \end{equation}
       where $c_a = 3\log(3HK/\delta)$. Therefore, with probability exceeding $1-\delta$, 
       \begin{equation*}
        (a_1) \le 6c_aH^2\sqrt{d}\sum_{i=1}^d \|\phi_i(s_h^{\tau},a_h^{\tau}) \1_i\|_{\tilde{\Lambda}_h^{-1}}.
       \end{equation*}
       Similarly, with probability exceeding $1-\delta$, one has
       \begin{align}
        (a_2) = |\phi(s_h^{\tau},a_h^{\tau}) \int_{\S}(\tilde{\mu}_{h}(s')-\mu_{h}^0(s')) \tilde{V}_{h+1}(s')\intd s'|\le 6c_aH\sqrt{d}\sum_{i=1}^d \|\phi_i(s_h^{\tau},a_h^{\tau}) \1_i\|_{\tilde{\Lambda}_h^{-1}}.\label{eq:a2_bound}
        \end{align}
        Combining \eqref{eq:a1_bound_1}, \eqref{eq:a1_bound_2}, and \eqref{eq:a2_bound}, we can obtain
        \begin{align}
            (a) \le (a_1) + 2H(a_2) &\le 12c_a H^2\sqrt{d}\sum_{i=1}^d \|\phi_i(s_h^{\tau},a_h^{\tau}) \1_i\|_{\tilde{\Lambda}_h^{-1}}\nonumber
        \end{align}
        with probability exceeding $1-2\delta$.

        \paragraph{Step 2: controlling $(b)$.} 
        Then, 
        \begin{align*}
            (b) &= \left|\Var_{P_h^0}V^{\star,\rho}_{h+1}(s_h^{\tau},a_h^{\tau}) - \Var_{P_h^0}\tilde{V}_{h+1}(s_h^{\tau},a_h^{\tau})\right| \\
            &\le \left|\int_{\S} P_{h,s_h^{\tau},a_h^{\tau}}(s')\left(V^{\star,\rho}_{h+1}(s') - \tilde{V}_{h+1}(s')\right)\left(V^{\star,\rho}_{h+1}(s') + \tilde{V}_{h+1}(s')\right) \intd s'\right| \\
            &\quad+ \left|\int_{\S} P_{h,s_h^{\tau},a_h^{\tau}}(s')\left(V^{\star,\rho}_{h+1}(s') - \tilde{V}_{h+1}(s')\right) \intd s'\right|\left|\int_{\S} P_{h,s_h^{\tau},a_h^{\tau}}(s')\left(V^{\star,\rho}_{h+1}(s') + \tilde{V}_{h+1}(s')\right) \intd s'\right|  \\
            &\le 4H \left|\int_{\S} P_{h,s_h^{\tau},a_h^{\tau}}(s')\left(V^{\star,\rho}_{h+1}(s') - \tilde{V}_{h+1}(s')\right) \intd s'\right| \le 4H \max_{s\in\S}  V_{h+1}^{\star,\rho}(s) - \tilde{V}_{h+1}(s).
        \end{align*}
        
        Denote $\tilde{\iota}_h(s, a) = \mathbb{B}^{\rho}_h \tilde{V}_{h + 1}(s, a) - \tilde{Q}_h(s, a)$, for any $(s,a)\in\S\times\A$ and 
        \begin{align}
            P_{h,s,\pi_h^{\star}(s)}^{\inf,\tilde{V}}(\cdot) \defeq \argmin_{P(\cdot) \in \Prho(P^0_{h,s,\pi_h^{\star}(s)})} \int_{\S}P(s')\tilde{V}_{h+1}(s') \intd s'.
        \end{align}
        For any $h\in[H]$, define $\tilde{P}^{\inf}_h:\S\to\S$ and $\tilde{\iota}_h^{\star}\in\S\to\mathbb{R}$ by
        \begin{equation}
            \tilde{P}^{\inf}_h(s) = P^{\inf,\tilde{V}}_{h,s,\pi_h^{\star}(s)}(\cdot)\quad \text{and} \quad\text{and}\quad
            \tilde{\iota}_h^{\star}(s) \defeq \tilde{\iota}_h(s,\pi^{\star}(s)),\quad\forall s\in\S.
        \end{equation}
        Following the step 1 and 2 in Section \ref{proof:thm_main_iid}, we have 
    \begin{align*}
        V_h^{\star,\rho}(s) - \tilde{V}_h(s)  &\le  \left(\prod_{t=h}^{H}\tilde{P}^{\inf}_j\right)\left(V_{H+1}^{\star,\rho} - \tilde{V}_{H+1}\right)(s) + 
            \sum_{t=h}^H  \left(\prod_{j=h}^{t-1} \tilde{P}^{\inf}_j\right) \tilde{\iota}_t^{\star}(s)\\
        &= \sum_{t=h}^H  \left(\prod_{j=h}^{t-1} \tilde{P}^{\inf}_j\right) \tilde{\iota_t}^{\star}(s)
    \end{align*}
    for any $s\in\S$ and $h\in[H]$, where the equality is from $V_{H+1}^{\star,\rho}(s)=\tilde{V}_{H+1}(s)=0$ for any $s\in\S$ and we denote 
    \begin{align*}
        \left(\prod_{j=t}^{t-1} \tilde{P}^{\inf}_j\right)(s) =  \1_s
        \quad\text{and}\quad 
        \tilde{d}_{h:t}^{\star} = d_h^{\star}\left(\prod_{j=h}^{t-1} \tilde{P}^{\inf}_j\right)\in\D_t^{\star}.
    \end{align*}
    for any $d_h^{\star}\in\D_h^{\star}$.
    Therefore, 
    \begin{equation}\label{eq:b2_1}
        \max_{s\in\S}  V_{h+1}^{\star,\rho}(s) - \tilde{V}_{h+1}(s) \le \sum_{t=h+1}^{H} \max_{s\in\S}\E_{\tilde{d}_{h:t}^{\star}}\tilde{\iota}_t^{\star} \le \sum_{h=1}^{H} \max_{(s,a)\in\S\times\A} \tilde{\iota}_h(s,a). 
    \end{equation}
     Note that for any $(s,a,h)\in\S\times\A\times[H]$
        \begin{align}
            |\tilde{\iota}_h(s,a)| &\le |(\widehat{\mathbb{B}}_{h}^{\rho} \tilde{V}_{h + 1})(s, a) - (\mathbb{B}_h^{\rho} \tilde{V}_{h + 1})(s, a) | + \Gamma_h(s,a)\le 2\Gamma_h(s,a) \nonumber\\
            &\le c_bH\sqrt{d}\sum_{i=1}^d\|\phi_i(s,a)\1_i\|_{\tilde{\Lambda}_h^{-1}}, \label{eq:b2_2}
        \end{align}
        where  $c_b = 12\log(3HK/\delta)$.
        Substituting \eqref{eq:b2_2} into \eqref{eq:b2_1}, we have
        \begin{equation*}
            \max_{s\in\S}  V_{h+1}^{\star,\rho}(s) - \tilde{V}_{h+1}(s) \le \max_{(s,a)\in\S\times\A} c_bH^2\sqrt{d}\sum_{i=1}^d\|\phi_i(s,a)\1_i\|_{\tilde{\Lambda}_h^{-1}}.
        \end{equation*}
        Therefore, 
        \begin{equation*}
            (b) \le \max_{(s,a)\in\S\times\A} 4c_bH^3\sqrt{d}\sum_{i=1}^d\|\phi_i(s,a)\1_i\|_{\tilde{\Lambda}_h^{-1}}
        \end{equation*}
        with probability exceeding $1-\delta$.
        
        \paragraph{Step 3: controlling $(c)$.} 
        Similarly,
        \begin{equation}
            (c) \le 4H\max_{s\in\S}  V_{h+1}^{\star,\rho}(s) - \widehat{V}_{h+1}(s) \le \sum_{t=h+1}^H\sup_{(s,a)\in\S\times\A} |\iota_t^{\sigma}(s,a)| \label{eq:b1_2}
        \end{equation}
        where
        \begin{align}\label{eq:decompose_iota}
            |\iota_h^{\sigma}(s,a)| &\le |(\widehat{\mathbb{B}}_{h}^{\rho,\sigma} \widehat{V}_{h + 1})(s, a) - (\mathbb{B}_h^{\rho} \widehat{V}_{h + 1})(s, a) | + \Gamma_h^{\sigma}(s,a).
        \end{align}
        Following \eqref{eq:control_suboptimality_gap_varest_1} and \eqref{eq:decompose_M1}, we have
        \begin{equation}\label{eq:control_b_iota_1}
                |(\widehat{\mathbb{B}}_{h}^{\rho,\sigma} \widehat{V}_{h + 1})(s, a) - (\mathbb{B}_h^{\rho} \widehat{V}_{h + 1})(s, a) |\le \left(5\sqrt{d} + \sqrt{2}M_{3,h}\right) \sum_{i=1}^d\|\phi_i(s,a)\1_i\|_{\Sigma_h^{-1}},
        \end{equation}
        where  $M_{3,h} := \sup_{\alpha^\dag\in \N(\epsilon_1,H)}\left\|\sum_{\tau=1}^{N_h} \frac{\phi(s_h^{\tau},a_h^\tau)}{\varest_h(s_h^{\tau},a_h^{\tau})} \epsilon_h^{\tau,\sigma}(\alpha^{\dag},\widehat{V}_{h+1})\right\|_{\Sigma_h^{-1}}$. Since $\varest_h$ and $\hatV_{h+1}$ are independent of $\D_h^0$, then we can directly apply Lemma \ref{lemma:Hoeffding} following the same arguments in Lemma \ref{lemma:self_normal}. Therefore, with probability exceeding $1-\delta$, we have
        \begin{align}
            M_{3,h} &\le H\sqrt{2\log(\frac{3H^2}{\epsilon_1\delta}) + d\log(1+N_h/\lambda_1)} \nonumber\\
            &\le H \sqrt{2\log(3H^3K/\delta)+d\log(2H^2K)} \le c_{b}H\sqrt{d}/\sqrt{2}, \label{eq:M2_Hoeffding}
        \end{align}
        where  $c_b = 12\log(3HK/\delta)$. Substituting \eqref{eq:M2_Hoeffding} into \eqref{eq:control_b_iota_1} and combining with \eqref{eq:decompose_iota} result in 
        \begin{align*}
            |\iota_h^{\sigma}(s,a)| &\le (5+c_b+\xi_1)H\sqrt{d}\sum_{i=1}^d\|\phi_i(s,a)\1_i\|_{\Sigma_h^{-1}},\\
            &\le 8c_b H\sqrt{d}\sum_{i=1}^d\|\phi_i(s,a)\1_i\|_{\Sigma_h^{-1}}, \quad \forall (s,a,h)\in\S\times\A\times [H].
        \end{align*}
        Therefore,
        \begin{equation*}
          (c)\le 4H\max_{s\in\S}  V_{h+1}^{\star,\rho}(s) - \widehat{V}_{h+1}(s) \le \sup_{(s,a)\in\S\times\A} 32c_b H^3\sqrt{d}\sum_{i=1}^d\|\phi_i(s,a)\1_i\|_{\Sigma_h^{-1}}
        \end{equation*}
        with probability at least $1-\delta$. 
       
        \paragraph{Step 4: finishing up.} Then, with probability at least $1-3\delta$, we have
        \begin{align}\label{eq:var_V_varest_diff}
            \left|\mathbb{V}_{P_h^0} V^{\star,\rho}_{h+1}(s_h^{\tau},a_h^{\tau})  - \varest_h^2(s_h^{\tau},a_h^{\tau}) \right| &\le \sup_{(s,a)\in\S\times\A} 7c_bH^3\sqrt{d} \sum_{i=1}^d\|\phi_i(s,a) \1_i\|_{\tilde{\Lambda}_h^{-1}},
        \end{align}
        since $c_b=4c_a$. With probability at least $1-\delta$,
        \begin{equation}\label{eq:var_Vhat_Vopt_diff}
        \left|\mathbb{V}_{P_h^0} \widehat{V}_{h+1}(s_h^{\tau},a_h^{\tau}) - \mathbb{V}_{P_h^0} V^{\star,\rho}_{h+1}(s_h^{\tau},a_h^{\tau}) \right| \le \sup_{(s,a)\in\S\times\A} 32c_b H^3\sqrt{d}\sum_{i=1}^d\|\phi_i(s,a) \1_i\|_{\Sigma_h^{-1}}.
        \end{equation}
    Recall that from Lemma \ref{lemma:penalty_term_full}, we can control the term $\sum_{i=1}^d\|\phi_i(s,a) \1_i\|_{\tilde{\Lambda}_h^{-1}}$ for any $(s,a,h)\in\S\times\A\times[H]$, as long as $N_h$ is sufficiently large.  
    Similar to Lemma \ref{lemma:penalty_term_full}, we also employ Lemma \ref{lemma:H5_Min} to control the term $\sum_{i=1}^d\|\phi_i(s,a) \1_i\|_{\Sigma_h^{-1}}$ for any $(s,a,h)\in\S\times\A\times[H]$ as follows, where the proof is deferred to Appendix \ref{sec:lemma_penalty_term_full_Sigma}.
    \begin{lemma}\label{lemma:penalty_term_full_Sigma}
        Consider $\delta\in(0,1)$.
        Suppose Assumption \ref{assump:d_rectangular}, Assumption \ref{assump:full_coverage} and all conditions in Lemma \ref{lemma:H5_Min} hold. For any $h\in[H]$, if $N_h\ge 512 \log(2Hd/\delta)H^4/\kappa^2$, we have 
        \begin{equation*}
            \sum_{i=1}^d\|\phi_i(s,a)\1_i\|_{\Sigma_h^{-1}} \le \frac{2}{\sqrt{N_h\kappa}},\quad\forall (s,a)\in\S\times\A,
        \end{equation*}
        with probability exceeding $1-\delta$. 
    \end{lemma}
    From Lemma \ref{lemma:penalty_term_full}, Lemma \ref{lemma:penalty_term_full_Sigma} and the fact $N_h\ge \frac{K}{24}$, with probability exceeding $1-2\delta$, we have 
    \begin{align}
        \sum_{i=1}^d\|\phi_i(s,a)\1_i\|_{\tilde{\Lambda}_h^{-1}} &\le \frac{2}{\sqrt{N_h\kappa}}\le \frac{10}{\sqrt{K\kappa}},\quad\forall (s,a,h)\in\S\times\A\times[H],\label{eq:sum_penalty_tildeLambda}\\
        \sum_{i=1}^d\|\phi_i(s,a)\1_i\|_{\Sigma_h^{-1}} &\le \frac{2}{\sqrt{N_h\kappa}}\le \frac{10}{\sqrt{K\kappa}},\quad\forall (s,a,h)\in\S\times\A\times[H],\label{eq:sum_penalty_Sigma}
    \end{align}
    as long as $K\ge c_1\log(2Hd/\delta)H^4/\kappa^2$ for some sufficient large universal constant $c_1$.
    Substituting \eqref{eq:sum_penalty_tildeLambda} and \eqref{eq:sum_penalty_Sigma} into \eqref{eq:var_V_varest_diff} and \eqref{eq:var_Vhat_Vopt_diff} respectively, we finally arrive at
        \begin{align}
            \left|\mathbb{V}_{P_h^0} V^{\star,\rho}_{h+1}(s_h^{\tau},a_h^{\tau})  - \varest_h^2(s_h^{\tau},a_h^{\tau}) \right| &\le  \frac{70c_bH^3\sqrt{d}}{\sqrt{K\kappa}},
        \end{align}
        \begin{equation}
        \left|\mathbb{V}_{P_h^0} \widehat{V}_{h+1}(s_h^{\tau},a_h^{\tau}) - \mathbb{V}_{P_h^0} V^{\star,\rho}_{h+1}(s_h^{\tau},a_h^{\tau}) \right| \le  \frac{320c_b H^3\sqrt{d}}{\sqrt{K\kappa}},
        \end{equation}
        for any $(s,a,h)\in\S\times\A\times[H]$, with probability at least $1-6\delta$, which completes the proof as long as $K\ge c_1\log(2Hd/\delta)H^4/\kappa^2$ for some sufficient large universal constant $c_1$.
        
    \subsubsection{Proof of Lemma \ref{lemma:penalty_term_full_Sigma}}\label{sec:lemma_penalty_term_full_Sigma}
        From Assumption \ref{assump:full_coverage}, one has $\lambda_{\min}(\E_{d_h^b}[\frac{\phi(s,a)\phi(s,a)^{\top}}{\varest_h^2(s,a)}])\ge \frac{\kappa}{H^2}$ for any $(s,a,h)\in\S\times\A\times[H]$. Following Lemma \ref{lemma:H5_Min}, we can obtain
    \begin{equation*}
        \|\phi_i(s,a)\1_i\|_{\Sigma_h^{-1}} \le \frac{2\phi_i(s,a)}{\sqrt{N_h\kappa}}, \quad\forall (i,s,a)\in[d]\times\S\times\A,
    \end{equation*}
    as long as $N_h\ge \max\{512 H^4\log(2Hd/\delta)/\kappa^2,4/\kappa\}$. In addition, 
    \begin{equation}
        1 = \int_{\S} P^0_h(s'|s,a)\intd s' = \int_{\S} \phi(s,a)^{\top}\mu_{h}^0(s')\intd s' = \sum_{i=1}^d \phi_i(s,a)\int_{\S}\mu_{h,i}^0(s') \intd s' =  \sum_{i=1}^d \phi_i(s,a),
    \end{equation}
    where the last equality is implied by Assumption \ref{assump:d_rectangular}.
    Therefore, 
    \begin{equation*}
        \sum_{i=1}^d\|\phi_i(s,a)\1_i\|_{\Sigma_h^{-1}}\le \sum_{i=1}^d \frac{2\phi_i(s,a)}{\sqrt{N_h\kappa}} \le \frac{2}{\sqrt{N_h\kappa}}.
    \end{equation*}